%% file: text.tex
\documentclass[nohyperref]{article}

\usepackage{microtype}
\usepackage{graphicx}
\usepackage{booktabs} % for professional tables

\usepackage{hyperref}

\usepackage[accepted]{icml2022}

\usepackage{amsmath}
\usepackage{amssymb}
\usepackage{mathtools}
\usepackage{amsthm}

\usepackage[capitalize,noabbrev]{cleveref}

\usepackage{tcolorbox}
\definecolor{algcolor}{RGB}{150,80,0}%{39,130,67}
\definecolor{mydarkgreen}{RGB}{0,160,0}%{39,130,67}
\definecolor{mydarkred}{RGB}{170,20,20}%{192,47,25}
\definecolor{mydarkblue}{RGB}{40,40,170}%{25,47,192}
\newcommand{\mygreen}{\color{mydarkgreen}}
\newcommand{\myred}{\color{mydarkred}}

\usepackage{multirow}
\usepackage{colortbl}
\definecolor{bgcolor}{rgb}{0.8,1,1}
\definecolor{bgcolor2}{rgb}{0.8,1,0.8}
\usepackage{xspace}
\newcommand{\algname}[1]{{\color{algcolor}\sf #1}\xspace}

\usepackage{float}
\usepackage{subcaption}
\usepackage{graphicx}

\newcommand{\cstep}{{\myred\gamma}} % client step size
\newcommand{\cstepsquared}{{\myred\gamma^2}} % client step size squared
\newcommand{\sstep}{{\mygreen\eta}} % server step size
\newcommand{\sstepsquared}{{\mygreen\eta^2}} % server step size
\newcommand{\gstepsize}{{\color{blue}\theta}} % server step size

\newcommand{\clper}{{\lambda_r}} %  client permutation
\newcommand{\dtper}{{\pi_j}} % data permutation

\theoremstyle{plain}
\newtheorem{theorem}{Theorem}[section]

\newtheorem{lemma}[theorem]{Lemma}

\theoremstyle{definition}
\newtheorem{definition}[theorem]{Definition}
\newtheorem{assumption}[theorem]{Assumption}
\theoremstyle{remark}
\newtheorem{remark}[theorem]{Remark}

\input{macros}

\usepackage[colorinlistoftodos,bordercolor=orange,backgroundcolor=orange!20,linecolor=orange,textsize=scriptsize]{todonotes}

\usepackage{xspace}

\newcommand{\sgd}{{\sc SGD}\xspace}

\newcommand{\mynorm}[1]{\left\|#1\right\|^2}

\usepackage{pifont}% http://ctan.org/pkg/pifont
\usepackage{xparse}   % http://ctan.org/pkg/xparse

\icmltitlerunning{Federated Learning with Regularized Client Participation}

\begin{document}

\twocolumn[
\icmltitle{Federated Learning with Regularized Client Participation}



\icmlsetsymbol{equal}{*}

\begin{icmlauthorlist}
\icmlauthor{Grigory Malinovsky}{yyy}
\icmlauthor{Samuel Horv\'ath}{comp}
\icmlauthor{Konstantin Burlachenko}{yyy,sdaia}
\icmlauthor{Peter Richt\'arik}{yyy}
\end{icmlauthorlist}

\icmlaffiliation{yyy}{King Abdullah University of Science and Technology (KAUST), Thuwal, Saudi Arabia}
\icmlaffiliation{comp}{Mohamed bin Zayed University of Artificial Intelligence (MBZUAI), Abu Dhabi, United Arab Emirates}
\icmlaffiliation{sdaia}{Center of Excellence in Data Science and Artificial Intelligence (SDAIA-KAUST AI), Thuwal, Saudi Arabia}

\icmlcorrespondingauthor{Grigory Malinovsky}{grigorii.malinovskii@kaust.edu.sa}
\icmlcorrespondingauthor{Samuel Horv\'ath}{ samuel.horvath@mbzuai.ac.ae}

\icmlkeywords{Machine Learning}

\vskip 0.3in
]



\printAffiliationsAndNotice{\icmlEqualContribution} % otherwise use the standard text.

\begin{abstract}
Federated Learning (FL) is a distributed machine learning approach where multiple clients work together to solve a machine learning task. One of the key challenges in FL is the issue of partial participation, which occurs when a large number of clients are involved in the training process. The traditional method to address this problem is randomly selecting a subset of clients at each communication round. In our research, we propose a new technique and design a novel regularized client participation scheme. Under this scheme, each client joins the learning process every $R$ communication rounds, which we refer to as a meta epoch. We have found that this participation scheme leads to a reduction in the variance caused by client sampling. Combined with the popular FedAvg algorithm~\citep{mcmahan2017communication}, it results in superior rates under standard assumptions. For instance, the optimization term in our main convergence bound decreases linearly with the product of the number of communication rounds and the size of the local dataset of each client, and the statistical term scales with step size quadratically instead of linearly (the case for client sampling with replacement), leading to better convergence rate $\cO(\nicefrac{1}{T^2})$ compared to $\cO(\nicefrac{1}{T})$, where $T$ is the total number of communication rounds.  Furthermore, our results permit arbitrary client availability as long as each client is available for training once per each meta epoch. 
%  Furthermore, we discuss several extensions including ... showcasing universality of our approach. 
Finally, we corroborate our results with experiments.
\end{abstract}

\section{Introduction}

\emph{Federated learning} (FL) aims to train models in a decentralized manner, preserving the privacy of client personal data by leveraging local computational capabilities. Clients' data never leave their devices; instead, the clients communicate with a server via updates intended for immediate aggregation to train a global model. 
Due to such advantages and promises, FL is now deployed in a variety of applications~\citep{hard2018federated, apple19wwdc} and is a promising direction for smart healthcare, where privacy is of an essential importance~\citep{rieke2020future, sheller2020federated}.

In this paper, we consider the standard FL problem formulation of solving an empirical risk minimization problem over the data available from all participating clients, i.e., 
\begin{align}
    \label{eq:objective}
    \min_{x \in \R^d} \sbr*{f(x) \eqdef \frac{1}{M}\sum_{m = 1}^M f_m(x)},
\end{align}
where
\begin{align*}
    f_m(x)\eqdef\frac{1}{N}\sum_{j=1}^{N} f_{m}^j(x).
\end{align*}
The term $f_{m}^j$ corresponds to the local loss of the current model parameterized by $x \in \R^d$ evaluated for the $j$-th data point on the dataset belonging to the $m$-th client. \footnote{Although we focus on the setting, where each client has the same amount of data, our results can be extended using techniques introduced by~\citet{mishchenko2021proximal}.}

We assume that there is a large number of small clients. This setting is often referred to as cross-device FL. Cross-device FL leverages edge devices such as mobile phones and different Internet of Things (IoT) devices to exploit data distributed over potentially millions of data sources~\citep{bonawitz2017practical, bhowmick2018protection, niu2020billion}. As such, it brings unique challenges compared to standard distributed learning. For instance, optimization methods must contend with issues related to edge computing~\citep{lim2020federated, xia2021survey}, participant selection~\citep{yang2021achieving, chen2020optimal, cho2020client, ribero2020communication}, system heterogeneity~\citep{diao2020heterofl, horvath2021fjord} and communication constraints such as low network bandwidth and high latency~\citep{arjevani2015communication, mcmahan2017communication, stich2018local, yu2019parallel, horvath2020better}. 
 
 In this work, we focus on \emph{partial participation}, i.e., clients only intermittently participate in the collaborative training process~\citep{bonawitz2017practical}. As previously discussed, in cross-device FL, the number of participating devices $M$ can be in the order of millions. At this scale, \emph{client sampling} (i.e., using a subset of clients for each update) is a necessity since it is impractical for all devices to participate in every round because it would require computation and communication that can consume a large amount of energy and also lead to network congestion. In addition, clients are only sometimes available. For example, suppose the client devices are smartphones. In that case, they may be willing to participate in FL only when they are charging and connected to a high-speed network (usually during night hours) to avoid draining the battery and creating a negative user experience. Finally, each device may only participate once or a few times during the entire training process, so \emph{stateless} methods (which do not rely on each client maintaining and updating local state throughout training) are particularly interesting.

\begin{figure*}[!ht]
    \centering
    \includegraphics[width=.8\textwidth]{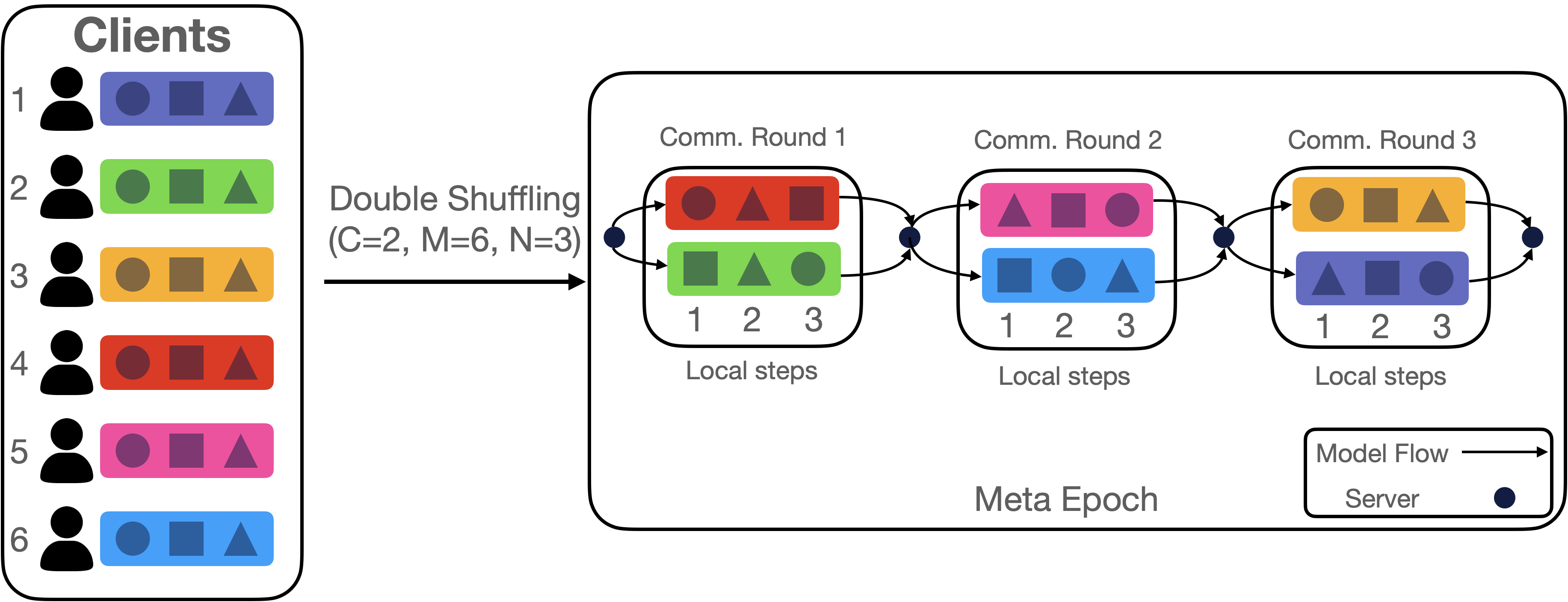}
    \caption{Visualization of double-shuffling procedure for $6$ clients, each with $3$ datapoints. Two clients are sampled in each communication round.}
    \label{fig:double_shuffling}
\end{figure*}

A typical approach to limit the number of clients participating in each is to employ uniform sampling~\citep{mcmahan2017communication}. In each round, the orchestrator (global server) picks $C$ clients sampled uniformly at random that perform local training. A more general approach is to select clients with a given importance-sampling-based probability distribution that is independent across rounds~\citep{fraboni2021clustered, fraboni2021impact, chen2020optimal, horvath2022fedshuffle}. In this work, we introduce a novel client sampling strategy, where we do not sample clients independently. Instead, we propose a regularized participation strategy, where each client participates once during a period we refer to as a meta-epoch. Our motivation stems from centralized training, where it is now well-understood that random reshuffling, i.e., data sampling without replacement, has a variance-reducing effect~\citep{mishchenko2020random}. Therefore, we propose to apply the random reshuffling procedure at the client level. 
We discuss random reshuffling and client sampling in detail in the related work section.
% 
% \samo{TODO: \\
% - clients only optimize during the night? Maybe we already analyze this case, we have to check in the analysis.\\
% - we shuffle client only once? Add a comment that we can resample in each meta epoch. \\
% - compare to ``super-client'' approach. How does the rate change? \\}
% 
\section{Contributions}
The key contributions of our work are the following
\begin{itemize}
    \item We design a novel client participation strategy based on regularized participation, where each client participates once during each meta epoch. 
    \item We combine the proposed client selection scheme with the \algname{FedAvg}-like method that, apart from partial participation, incorporates local steps, local dataset shuffling, and server and client step sizes. We refer to this new method as \algname{RR-CLI}. We provide rigorous convergence guarantees and show that in the considered setups, our results give state-of-the-art performance by providing the best scaling in terms of both the linearly decreasing optimization term and the statistical term proportional to squared step sizes.
    \item The theoretical analysis is corroborated by the experimental evaluation that validates our findings.
\end{itemize}

\section{Related Work}

Cross-device FL can be hindered by communication costs, as edge devices such as mobile phones and IoT devices often have limited bandwidth and connectivity~\citep{5090858, huang2013depth}. These limitations can make wireless connections and internet connections expensive and unreliable. Additionally, limitations in the capacity of the aggregating master and other FL system factors can restrict the number of clients allowed to participate in each communication round. To address these issues, there is significant interest in reducing the communication bandwidth of FL systems through techniques such as local updates, communication compression, and client selection methods. Our work primarily focuses on client selection techniques, but it is worth noting that these approaches can be combined to achieve a more effective outcome.

\subsection{Local Methods} This strategy involves reducing the frequency of communication and emphasising local computation, where each device performs multiple local steps before communicating its updates back to the central node. A prototypical method in this category is the Federated Averaging ({\tt FedAvg}) algorithm~\citep{mcmahan2017communication}, an adaption of local-update to parallel {\tt SGD}, where each client runs some predefined number of {\tt SGD} steps based on its local data before local updates are averaged to form the global pseudo-gradient update for the global model on the master node. Recently, there has been significant interest and attempts to provide theoretical guarantees for this method, or its variants~\citep{stich2018local, lin2018don, karimireddy2019scaffold, stich2019error, khaled2020tighter, Hanzely2020,malinovskiy2020local, koloskova2020unified, mishchenko2022proxskip, malinovsky2022variance} as the original work was a heuristic, offering no theoretical guarantees. 

\subsection{Communication Compression Methods} Another popular technique works by reducing the size of the updates communicated from clients to the master. This approach is referred to as communication compression. In this approach, instead of transmitting the full-dimensional update vector $g \in \R^d$, each client only transmits a compressed vector $\cC(g)$, where $\cC: \R^d \rightarrow \R^d$ is a (possibly random) operator chosen such that $\cC(g)$ can be represented using fewer bits than $g$, for instance, by using limited bit representation (quantization)~\citep{alistarh2017qsgd, wen2017terngrad, zhang2017zipml, horvath2019natural, ramezani2019nuqsgd} or by enforcing sparsity (sparsification)~\citep{wangni2018gradient, konevcny2018randomized,stich2018sparsified, mishchenko201999, vogels2019powersgd}.

\subsection{Client Sampling/Selection Methods} On top of the uniform~\citep{mcmahan2017communication,karimireddy2019scaffold,grudzien2022can} or arbitrary sampling~\citep{horvath2022fedshuffle}, several proposed approaches focus on a careful selection of the participating clients to improve communication complexity~\citep{cho2020client, nguyen2020fast, ribero2020communication, lai2021oort, luo2022tackling, chen2020optimal}. These techniques rely on the extra partial information, such as the client's loss or the norms of the updates, to speed up the training by selecting more informative updates. Another stream of works tackles convergence under arbitrary client participation patterns~\citep{yang2022anarchic, wang2022unified, gu2021fast, yan2020distributed, ruan2020towards}. In contrast, our proposed method selects clients using regularized sampling strategy based on client reshuffling. We note that such sampling is not independent across communication rounds and is not arbitrary, i.e., our goal is not to provide bounds for arbitrary participation patterns, as we assume we have access to client sampling to provide a better practical and theoretical sampling strategy. To account for the standard practice of FL training, where clients are only available during certain hours when their device is on charge and connected to the high-speed network, we also provide the convergence rates under non-random deterministic client shuffling that can still guarantee convergence under this challenging scenario; see Remark~\ref{rem:deterministic_rr}.

\subsection{Random reshuffling} A particularly successful technique to optimize the empirical risk minimization objective is to randomly permute (i.e., reshuffle) the training data at the beginning of every epoch~\citep{bottou2012stochastic} instead of randomly sampling a data point (or a subset of them) with replacement at each step, as in the standard analysis of \sgd. This process is repeated several times, and the resulting method is usually referred to as Random Reshuffling ({\sc RR}). {\sc RR} is often observed to exhibit faster convergence than sampling with replacement, which can be intuitively attributed to the fact that {\sc RR} is guaranteed to process each training sample exactly once every epoch. In contrast, with-replacement sampling needs more steps than the equivalent of one epoch to see every sample with a high probability. 
Correctly understanding the random reshuffling trick and why it works has been a challenging open problem~\citep{bottou2009curiously, ahn2020tight, gurbuzbalaban2021random} until recent advances in~\citet{mishchenko2020random} introduced a significant simplification of the convergence analysis technique. The difficulty of analyzing {\sc RR} stems from the fact that updates conditioned on the previous iterate result in biased gradient estimates, unlike with-replacement sampling. The subsequent works provide better convergence guarantees for RR in different settings~\citep{NEURIPS2020_cb8acb1d,malinovsky2021random,beznosikov2021random}.
To our knowledge, in terms of FL, {\sc RR} was only explored in terms of local steps. Initial works \citep{mishchenko2021proximal, yun2021minibatch,malinovsky2022federated,sadiev2022federated} require full participation in each communication round. The partial participation framework was considered in the following works~\citep{Nastya, horvath2022fedshuffle},
% \grisha{RR-comp, sadiev, suvrit} 
but the authors only consider unbiased client participation. In this work, we fill this missing gap and show that {\sc RR} employment on the client level leads to superior theoretical and practical performance in FL. 
\begin{algorithm*}[!t]
   \caption{\algname{RR-CLI}: Federated optimization with server step sizes and global step sizes, random shuffling and partial participation with shuffling}
   \label{alg:pp-jumping}
\begin{algorithmic}[1]
	\STATE {\bf Input:}  {\myred client step size $\cstep > 0$}; {\mygreen server step size $\sstep > 0$}; {\color{blue} global step size $\theta > 0$}; cohort size $C \in \{1,2,\dots,M\}$; number of rounds $R = M/C$; initial iterate/model $x_0 \in \mathbb{R}^d$; number of epochs $T\geq 1$
    \STATE {\mygreen \textbf{Client-Shuffle-Once option:} sample a permutation $\lambda=(\lambda_0, \lambda_1, \ldots, \lambda_R)$ of  $[R]$}
    \STATE {\myred \textbf{Data-Shuffle-Once option:} for each client $m$, sample a permutation $\pi_m=(\pi^0_{m}, \pi^1_{m}, \ldots, \pi^{N-1}_{m})$ of  $[N]$}
    \FOR{meta-epoch $t = 0,1,\ldots, T-1$}
    \STATE {\mygreen \textbf{Client-Reshuffling option:} sample a permutation $\lambda=(\lambda_0, \lambda_1, \ldots, \lambda_R)$ of  $[R]$}
    \FOR{communication rounds $r=0,\ldots,R-1$}
    \STATE Send model $x^r_{t}$ to all participating clients $m\in S^\clper_t$  \hfill {\footnotesize (server broadcasts  $x^r_t$ to all clients $m\in S^\clper_t$ )}
    \FOR{all clients $m\in S^\clper_{t}$, locally in parallel}
    \STATE $x^{r,0}_{m,t} = x^r_{t}$ \hfill {\footnotesize (client $m$ initializes local training  using the latest global model $x^r_{t}$)}
    \STATE {\myred \textbf{Data-Random-Reshuffling option:} sample a permutation $\pi_m=(\pi^0_{m}, \pi^1_{m}, \ldots, \pi^{N-1}_{m})$   of $[N]$}
    \FOR{all local training data points $j=0, 1, \ldots, N-1$}
    \STATE $x_{m,t}^{r,j+1} = x_{m,t}^{r,j} - \cstep \nabla f^{\pi^j_{m}}_{m} (x_{m,t}^{r,j})$ \hfill {\footnotesize (client $m$ makes one  pass over its local training data in the order dictated by $\pi_m$)}
    \ENDFOR
    \STATE $g^r_{m,t} = \frac{1}{{\color{blue}\cstep} n}(x^r_{t} - x^{r,N}_{m,t})$ \hfill {\footnotesize (client $m$ computes local update direction $g_{m,t}$)}
    \ENDFOR
    \STATE $g^r_{t} = \frac{1}{C}\sum \limits_{m\in S^\clper_{t}}g^r_{m,t} $ \hfill {\footnotesize (server aggregates the local update directions $g_{m,t}$ discovered by the cohort $S_t$ of clients)}
    \STATE $x^{r+1}_{t} = x^r_{t} - {\myred\sstep} g^r_{t} $ \hfill {\footnotesize (server updates the model using  the aggregated direction $g_t$ and applying server step size $\sstep$)}
    \ENDFOR
    \STATE $x_{t+1} = x_t - \gstepsize\frac{x_t - x^R_{t}}{\eta R}$ \hfill {\footnotesize global step after all communication rounds during meta-epoch}
    \ENDFOR
    \end{algorithmic}
\end{algorithm*}
\section{Notation and Assumptions}
The loss function for client $m$ is made up of individual losses $f_m^{j}(x)$ for each local data point $j$, where $x$ is a parameter that we want to optimize. We assume that client $i$ has access to an oracle that, when given input $(j, x)$, returns the gradient $\nabla f_{m}^j(x)$. We denote $[l] = \cbr*{1, 2, \ldots, l}$ for any $l \in \N$.
To show the convergence of our method, we make certain standard assumptions.
\begin{assumption}
    \label{ass:strong-convexity}
    The function $h$ is \textbf{$\mu$-convex} for $\mu \geq 0$; i.e.,
      \begin{equation}
      \label{eq:strong-convexity}
  \inp*{\nabla h(x)}{y - x} \leq - \Bigl(h(x) - h(y) + \frac{\mu}{2}\norm*{x - y}^2\Bigr)\,.
      \end{equation}
      We say that $h$ is $\mu$-strongly convex if $\mu > 0$, and otherwise $h$ is (general) convex.
\end{assumption}
\begin{assumption}
\label{ass:smoothness}
    The function $h$ are \textbf{$L$-smooth}; i.e., there is $L > 0$ 
    \begin{equation}
    \label{eq:lip-grad}
    \norm*{\nabla h(x) - \nabla h(y)} \leq L \norm*{x - y}\,.
  \end{equation}
\end{assumption} 

We also define the Bregman divergence 
\begin{align}
\label{eq:bregman}
	D_h(x, y)
	\eqdef h(x) - h(y) - \dotprod{\nabla f(y)}{x - y}.
\end{align}

Next, we proceed with the definition of double shuffling sampling, which plays a key role in our theoretical analysis.
\begin{definition}[Double Shuffling]
\label{def:double-shuffling}
Let $\tilde{\pi} = \cbr*{\tilde{\pi}_1, \tilde{\pi}_2, \ldots, \tilde{\pi}_M}$ be a random permutation of $[M]$ and $\cbr{\hat{\pi}_m = \cbr*{\hat{\pi}_m^1, \hat{\pi}_m^2, \ldots, \hat{\pi}_m^N}}_{m=1}^M$ is a set of $M$ independent random permutations of $[N]$. Then the \textbf{double-shuffling} procedure $\pi = \cbr*{\pi_1, \pi_2, \ldots, \pi_{MN}}$ is defined as 
\begin{align}
    \label{eq:double-shuffling} 
    \pi_k = \hat{\pi}_{\tilde{\pi}_{m_k}}^{j_k} \eqdef \pi_{m}^j,\quad \forall k \in [MN],
\end{align}
where $m_k = \left\lfloor\frac{k}{N}\right\rfloor$ and $j_k = k - m_k N$, i.e., $j_k$ and $m_k$ are quotient and remainder of $k$ with respect to $M$. 

For the case of mini-batching with batch size $C$, we, for simplicity, assume $M = CR$. We first split $M$ clients into $C$ equisized groups $G_1, G_2, \ldots, G_C$ obtained by without-replacement sampling from $[N]$, i.e., $\bigcup_{o=1}^N G_o = [N]$ and $G_i \cap G_j = \emptyset, \; \forall i \neq j.$ To obtain $C$ samples in steps $1, \ldots, RN$, we apply double-shuffling within each group.
\end{definition}
The visualization of the double-shuffling process is displayed in Figure~\ref{fig:double_shuffling}.
Equipped with these definitions, we proceed with the proposed algorithm and the main results.
\section{Description of Algorithm}
The backbone of the proposed algorithm is based on the celebrated \algname{FedAvg}~\citep{mcmahan2017communication} further inspired by recent advances~\citep{horvath2022fedshuffle, Nastya}. Our method combines previously considered local steps, local dataset reshuffling, and server and client step sizes, but, in addition, we also introduce regularized client participation, i.e., sampling without replacement of clients, to the algorithm. This extra feature is the main algorithmic contribution of our work.  

Each meta epoch $t$ starts with the partitioning of all $M$ clients into $M/C$ cohorts $S_t = \cbr*{S_t^{\lambda_r}}_{r=0}^{R-1}$, each $S_t^{\lambda_r}$ with size $C$. These cohorts are either obtained using the without-replacement sampling of clients, i.e., the outer loop of the double shuffling procedure or given by client availability. In our main theoretical part, we assume cohorts $S_t$ to be obtained using the without-replacement sampling of clients, but we also provide an extension that works with any (including deterministic) reshuffling. Client shuffling could be the same or resampled for each meta epoch, our theory handles both cases, and they lead to the same convergence bound.

We also use permutations locally, i.e., shuffling of local client's data points, which corresponds to the inner loop of the double shuffling procedure.  For both permutations, we admit two options. Either we sample one permutation at the beginning that is used in each step, we call this option Shuffle-Once, or we resample new permutations in each meta epoch, which we refer to as Random Reshuffling. Similarly to the previous case, both options lead to the same convergence bound. 

Each meta epoch contains $R = M/C$ communication rounds. For each communication round $r$, the server sends the current model estimate $x^r_t$ to clients, which belong to cohort $S^\clper_t$. Each client $m \in S^\clper_t$ sets $x^{r,0}_{m,t} = x^r_t$ and proceeds with $N$ local steps using permutation of datapoints $\pi^j_m$, i.e.,
\begin{align*}
	x^{r,j+1}_{m,t} = x^{r,j}_{m,t} - \cstep \nabla f_m^{\pi^j_m}\left(x^{r,j}_{m,t}\right).
\end{align*}
Once the local epoch ($N$ local steps) is finished, each client $m \in S^\clper_t$ forms the following local gradient estimator
 \begin{align*}
 	g^r_{m,t} = \frac{1}{\cstep N}\left( x^r_t - x^{r,N}_{m,t} \right).
 \end{align*}
All clients in the cohort $S^\clper_t$ send their estimated directions $g^r_{m,t}$ to the server, and these updates are aggregated using averaging operator 
\begin{align*}
	g^r_t = \frac{1}{C} \sum_{m\in S^\clper_t} g^r_{m,t}.
\end{align*}
The aggregated gradient estimator $g^r_t$ is used to perform the server-side step, which is taken from the initial point $x^r_t$
\begin{align*}
x^{r+1}_t = x^r_t - \sstep g^r_t.
\end{align*}
Note that if we set $\sstep = \cstep N$, then we have that the updated model $x^{r+1}_t$ equal to the average of models on the clients from the current cohort $S^\clper_t$.
\begin{align*}
	x^{r+1}_t &= x^r_t - \sstep g^r_t = x^r_t - \sstep \frac{1}{C} \sum_{m\in S^\clper_t} g^r_{m,t}\\
	 &= x^r_t - \sstep \frac{1}{C} \sum_{m\in S^\clper_t} \frac{1}{\cstep N}\left( x^r_t - x^{r,N}_{m,t} \right)\\
	 &= x^r_t - \frac{\sstep}{\cstep\mu} \frac{1}{C} \sum_{m\in S^\clper_t}  x^r_t +\frac{\sstep}{\cstep\mu} \frac{1}{C} \sum_{m\in S^\clper_t} x^{r,N}_{m,t} \\
	 &= \frac{1}{C} \sum_{m\in S^\clper_t} x^{r,N}_{m,t}.
\end{align*}
Similar to the local update, the server constructs the global update $\frac{x_t - x^R_t}{\sstep R}$ at the end of the meta epoch $t$, when all clients participated and apply this update for its global step to obtain new estimator in the form
\begin{align*}
	x_{t+1} = x_t - \gstepsize \frac{x_t - x^R_t}{\sstep R}. 
\end{align*}
Analogously to server-side steps after local epochs, if we set $\gstepsize = \sstep R$, then the new model $x_{t+1}$ is equal to $x^R_t$
\begin{align*}
x_{t+1} &= x_t - \gstepsize \frac{x_t - x^R_t}{\sstep R} \\
&= x_t - \sstep R \frac{x_t - x^R_t}{\sstep R}\\
&=x_t - x_t + x^R_t = x^R_t.
\end{align*}
The pseudocode for this procedure is provided in Algorithm~\ref{alg:pp-jumping}. The introduction of extra global step sizes $\sstep$ and $\cstep$ is a useful algorithmic trick that will enable faster rates in scenarios, where we can't directly analyze local improvement, e.g., see \citep{karimireddy2019scaffold}. Our main result (Theorem~\ref{thm:main}) does not require this trick and provides the fastest convergence guarantees.
\section{Convergence Guarantees}
Before proceeding with our theoretical results, we first define the variance quantities that commonly appear in the convergence analysis of stochastic methods
\begin{align}
    \label{eq:star_variance}
    \begin{split}
        \tilde{\sigma}_\star^2 &\eqdef \frac{1}{M} \sum_{m=1}^M \norm*{\nabla f_m(x_\star)}^2, \\
        \sigma^2_\star &\eqdef \frac{1}{MN} \sum_{m=1}^M \sum_{j=1}^N \norm*{\nabla f^j_m(x_\star)}^2,
    \end{split}
\end{align}
and the star sequence
\begin{align}
    \label{eq:star_sequence}
    \begin{split}
    x^0_\star &= x^{0, 0}_{m,\star} = x_\star,\; \forall m \in S_t^{\lambda_0}\\
    x^{r,j+1}_{m,\star} &= x^{r,j}_{m,\star} - \gamma\nabla f^{\pi^j_{m}}_{m}\left( x_\star \right),\; \forall m \in [M], \\
    x^{r+1}_\star &= \frac{1}{C}\sum_{m\in S_t^\clper} x^{r,N}_{m,\star}. 
    \end{split}
\end{align}
that corresponds to running our algorithm with the optimal solution $x_\star$ being the starting point and all the local gradients are estimated at $x_\star$. Note that for this sequence, $x^R_\star = x^\star$. Equipped with these definitions, we proceed with the analysis.

We start with the main result, where we assume each function $f_m^j$ to be $\mu$-strongly convex. In this case, we show that the optimization term decreases linearly with the power that is a product of the number of local data points $N$, the number of communications rounds in each meta-epoch $R$ and the number of meta-epochs $T$. In addition, due to applying sampling without replacement of both data points and cohorts, the statistical term scales proportionally to the squared step size $\cstepsquared$. The following theorem formulates the claim.
\begin{theorem}
\label{thm:main}
	Suppose each function $f_m^j$ is $\mu$-strongly convex and $L$-smooth. Then for Algorithm~\ref{alg:pp-jumping} with constant step sizes $\cstep\leq \frac{1}{L}$, $\sstep = \cstep N$, $\gstepsize = \sstep R$, the iterates generated by the Algorithm~\ref{alg:pp-jumping} satisfy
	\begin{align*}
		\Exp{\mynorm{x_{T} - x_\star } }	&\leq (1-\cstep\mu)^{NRT} \mynorm{ x_0 - x_{\star} }\\
		&  +\frac{2\cstepsquared}{\mu}\max_{r,m}	\sigma^2_{m, \text{DS}},
	\end{align*}
    where $\sigma^2_{m, \text{DS}} \eqdef  \frac{1}{\cstepsquared}\Exp{D_{f_m^{\pi_j}}\left( x^{r,j}_{m,\star}, x_\star \right)}.$
\end{theorem}
To our knowledge, this bound is the first result in FL literature which combines exponentially fast decaying in the first term (optimization term), which is proportional to the number of all gradient steps ($\#$ of data $\times$ $\#$ of round $\times$ $\#$ of meta-epochs) and the second term (variance/statistical term), which is proportional to $\cstepsquared$. This result is possible due to our careful algorithmic construction that involves sampling without the replacement of both clients and local data points. Let us now analyze the variance term $\sigma^2_{m, \text{DS}}$. Lemma~\ref{lem:bregman_bound} in the appendix gives the following upper bound
\begin{align*}
    \max_{m \in [M]} \sigma^2_{m, \text{DS}} &\leq L \rbr*{\frac{MN^2}{2C^2} + 2N^2} \tilde{\sigma}_\star^2 + \frac{LN}{2}\sigma_\star^2,
\end{align*}
which is independent of $\cstep$ and $T$. Therefore, the final convergence rate also scales with $\cO\rbr*{\nicefrac{1}{T^2}}$ compared to $\cO\rbr*{\nicefrac{1}{T}}$ to any method that samples clients uniformly at random in each step, e.g., \algname{FedAvg}~\citep{mcmahan2017communication} and \algname{SCAFFOLD}~\citep{praneeth2019scaffold}.

The next set of results is slightly weaker than the theorem above since we analyse problem \eqref{eq:objective} under the weaker assumptions, where we assume that only local functions $f_m$ are $\mu$-strongly convex while individual loss functions $f_m^j$ are general convex and $L$-smooth. In this regime, we can't guarantee a linear decrease in each local step. Therefore, the linear part of the convergence result has power, which is the product of  the number of communications rounds in meta-epoch $R$ and the number of meta-epochs $T$. In addition, the linear term depends on server-side step size $\sstep$. The variance term is decoupled into two parts. The first part, which is proportional to $\sstepsquared$ (note that this is still quadratic dependence), is related to sampling without replacement of clients. The second part, which is proportional to $\cstepsquared$ is related to the reshuffling of data points. Due to the biased nature of updates and lack of individual strong convexity, the step size should be significantly small with the condition $\cstep\leq \frac{1}{8NL\sqrt{\kappa}}$. This is not surprising and is consistent with the analysis of biased \algname{SGD}~\citep{ajalloeian2020convergence}. The formal statement of the theorem follows.
\begin{theorem}
	Assume that each $f_m$ is $\mu$-strongly convex. Also, assume that each $f_{m}^j$ is convex and $L$-smooth. Let $\sstep \leq \frac{1}{4L}$ and $\cstep\leq \frac{1}{8NL\sqrt{\kappa}}$, then for iterates generated by Algorithm~\ref{alg:pp-jumping}, we have
\begin{align*}
	\Exp{\|x_{T} - x_\star\|^2} &\leq (1-\sstep\mu)^{RT}\| x_0 - x_\star \|^2\\
	&+\frac{4}{\mu}\sstepsquared\max_{r,m}\sigma^2_{m,\text{CS}}\\
	&+12\kappa^2\cstepsquared N^2\tilde{\sigma}^2_\star+12\cstepsquared\kappa^2N\sigma^2_\star,
\end{align*}
	where $\sigma^2_{m,\text{CS}} = \frac{1}{\sstepsquared}\Exp{D_{f_{m}}\left(x_\star^r, x_\star\right)}.$
\end{theorem}
Note that the last two terms can be as small as we want by taking $\cstepsquared$ small enough. For the first two terms, we first give the upper bound on 
$ \sigma^2_{m, \text{CS}}$ using Lemma~\ref{lem:bregman_bound} from the appendix that gives
\begin{align*}
    \max_{m \in [M]} \sigma^2_{m, \text{CS}} &\leq \frac{LM}{2C^2}\Tilde{\sigma}_\star^2.
\end{align*}
which is independent of $\sstep$ or $T$. Therefore, similar to the prior case, the final convergence rate scales with $\cO\rbr*{\nicefrac{1}{T^2}}$.

\begin{figure*}[ht!]
	\centering
	\captionsetup[sub]{font=scriptsize,labelfont={}}	
		
	\includegraphics[width=0.48\textwidth]{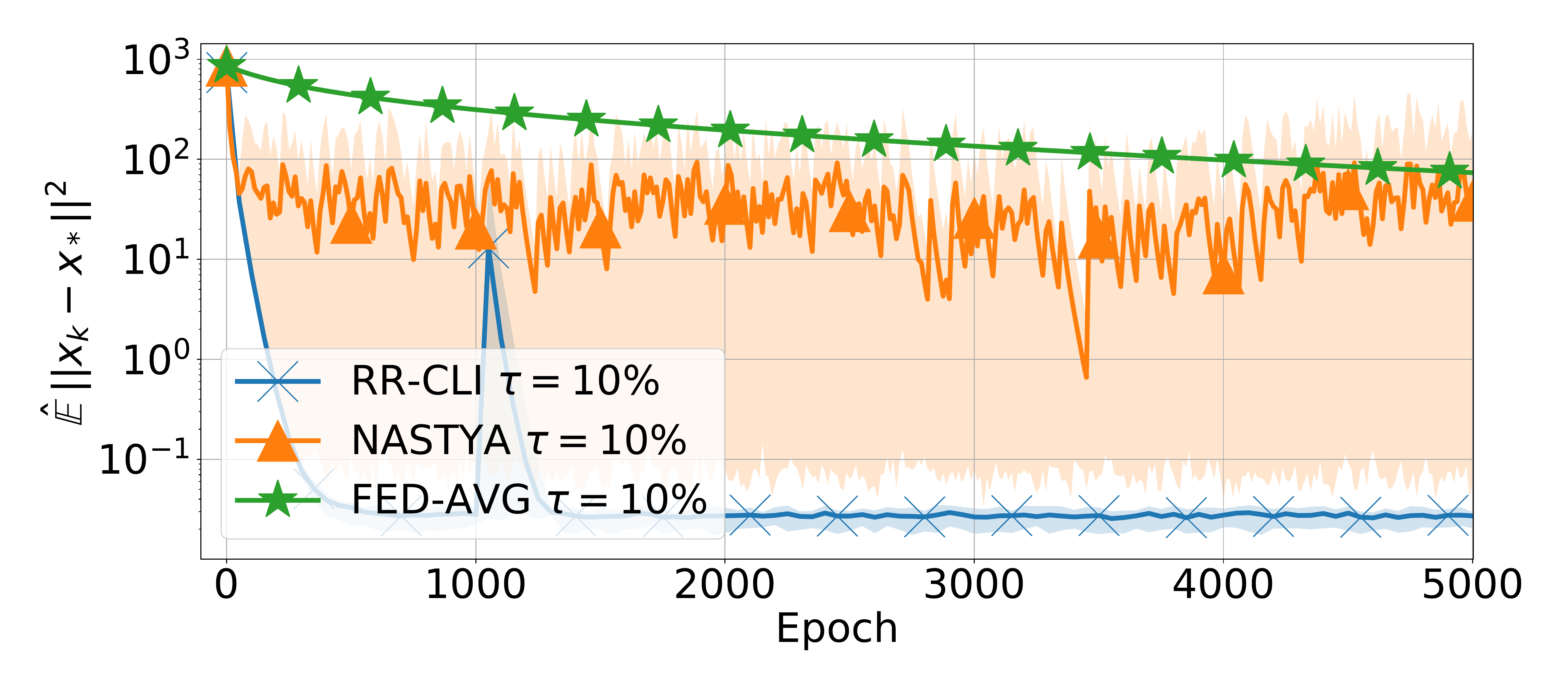}
	\includegraphics[width=0.48\textwidth]{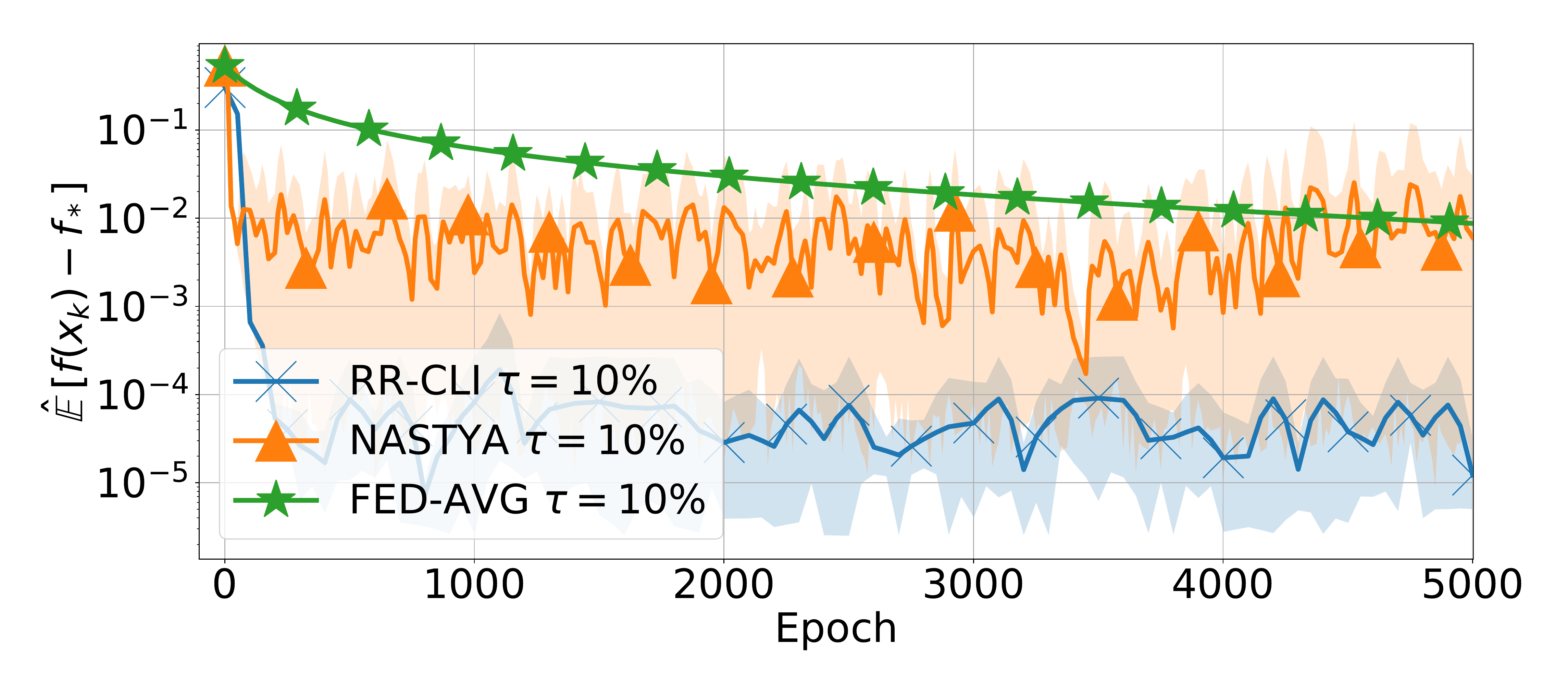}
	\caption{\small{Training \texttt{Logistic Regression} on \texttt{phishing}, with $n=12$ clients. Theoretical local and global step sizes, $3$ clients per round with $10$ local steps.}}
	\label{fig:exp1_th_step sizes}
\end{figure*}
The final theorem analyzes the most restrictive case when only the function $f$ is $\mu$ strongly convex and individual losses are general convex and $L$-smooth. In this case, the trick with three step sizes is suitable as this helps us to decompose the bound into three parts. The linear part depends on global step size $\gstepsize$, the variance coming from the sampling of data points, which depends on $\cstep$, and the variance coming from the sampling of clients, which depends on $\sstep$. This is summarized in the following theorem.
\begin{theorem}
\label{thm:main_last}
	Suppose that each $f^j_m$ is convex and $L$-smooth, $f$ is $\mu$-strongly convex. Then provided the step size satisfies $\cstep N R \leq \sstep R \leq \gstepsize \leq \frac{1}{16L}$ the final iterate generated by Algorithm~\ref{alg:pp-jumping} satisfies
	\begin{align*}
		\Exp{\mynorm{x_T - x_\star}} &\leq \left(1-\frac{\gstepsize\mu}{2}\right)^T\mynorm{x_0 - x_\star}\\
		&+16\cstepsquared \kappa N^2 \tilde{\sigma}_\star^2 \\
		&+16\cstepsquared  \kappa N \sigma_\star^2\\
		&+16 \sstepsquared  \frac{\kappa}{N^2R}\frac{M-C}{(M-1)C}\tilde{\sigma}^2_\star .
	\end{align*}
\end{theorem}
Note that because of our decomposition, the last three terms can be arbitrarily small by taking $\cstep$ and $\sstep$ sufficiently small. For the first term, we can take $\gstepsize = \frac{1}{16L}$ to obtain linear convergence.
\begin{remark}
\label{rem:deterministic_rr}
All these results can be also adjusted to work with the deterministic client shuffling with the minor change of the rates and convergence analysis; see Section~\ref{sec:deterministic_rr} in the appendix for a detailed discussion.
\end{remark}
\section{Numerical experiments}
\label{sec:experiments}

We conduct numerical experiments in which we analyzed how \algname{RR-CLI} algorithm (Algorithm~\ref{alg:pp-jumping}) compares with its closest competitors, namely \algname{NASTYA}~\citep{Nastya} and \algname{FED-AVG} \citep{mcmahan2017communication}. For the shuffling, we use the Shuffle Once option.

\subsection{Computing and software environment} 
We performed experiments on a cluster with nodes running CentOS Linux release 7.9.2009 and Linux Kernel 3.10.0-1160 x86\_64. Each experiment runs on a compute node with an NVIDIA GPU and $40$ GBytes of virtual memory for the Python interpreter. We use double precision (fp64) float point arithmetic. The distributed environment within each experiment is simulated using Python software suite \texttt{FL\_PyTorch}~\citep{burlachenko2021fl_pytorch}.

\subsection{Optimization problem and experimental setup}
We consider the following formulation. Each $f_m(x)$ is in the form of empirical risk minimization objective for logistic loss with additive $L_2$ regularization term for local data $D_i$
\[f_m(x)=\dfrac{1}{N} \sum_{j=1}^{N}l^j_{m} (x, a^{(j)}, b^{(j)}) + \frac{\alpha}{2} \|x\|^2,\]
\[l^j_{m}(x,a,b)=\log(1+\exp(-b^{(j)}_m \cdot x^{\top} a^{(j)}_m)).\]

For our experiments, we have considered three LIBSVM datasets ~\citep{libsvm} -- {\tt phishing}, {\tt w3a} and {\tt a3a}. To distribute the dataset across clients, we reshuffle this dataset $D$ uniformly and split it into $M=12$ clients so that each client has $N = \lfloor \frac{|D|}{M} \rfloor$ data samples. Along different runs, we have fixed the split of data across the clients and the initial iterate $x_0$. In the main part, we include only the {\tt phishing} dataset. The remaining experiments are provided in Appendix~\ref{sec:extra_experiments}.

Since we use logistic regression, we can exactly compute the smoothness parameters.
% For {\tt phishing} dataset (for $\alpha=0$) constructed logistic loss function has $\max_{1\le i \le n, 1 \le j \le m} (L_{l_{m}^j}) = 0.5$, $\max_{1\le i \le n} (L_{f_i}) = 0.41$ and $L_{f}=0.41$. 
We set the regularizer $\alpha = 5e-4$. This results in having the condition number $\kappa \approx 10^4$ for {\tt phishing}. Quantities of our interest are the rate of decreasing distance to the optimal point $\|x_k - x_*\|^2$ and functional gap $f(x_k) - f(x_*)$. We have pre-computed numerically $x_\star$ such that $\nabla f(x_\star) \leq \cdot 10^{-14}$. All our experiments involve $5$ independent runs to obtain estimates of $\mathbb{E}[f(x_k)-f^*]$ and $\mathbb{E}[\|x_k - x_*\|^2]$.
\begin{figure*}[th!]
	\centering
	\captionsetup[sub]{font=scriptsize,labelfont={}}	
	\includegraphics[width=0.48\textwidth]{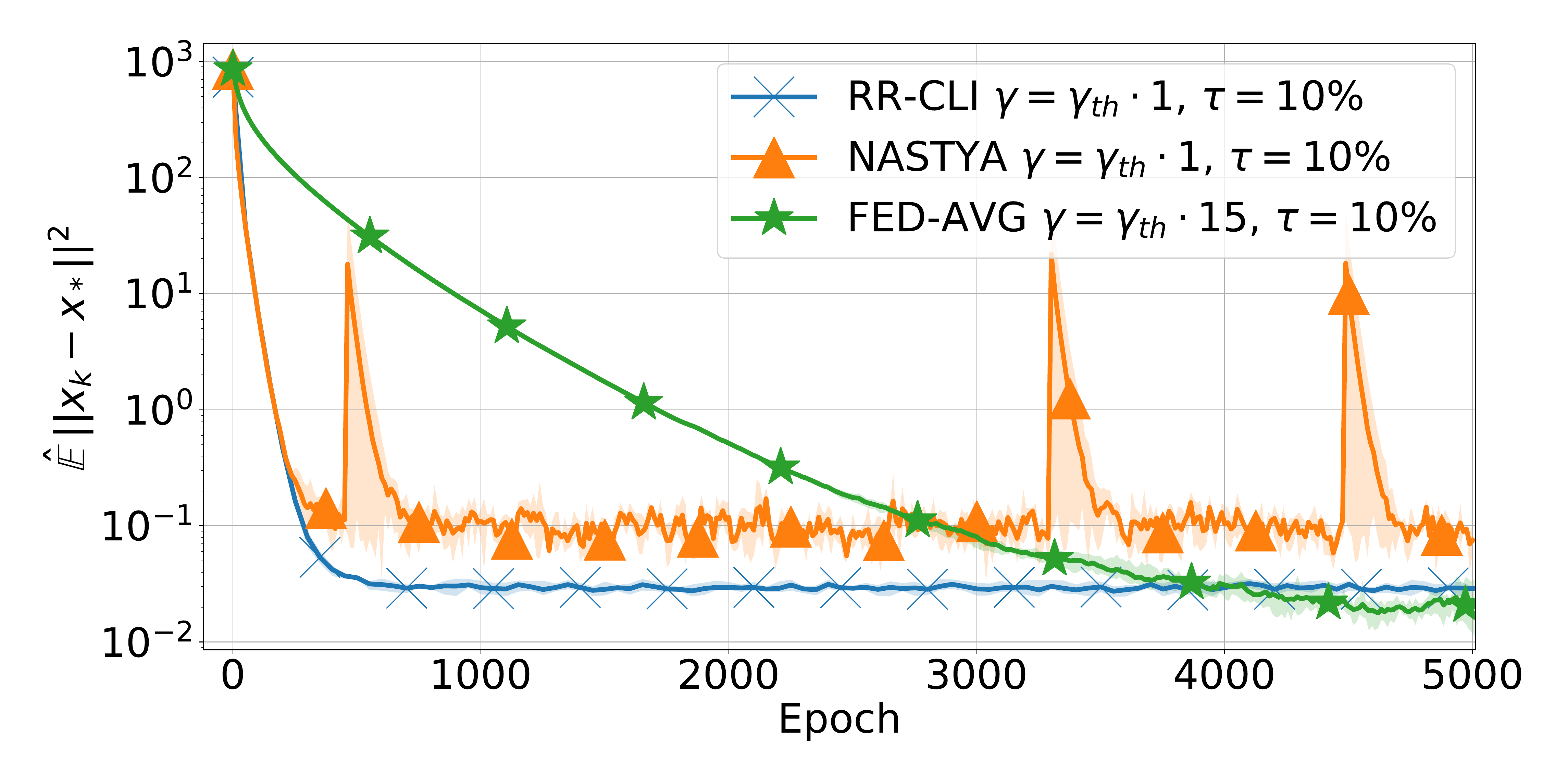}
	\includegraphics[width=0.48\textwidth]{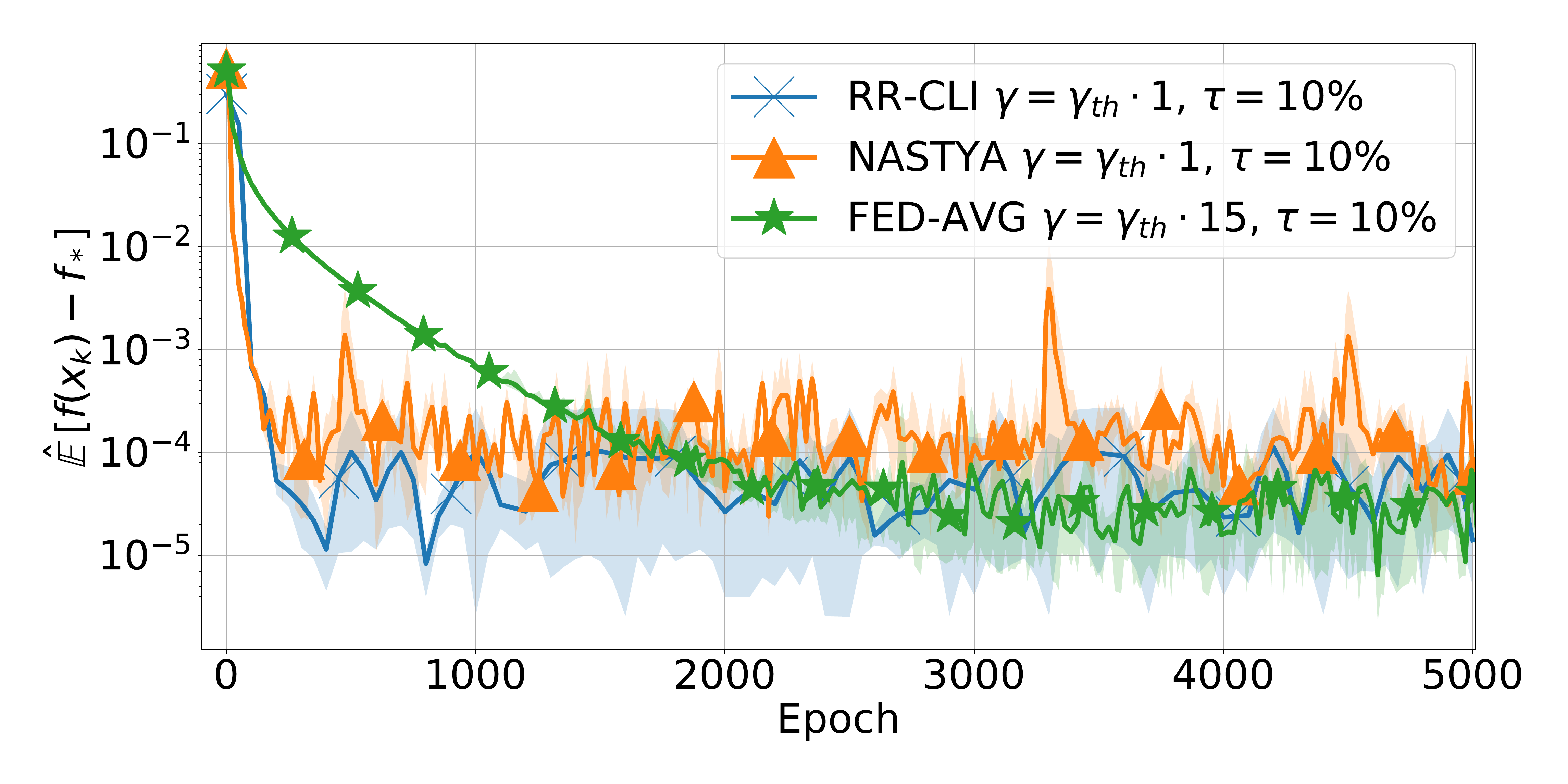}	
	\caption{\small{Training \texttt{Logistic Regression} on \texttt{phishing} with $n=12$ clients. Theoretical global step size and tune local step sizes. Partial participation with $3$ clients per round with $10$ local steps.}}
	\label{fig:exp2_multth_step sizes_best_to_best}
\end{figure*}
\begin{figure*}[th!]
	\centering
	\captionsetup[sub]{font=scriptsize,labelfont={}}	
	\includegraphics[width=0.48\textwidth]{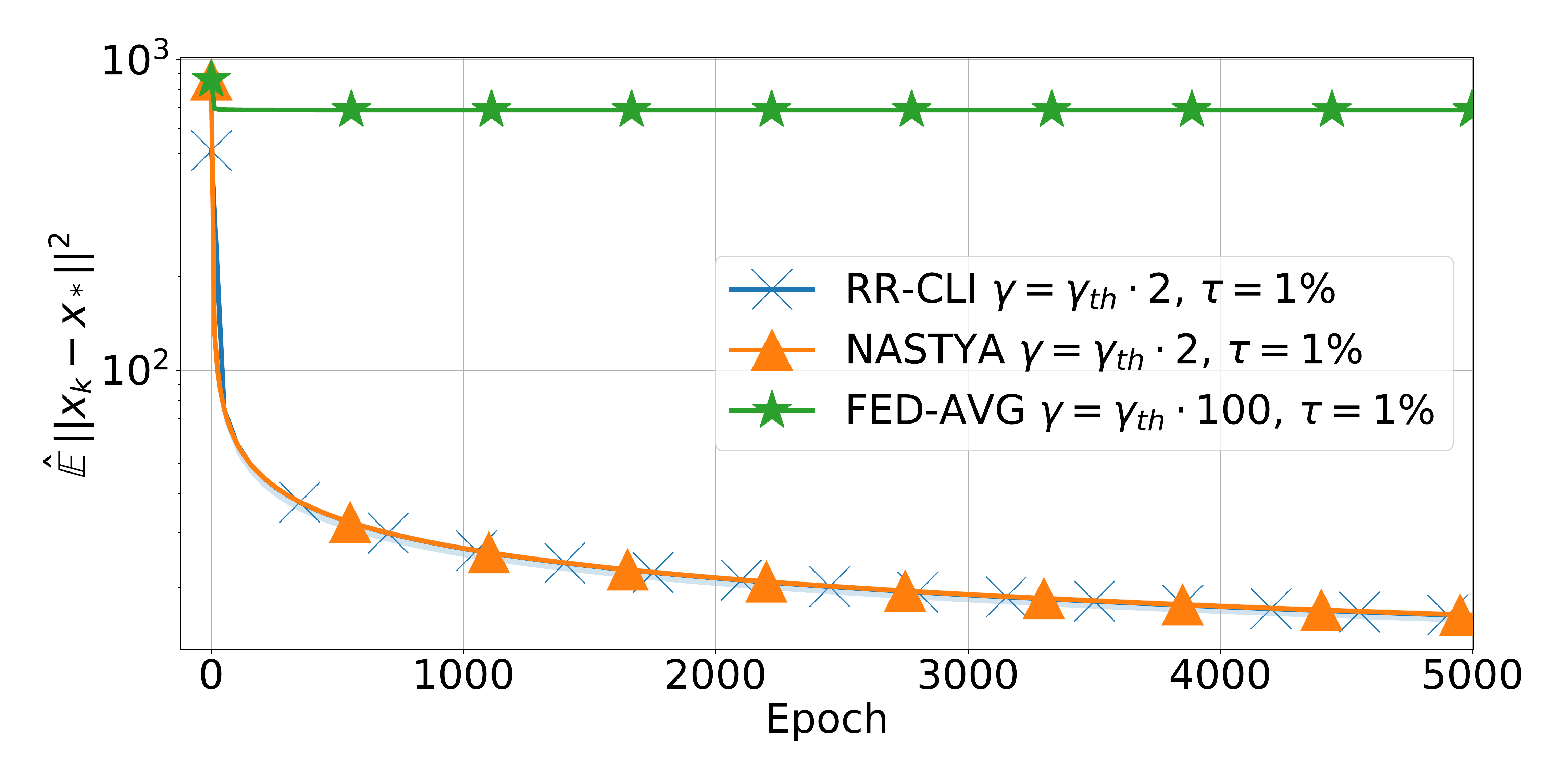}
	\includegraphics[width=0.48\textwidth]{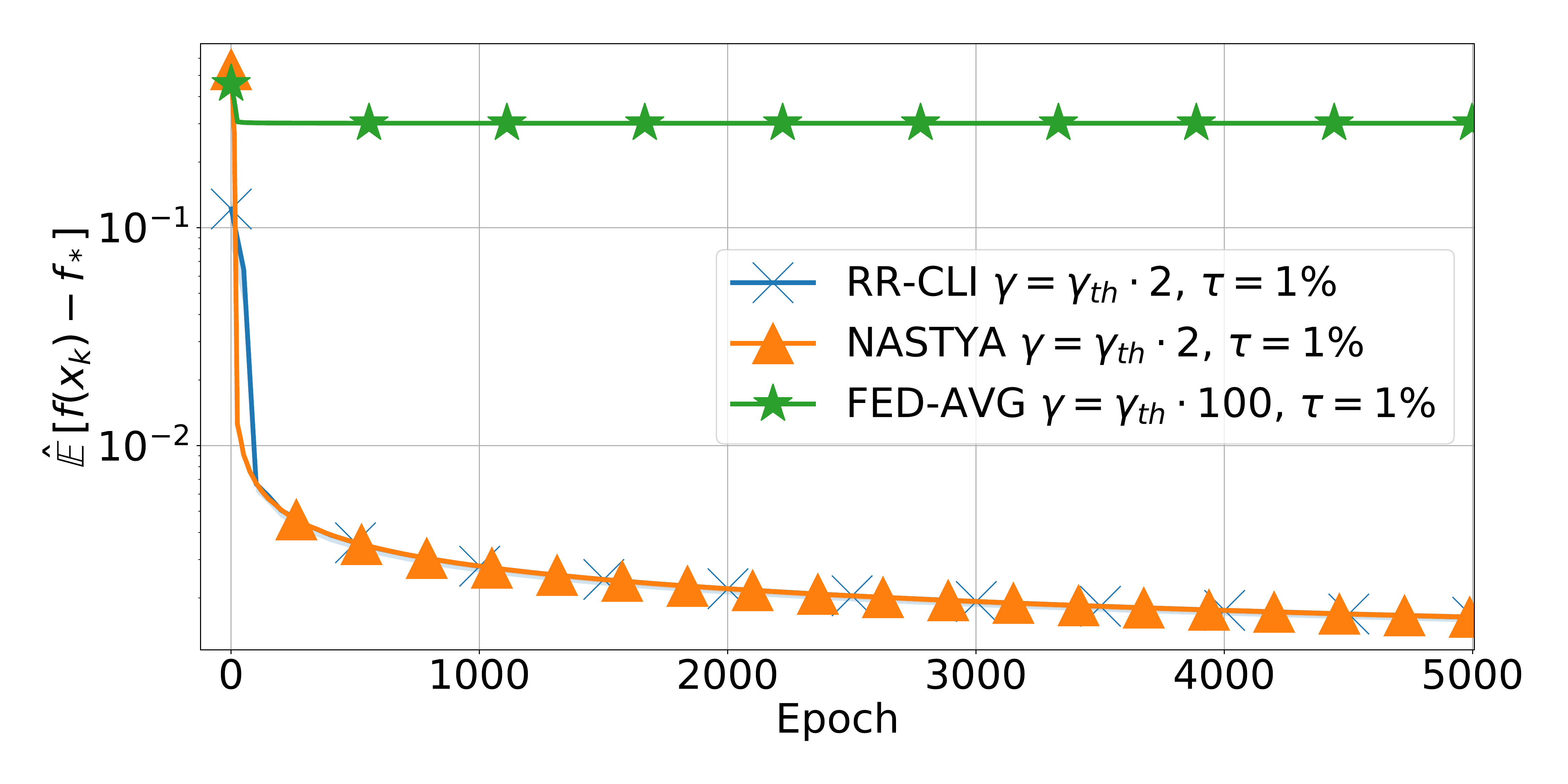}	
	\caption{\small{Training \texttt{Logistic Regression} on \texttt{phishing} with $n=12$ clients with heoretical global step size and tuned local step sizes. Partial participation with $3$ clients per round, $100$ local step. Local and global step size are decaying $\propto \frac{1}{1+\mathrm{\#passed epochs}}$. Local gradient estimators are computed with $1\%$ of local samples.}}
	\label{fig:exp3_multth_step sizes_best_to_best_decay}
\end{figure*}
\subsection{Theoretical step sizes} Firstly, we use theoretical step sizes for all methods described in the original papers, except \algname{FED-AVG} for which we have used theoretical step size from more recent paper \citep{praneeth2019scaffold}. The \algname{FED-AVG} uses uniform sampling without replacement with $10\%$ of local data points sampled in each local step. Both \algname{FED-AVG} and \algname{NASTYA} sample $3$ clients from a total of $n=12$ clients in each round uniformly at random. The \algname{RR-CLI} samples clients using the random reshuffling procedure with $4$ cohorts, each with $3$ clients. All algorithms make $10$ local steps before communication with the central server. This makes all the algorithms equivalent in terms of local computations and communication. Since different algorithms use different sampling strategies, we measure the number of epochs, where one epoch is the amount of computation equivalent to the computation of the full gradient across the whole dataset. Our results are presented in Figure~\ref{fig:exp1_th_step sizes}. One can note that \algname{RR-CLI} outperforms both baselines. Firstly, as theory predicts, \algname{RR-CLI} dominates \algname{NASTYA} due to smaller variance from partial participation since the variance terms scaled with $\cstep^2$ (\algname{RR-CLI}) instead of  $\cstep$ (\algname{NASTYA}). \algname{FED-AVG} exhibit poor behaviour its theory only permits small step size $\cstep = \dfrac{1}{L + \alpha}$.

% The theoretical local step size for \algname{NASTYA} and \algname{RR-CLI}, which we have used $\gamma_{th}=\dfrac{1}{\max_{1 \le i \le n}(L_i)}$ and global step size $\eta_{th}=\dfrac{N}{\max_{1 \le i \le n}(L_i)}$, where $N$ is denoted by a number of local steps in the methods. The theoretical local step size for \algname{FED-AVG} $\gamma_{th} = \dfrac{1}{6 \beta N (1 + B^2) \cdot \eta_{th}}$ and global step size for \algname{FED-AVG}  $\eta_{th}=\sqrt{S}$, where $B^2 = 2 \left( \dfrac{n - s}{s \cdot (n - 1)} + \dfrac{n (s - 1)}{s \cdot (n - 1)} \cdot \dfrac{L}{\max(Li)}\right)$. Here $N$ is denoted by a number of local steps in the methods, $n$ is the number of clients, and $s$ number of sampled clients per round. It is possible to show that \texttt{SGD-NICE} sampling for clients satisfies $(G,B)$ assumption A1  \citep{praneeth2019scaffold} with this $B^2$ quantity.

%===============================================================
%fonts: 
%line width: 4
%fonts: 34
%fonts: 34

%Legend:{algorithm} $\tau=${opt-tau}
%Optimal solution: %C:\Users\burla\plots_th_step_sizes_phishing\kappa_10000_regul_0.00005

%\begin{figure*}[t]

%===============================================================

\subsection{Tuned step sizes} Since we provide the worst-case analysis, the step size estimates might be too conservative. The goal of this ablation is to analyze how much we can increase local step sizes in practice. To test this, we consider the list of multipliers for theoretical local step size: $\{1,2,3,5,10,12,15\}$. 
% From numerical experiments presented in Figure \ref{fig:exp2_multth_step sizes} we see that \algname{NASTYA} and \algname{RR-CLI} do not provide a way to tune local step size. This experiment demonstrates that for fixed global step size $\eta$ these methods are tight. However \algname{FED-AVG} allows in this experiment to scale theoretical local step size by a factor $15$ (See Figure~\ref{fig:exp2_multth_step sizes}) with still preserving convergence during making $10$ local steps. 
In Figure~\ref{fig:exp2_multth_step sizes_best_to_best}, we showcase a comparison of the three considered algorithms using the tuned local step sizes. We note that \algname{RR-CLI} still outperforms both baselines.

\subsection{Tuned step sizes with decaying} In the last experiment, we consider decaying local and global step size with factor $\propto \frac{1}{\mathrm{\#passed\, epochs}+1}$. In this experiment, we again analyze several possible multipliers of theoretical step size, concretely $\{1,2,3,5,10,12,15,50,100\}$. In this experiment, we increase the number of local steps to $100$ by decreasing the local batch size by a factor of $10$. Results are presented in Figure~\ref{fig:exp3_multth_step sizes_best_to_best_decay}. We note that both \algname{NASTYA} and \algname{RR-CLI} follow a similar trend, while \algname{FED-AVG} lacks in performance.

\section{Conclusion and Future Work}
In conclusion, we propose a new technique for addressing the issue of partial participation in Federated Learning. We design a novel regularized client participation scheme. By having each client join the learning process once during each meta epoch, the proposed scheme leads to a reduction in the variance caused by client sampling that is reflected both in theoretical analysis and practical performance. 

For future work, we are interested in combining the proposed algorithm with momentum or different adaptive methods to make it more practical.

\clearpage

\bibliography{literature}
\bibliographystyle{icml2022}

%%%%%%%%%%%%%%%%%%%%%%%%%%%%%%%%%%%%%%%%%%%%%%%%%%%%%%%%%%%%%%%%%%%%%%%%%%%%%%%
%%%%%%%%%%%%%%%%%%%%%%%%%%%%%%%%%%%%%%%%%%%%%%%%%%%%%%%%%%%%%%%%%%%%%%%%%%%%%%%
% APPENDIX
%%%%%%%%%%%%%%%%%%%%%%%%%%%%%%%%%%%%%%%%%%%%%%%%%%%%%%%%%%%%%%%%%%%%%%%%%%%%%%%
%%%%%%%%%%%%%%%%%%%%%%%%%%%%%%%%%%%%%%%%%%%%%%%%%%%%%%%%%%%%%%%%%%%%%%%%%%%%%%%
\clearpage
\onecolumn
\appendix
    
\part*{Appendix}

\section{Notation}
For the ease of the reader, we provide the table with the used notation below. 
\begin{table}[h!]
	\caption{Notation}
	\begin{center}
		\begin{tabular}{cc}
						\quad Symbol \qquad& \qquad Description\quad\\
			\hline
			\quad$M$\qquad& \qquad the total number of clients.\quad\\
			\hline
			\quad$N$\qquad& \qquad the total  number of data points per client.\quad\\
			\hline
			\quad$C$\qquad& \qquad  the cohort size of clients selected for participation.\quad\\
			\hline
			\quad$R = M/C$\qquad& \qquad  the number of communication rounds during meta-epoch.\quad\\
			\hline
			\quad$r$\qquad& \qquad  index of communication round in meta-epoch.\quad\\		
			\hline
			\quad$j$\qquad& \qquad  index of data point in local epoch.\quad\\
			\hline
			\quad$\lambda_r$\qquad& \qquad permutation of cohorts for communication round $r$.\quad\\
			\hline 
			\quad $ \pi_j $ \qquad& \qquad permutation of data point with index $j$.\quad\\
			\hline
			\quad $S_t^{\lambda_r}$ \qquad& \qquad set of $C$ clients in cohort corresponding to communication round $r$ during meta epoch $t$.\quad\\
	\hline
\quad $x^r_t$  \qquad& \qquad the initial point for round $r$ in epoch $t$.\quad\\
	\hline
\quad $x^{r,j}_{m,t}$ \qquad& \qquad  the intermediate point for round $r$ in epoch $t$ on client $m$ and for $j$-th data point.\quad\\
	\hline
\quad $g^r_t$ \qquad& \qquad  the approximation of gradient for round $r$ in epoch $t$.\quad\\
\hline
		\end{tabular}
	\end{center}
\end{table}

\section{Variance bounds}

In this section, we present bounds that can be used to bound the variance of the gradient estimators used in this work. We start by introducing standard variance decomposition and then presenting two lemmas. 

For random variable $X$ and any $y \in \R^d$, the variance can be decomposed as
\begin{align}
    \label{eq:var_decomposition}
    \E{\norm*{X - \E{X}}^2} = \E{\norm*{X - y}^2} - \norm*{\E{X} - y}^2.
\end{align}

The following lemma bounds the variance of the estimator obtained using the sampling without replacement with respect to both clients and local data points.

\begin{lemma}\label{lem:sampling_wo_replacement}
	Let $\cbr{\cbr*{\zeta_{m}^j}_{m=1}^M}_{j=1}^N \in \R^d$ be fixed vectors, and 
	$$
	\Tilde{\zeta} \eqdef \frac{1}{MN}\sum_{m=1}^M \sum_{j=1}^N \zeta_{m}^j, \qquad \Tilde{\zeta}_m \eqdef \frac{1}{N} \sum_{j=1}^M \zeta_{m}^j
	$$ be their averages, and 
	$$
	\sigma^2 \eqdef \frac{1}{NM}\sum_{m=1}^M \sum_{j=1}^N \norm*{\zeta_{m}^j-\Tilde{\zeta}}^2, \qquad
	 \tilde{\sigma}^2 \eqdef \frac{1}{M}\sum_{m=1}^M \norm*{\Tilde{\zeta}_m-\Tilde{\zeta}}^2,
	$$ be the population variances. 
	\begin{enumerate}
	    \item Fix any $k\in\{1,\dotsc, MN\}$, let $\zeta_{\pi_1}, \dotsc \zeta_{\pi_k}$ be sampled using double-shuffling procedure (see Definition~\ref{def:double-shuffling}) from $\cbr{\cbr*{\zeta_{m}^j}_{j=1}^m}_{i=1}^n$ and $\Tilde{\zeta}^k_\pi$ be their average. Then, the following holds for the sample average and variance
	\begin{align}
	\label{eq:sampling_wo_replacement}
        \begin{split}
            \E{\Tilde{\zeta}^k_\pi}&=\Tilde{\zeta}, \\
    	    \Bar{\sigma}^2(k) \eqdef \E{\norm*{\Tilde{\zeta}^k_\pi - \Tilde{\zeta}}^2} &=  \frac{j_k (N - j_k)}{k^2(N-1)}\sigma^2 + \rbr*{\frac{(m_k N^2 + j_k^2) (MN-1)}{k^2(N-1)(M-1)} - \frac{N}{k(N-1)} - \frac{1}{M-1}} \tilde{\sigma}^2.
        \end{split}
	\end{align}
    \item Let $N = CR$ and $G_1, G_2, \ldots, G_C$ be sets with $R$ elements each, obtained by sampling uniformly at random without replacement from $[N]$, i.e., $\bigcup_{o=1}^b G_o = [N]$ and $G_1 \cap G_j = \cbr{}, \; \forall i \neq j.$ 
    Fix any $k\in\{1,\dotsc, NR\}$, let $\zeta^o_{\pi_1}, \dotsc \zeta^o_{\pi_k}$ be sampled using double-shuffling procedure (see Definition~\ref{def:double-shuffling}) from $\cbr{\cbr*{\zeta_{m}^j}_{j=1}^m}_{m \in G_o}$ for $o \in [R]$. 
    Let $k_N \eqdef \left\lfloor \frac{k}{N}\right\rfloor N$ and 
    \begin{align}
        \Tilde{\zeta}^k_{\pi, o} \eqdef \frac{1}{k} \rbr*{\sum_{i=1}^{k_N} \frac{1}{C}\sum_{p=1}^C \zeta^p_{\pi_i} + \sum_{i=k_N + 1}^{k} \zeta^o_{\pi_i}}.
    \end{align}
    Then the following holds for the expectation and variance of $\Tilde{\zeta}^k_{\pi, o}$
	\begin{align}
	    \label{eq:minibatch_sampling_wo_replacement}
	    \E{\Tilde{\zeta}^k_{\pi, o}}=\Tilde{\zeta}, \quad
	     \E{\norm*{\Tilde{\zeta}^k_{\pi, o} - \Tilde{\zeta}}^2} =  \rbr*{\frac{k_N}{k}}^2 \Bar{\sigma}^2(k_N C) + \rbr*{\frac{k - k_N}{k}}^2 \Bar{\sigma}^2(k - k_N) - \frac{2 k_N (k - k_N)}{k^2 C (M-1)} \Tilde{\sigma}^2.
	\end{align}
	\end{enumerate}

\end{lemma}
\begin{proof}
	The first claim follows from the linearity of expectation and uniformity of sampling with respect to both permutations. Therefore,
	\[
		\E{\Tilde{\zeta}^k_\pi} 
		= \frac{1}{k}\sum_{i=1}^k \zeta_{\pi_k}
		= \frac{1}{k}\sum_{i=1}^k \frac{1}{M}\sum_{m=1}^M \frac{1}{N} \sum_{j=1}^N \zeta_{m}^j
		= \Tilde{\zeta}.
	\]
	To prove the second claim, let us first establish the identities for $\mathrm{cov}(\zeta_{\pi_{r}}, \zeta_{\pi_{s}})$ for any $r\neq s$. Firstly, we consider the case such that $\left\lfloor \frac{r}{N}\right\rfloor = \left\lfloor \frac{s}{N}\right\rfloor$. Then,
	\begin{align*}
		\mathrm{cov}\sbr*{\zeta_{\pi_r}, \zeta_{\pi_s}}
		&= \E{\inp*{\zeta_{\pi_r} - \Tilde{\zeta}}{ \zeta_{\pi_s} - \Tilde{\zeta}}}
		= \frac{1}{M} \sum_{m=1}^M\frac{1}{N(N-1)}\sum_{k=1}^N\sum_{l=1,k\neq l}^N\inp*{\zeta_{mk} - \Tilde{\zeta}}{ \zeta_{ml} - \Tilde{\zeta}} \\
		&= \frac{1}{M} \sum_{m=1}^M\frac{1}{N(N-1)}\sum_{k=1}^N\sum_{l=1}^N\inp*{\zeta_{mk} - \Tilde{\zeta}}{ \zeta_{ml} - \Tilde{\zeta}}  - \frac{1}{M} \sum_{m=1}^M\frac{1}{N(N-1)}\sum_{k=1}^N \norm*{\zeta_{mk} - \Tilde{\zeta}}^2 \\
		&= \frac{1}{MN(N-1)}\sum_{m=1}^M\rbr*{N^2\norm*{\Tilde{\zeta}_m - \Tilde{\zeta}} -  \sum_{k=1}^N \norm*{\zeta_{mk} - \Tilde{\zeta}}^2} \\
		&= \frac{1}{N-1}\rbr*{N\tilde{\sigma}^2 - \sigma^2}
	\end{align*}
	For the case, $\left\lfloor \frac{r}{N}\right\rfloor \neq \left\lfloor \frac{s}{N}\right\rfloor$, we have
	\begin{align*}
		\mathrm{cov}\sbr*{\zeta_{\pi_r}, \zeta_{\pi_s}}
		&= \E{\inp*{\zeta_{\pi_r} - \Tilde{\zeta}}{ \zeta_{\pi_s} - \Tilde{\zeta}}}
		= \frac{1}{M(M-1)} \sum_{m=1}^M \sum_{o=1, i\neq o}^M\frac{1}{N^2}\sum_{k=1}^N\sum_{l=1}^N\inp*{\zeta_{mk} - \Tilde{\zeta}}{ \zeta_{ol} - \Tilde{\zeta}} \\
		&= \frac{1}{M(M-1)} \sum_{m=1}^M \sum_{o=1}^M\frac{1}{N^2}\sum_{k=1}^N\sum_{l=1}^N\inp*{\zeta_{mk} - \Tilde{\zeta}}{ \zeta_{ol} - \Tilde{\zeta}} - \frac{1}{M(M-1)} \sum_{m=1}^M \frac{1}{N^2}\sum_{k=1}^N\sum_{l=1}^N\inp*{\zeta_{mk} - \Tilde{\zeta}}{ \zeta_{ml} - \Tilde{\zeta}} \\
		&= \frac{1}{M(M-1)}  \frac{1}{N^2}\inp*{\sum_{m=1}^N\sum_{k=1}^N\zeta_{mk} - \Tilde{\zeta}}{\sum_{o=1}^M \sum_{l=1}^N \zeta_{ol} - \Tilde{\zeta}} - \frac{1}{M(M-1)}  \sum_{m=1}^M\inp*{\frac{1}{N}\sum_{k=1}^M\rbr*{\zeta_{mk} - \Tilde{\zeta}}}{\frac{1}{N}\sum_{l=1}^M\rbr*{ \zeta_{ml} - \Tilde{\zeta}}} \\
		&= -\frac{1}{M(M-1)}  \sum_{m=1}^M\norm*{\Tilde{\zeta}_m - \Tilde{\zeta}}^2 = -\frac{\tilde{\sigma}^2}{M-1}.
	\end{align*}
Therefore, these identities help us to establish the formula for the sample variance
	\begin{align}
    \label{eq:var_middle}
    \begin{split}
    		\E{\norm*{\Tilde{\zeta}^k_\pi - \Tilde{\zeta}}^2}
    		&= \frac{1}{k^2} \sum_{i=1}^k\sum_{j=1}^k \mathrm{cov}\sbr*{\zeta_{\pi_i}, \zeta_{\pi_j}} \\
        &= \frac{1}{k^2}\rbr*{\E{\sum_{i=1}^k \norm*{\zeta_{\pi_i} - \Tilde{\zeta}}^2} + \sum_{r=1}^{k}\sum_{s=1, \left\lfloor \frac{r}{N}\right\rfloor = \left\lfloor \frac{s}{N}\right\rfloor \;\&\; j_r \neq j_s}^{k} \mathrm{cov}\sbr*{\zeta_{\pi_r}, \zeta_{\pi_s}}
        + \sum_{r=1}^{k}\sum_{s=1, \left\lfloor \frac{r}{N}\right\rfloor \neq \left\lfloor \frac{s}{N}\right\rfloor}^{k} \mathrm{cov}\sbr*{\zeta_{\pi_r}, \zeta_{\pi_s}}} \\
    	&= \frac{1}{k^2}\rbr*{k\sigma^2 + \frac{m_k N^2 + j_k^2- k}{N-1}\rbr*{N\tilde{\sigma}^2 - \sigma^2} - \frac{k^2 - m_k N^2 - j_k^2 }{M-1} \tilde{\sigma}^2} \\
    	&= \frac{j_k (N - j_k)}{k^2(N-1)}\sigma^2 + \rbr*{\frac{(m_k N^2 + j_k^2) (MN-1)}{k^2(N-1)(M-1)} - \frac{N}{k(N-1)} - \frac{1}{M-1}} \tilde{\sigma}^2.
    \end{split}
	\end{align}
	
For the second part, we have 
\begin{align*}
    \E{\Tilde{\zeta}^k_{\pi, o}}  =  \E{\frac{1}{k} \rbr*{\sum_{i=1}^{k_N} \frac{1}{C}\sum_{p=1}^C \zeta^p_{\pi_i} + \sum_{i=k_N + 1}^{k} \zeta^o_{\pi_i}}} = 
    \frac{1}{k} \rbr*{\sum_{i=1}^{k_N} \frac{1}{C}\sum_{p=1}^C \Tilde{\zeta} + \sum_{i=k_N + 1}^{k} \Tilde{\zeta}} = \Tilde{\zeta}.
\end{align*}
Then, for the variance
\begin{align*}
    \E{\norm*{\Tilde{\zeta}^k_{\pi, o} - \Tilde{\zeta}}^2} &= E{\norm*{\frac{1}{k} \rbr*{\sum_{i=1}^{k_N} \frac{1}{C}\sum_{p=1}^C \rbr*{\zeta^p_{\pi_i} - \Tilde{\zeta}} + \sum_{i=k_N + 1}^{k} \rbr*{\zeta^o_{\pi_i} - \Tilde{\zeta}}} }^2} \\
    &=  \rbr*{\frac{k_N}{k}}^2 E{\norm*{\frac{1}{k_N C} \sum_{i=1}^{k_N} \sum_{p=1}^C \rbr*{\zeta^p_{\pi_i} - \Tilde{\zeta}}}^2} + \rbr*{\frac{k - k_N}{k}}^2 \E{\norm*{\frac{1}{k - k_N}\sum_{i=k_N + 1}^{k} \rbr*{\zeta^o_{\pi_i} - \Tilde{\zeta}}}^2} \\
    &\quad + \frac{2}{k^2C^2} \E{\dotprod{\sum_{i=1}^{k_N} \sum_{p=1}^C \rbr*{\zeta^p_{\pi_i} - \Tilde{\zeta}}}{\sum_{i=k_N + 1}^{k} \rbr*{\zeta^o_{\pi_i} - \Tilde{\zeta}}}} \\
    &=  \rbr*{\frac{k_N}{k}}^2 \Bar{\sigma}^2(k_N C) + \rbr*{\frac{k - k_N}{k}}^2 \Bar{\sigma}^2(k - k_N) - \frac{2 k_N (k - k_N)}{k^2 C (M-1)} \Tilde{\sigma}^2
\end{align*}
\end{proof}

Let us now analyse the obtained results. Firstly, one can notice that in the case $C=1$, \eqref{eq:minibatch_sampling_wo_replacement} is equivalent to \eqref{eq:sampling_wo_replacement} since $k_N = m_k N$ and $k-k_N = j_k$. Next, we link the obtained result with the existing works. In the special case $M=1$, we have $\Tilde{\sigma}^2 = 0$. Therefore, $\Bar{\sigma}^2(k) = \frac{k(N-k)}{k^2(N-1)} \sigma^2$ that recovers the variance bound of \citep[Lemma 1]{mishchenko2020random} for simple random reshuffling. In the full participation case, i.e., $M=C$, we have $k_N = m_k = 0$ and $ j_k=k$. Therefore, 
\begin{align*}
    \E{\norm*{\Tilde{\zeta}^k_\pi - \Tilde{\zeta}}^2} = \Bar{\sigma}^2(k) = \frac{k(N-k)}{k^2(N-1)} \sigma^2 + \frac{N}{N-1} \rbr*{1 - \frac{1}{k}} \Tilde{\sigma}^2 \leq \frac{N}{2k^2} \sigma^2 + 2\Tilde{\sigma}^2.
\end{align*}
The expression above can be used to recover the variance bound for full participation algorithm \algname{FedRR}~\citep[Theorem 1]{mishchenko2021proximal}. The next step of the analysis is to give an upper bound on the quantity $k^2\E{\norm*{\Tilde{\zeta}^k_\pi - \Tilde{\zeta}}^2}$, which is the key quantity that we use to bound the variance due to double reshuffling sampling procedure. 

\begin{lemma}
\label{lem:sampling_wo_replacement_upper}
    Let the settings of Lemma~\ref{lem:sampling_wo_replacement}, then 
    \begin{align}
        \label{eq:sampling_wo_replacement_upper}
        k^2\E{\norm*{\Tilde{\zeta}^k_\pi - \Tilde{\zeta}}^2} \leq \rbr*{\frac{M}{2C^2} + 2} N^2\tilde{\sigma}^2 + \frac{N}{2}\sigma^2
    \end{align}
\end{lemma}
\begin{proof}
    First, we recall the definition of $\E{\norm*{\Tilde{\zeta}^k_\pi - \Tilde{\zeta}}^2}$. We have
    \begin{align*}
        \E{\norm*{\Tilde{\zeta}^k_{\pi, o} - \Tilde{\zeta}}^2} =  \rbr*{\frac{k_N}{k}}^2 \Bar{\sigma}^2(k_N C) + \rbr*{\frac{k - k_N}{k}}^2 \Bar{\sigma}^2(k - k_N) - \frac{2 k_N (k - k_N)}{ C (M-1)} \Tilde{\sigma}^2.
    \end{align*}
    Since $k_N = m_k N$ and $k-k_N = j_k$, then 
    \begin{align*}
        k^2\E{\norm*{\Tilde{\zeta}^k_{\pi, o} - \Tilde{\zeta}}^2} &=  k_N^2 
        \rbr*{\frac{k_N C N(MN-1)}{k_N^2 C^2(N-1)(M-1)} - \frac{N}{k_N C (N-1)} - \frac{1}{M-1}} \tilde{\sigma}^2 \\
        &\quad + j_k^2 \rbr*{\frac{j_k (N - j_k)}{j_k^2(N-1)}\sigma^2 + \rbr*{\frac{j_k^2 (MN-1)}{j_k^2(N-1)(M-1)} - \frac{N}{j_k(N-1)} - \frac{1}{M-1}} \tilde{\sigma}^2} - \frac{2 k_N j_k}{k^2 C (M-1)} \Tilde{\sigma}^2 \\
        &= 
        \rbr*{\frac{k_N N(MN-1)}{C(N-1)(M-1)} - \frac{k_N N}{ C (N-1)} - \frac{k_N^2}{M-1}} \tilde{\sigma}^2 \\
        &\quad + \frac{j_k(N - j_k)}{(N-1)}\sigma^2 + \rbr*{\frac{j_k^2(MN-1)}{(N-1)(M-1)} - \frac{j_k N}{(N-1)} - \frac{j_k^2}{M-1}} \tilde{\sigma}^2 - \frac{2 k_N j_k}{C (M-1)} \Tilde{\sigma}^2 \\
        &= 
        \rbr*{\frac{k_N MN}{C(M-1)} - \frac{k_N^2}{M-1}} \tilde{\sigma}^2 + \frac{j_k(N - j_k)}{(N-1)}\sigma^2 
         + \rbr*{\frac{j_k^2M}{M-1} - \frac{j_k N}{(N-1)}} \tilde{\sigma}^2 - \frac{2 k_N j_k}{C (M-1)} \Tilde{\sigma}^2.
    \end{align*}
    To obtain the upper bound, we first use $\frac{j_k(N - j_k)}{(N-1)} \leq \frac{N}{2}$. Therefore,
    \begin{align*}
        k^2\E{\norm*{\Tilde{\zeta}^k_{\pi, o} - \Tilde{\zeta}}^2} &\leq
        \rbr*{\frac{k_N MN}{C(M-1)} - \frac{k_N^2}{M-1}} \tilde{\sigma}^2 + \frac{N}{2}\sigma^2 
         + \rbr*{\frac{j_k^2M}{M-1} - \frac{j_k N}{(N-1)}} \tilde{\sigma}^2 - \frac{2 k_N j_k}{C (M-1)} \Tilde{\sigma}^2.
    \end{align*}
    The first part of the first term is a quadratic function with respect to $k_N$, so we can estimate its maximum by equating its derivative to zero. For the term $\frac{j_k^2M}{M-1}$, we have $\frac{j_k^2M}{M-1} \leq 2N^2$ for $M \geq 2$. We ignore other negative terms. This yields the following upper bound
    \begin{align*}
        k^2\E{\norm*{\Tilde{\zeta}^k_{\pi, o} - \Tilde{\zeta}}^2} &\leq
        \rbr*{\frac{MN^2}{2C^2} + 2N^2} \tilde{\sigma}^2 + \frac{N}{2}\sigma^2,
    \end{align*}
    which concludes the proof. 
\end{proof}

In the next step, we will use this result to upper bound the following sequence
\begin{align}
    \label{eq:star_sequence_app}
    \begin{split}
    x^0_\star &= x^{0, 0}_{m,\star} = x_\star,\; \forall m \in S_t^{\lambda_0}\\
    x^{r,j+1}_{m,\star} &= x^{r,j}_{m,\star} - \gamma\nabla f^{\pi^j_{m}}_{m}\left( x_\star \right),\; \forall m \in [M], \\
    x^{r+1}_\star &= \frac{1}{C}\sum_{m\in S_t^\clper} x^{r,N}_{m,\star}. 
    \end{split}
\end{align}
Concretely, we are interested in upper bounding the distance of this sequence from the optimal solution $x_\star$ as this quantity will be useful to provide the upper bound for the statistical term in our convergence analysis. Note that $x^{r+1}_\star = x_\star$. Our result is summarized in the next lemma. 
\begin{lemma}
\label{lem:star_bound}
    The distance from the optimum of the iterates defined by \eqref{eq:star_sequence} is bounded, i.e.,
    \begin{align}
        \label{eq:star_bound}
        \E{\norm*{x^{r,j+1}_{m,\star} - x_\star}^2} \leq \gamma^2 \rbr*{\frac{MN^2}{2C^2} + 2N^2} \tilde{\sigma}^2 + \frac{\gamma^2 N}{2}\sigma^2, 
    \end{align}
    where $\tilde{\sigma}_\star^2 \eqdef \frac{1}{M} \sum_{m=1}^M \norm*{\nabla f_m(x_\star)}^2$ and $\sigma^2_\star \eqdef \frac{1}{MN} \sum_{m=1}^M \sum_{j=1}^N \norm*{\nabla f^j_m(x_\star)}^2.$
\end{lemma}

\begin{proof}
    Using \eqref{eq:star_sequence}, we get
    \begin{align*}
        \E{\norm*{x^{r,j+1}_{m,\star} - x_\star}^2} = \gamma^2 \E{\norm*{\frac{1}{C}\sum_{q=1}^{r-1} \sum_{z \in S_t^\clper} \sum_{i=1}^N \nabla  f^{\pi^i_{z}}_{z}\left( x_\star \right) + \sum_{l=0}^{j} \nabla  f^{\pi^j_{m}}_{m}\left( x_\star \right)}^2}
    \end{align*}
    This is exactly the setup of the double shuffling procedure discussed in Lemma~\ref{lem:sampling_wo_replacement}. Therefore, we can apply Lemma~\ref{lem:sampling_wo_replacement_upper} that yields
    \begin{align*}
        \E{\norm*{x^{r,j+1}_{m,\star} - x_\star}^2} \leq \gamma^2 \rbr*{\frac{MN^2}{2C^2} + 2N^2} \tilde{\sigma}^2 + \frac{\gamma^2 N}{2}\sigma^2, 
    \end{align*}
    where the corresponding value for $\tilde{\sigma}^2 = \frac{1}{M} \sum_{m=1}^M \norm*{\nabla f_m(x_\star)}^2$ and for $\sigma^2 = \frac{1}{MN} \sum_{m=1}^M \sum_{j=1}^N \norm*{\nabla f^j_m(x_\star)}^2$, which concludes the proof.
\end{proof}

Using the above results, we can upper bound both quantities, based on the Bregman divergence, that appear in our analysis, i.e., $\sigma^2_{m, \text{DS}}$ and $\sigma^2_{m, \text{CS}}$. The lemma follows.

\begin{lemma}
\label{lem:bregman_bound}
    The variance introduced by Algorithm~\ref{alg:pp-jumping} is bounded, i.e.,
    \begin{align}
        \label{eq:bregman_bound}
        \begin{split}
            \max_{m \in [M]} \sigma^2_{m, \text{DS}} &\leq L \rbr*{\frac{MN^2}{2C^2} + 2N^2} \tilde{\sigma}_\star^2 + \frac{LN}{2}\sigma_\star^2, \\
            \max_{m \in [M]} \sigma^2_{m, \text{CS}} &\leq \frac{LM}{2C^2}\Tilde{\sigma}_\star^2.
        \end{split}
    \end{align}
\end{lemma}

\begin{proof}
    First, we recall the definition of $\sigma^2_{m, \text{DS}}$ and apply the smoothness assumption.
    \begin{align*}
        \sigma^2_{m, \text{DS}} =  \frac{1}{\gamma^2}\Exp{D_{f_m^{\pi_j}}\left( x^{r,j}_{m,\star}, x_\star \right)} 
        \overset{\eqref{eq:lip-grad}}{\leq} \frac{L}{\gamma^2}\Exp{\norm*{x^{r,j}_{m,\star} - x_\star}^2}.
    \end{align*}
    Then, we apply \eqref{eq:star_bound} that gives desired result
    \begin{align*}
        \sigma^2_{m, \text{DS}} \leq L \rbr*{\frac{MN^2}{2C^2} + 2N^2} \tilde{\sigma}_\star^2 + \frac{LN}{2}\sigma_\star^2.
    \end{align*}
    For the second term, we again use its definition and the smoothness assumption. Therefore, 
    \begin{align*}
        \sigma^2_{m, \text{CS}} =  \frac{1}{\eta^2}\Exp{D_{f_m}\left( x^{r}_{m,\star}, x_\star \right)} 
        \overset{\eqref{eq:lip-grad}}{\leq} \frac{L}{\eta^2}\Exp{\norm*{x^{r}_{m,\star} - x_\star}^2}.
    \end{align*}
    This setup is also reflected in Lemma~\ref{lem:sampling_wo_replacement}, but with $N=1$. Therefore, $\sigma^2 = \Tilde{\sigma^2}$. Furthermore, we can't apply \eqref{eq:sampling_wo_replacement} directly as it is not defined for $N=1$. We can instead derive the variance bound for the case $N=1$ using the proof techniques provided in the proof of Lemma~\ref{lem:sampling_wo_replacement}, where we ignore the middle term in \eqref{eq:var_middle} as the number of summands satisfying condition $s=1, \left\lfloor \frac{r}{N}\right\rfloor \neq \left\lfloor \frac{s}{N}\right\rfloor$ is zero for $N=1$. This, combined with the fact that $m_k = k$ and $j_k = 0$, yields  
    The above result together with \eqref{eq:star_bound} yields $\Bar{\sigma}^2(k) = \frac{\Tilde{\sigma}^2 (M-k)}{k (M-1)}$ for $k \in [\frac{M}{C}]$. We can plug this equality to \eqref{eq:minibatch_sampling_wo_replacement}, where $k_1 = k$. Therefore,
    \begin{align*}
        \sigma^2_{m, \text{DS}} &\leq L \frac{L}{\gamma^2}\Exp{\norm*{x^{r}_{m,\star} - x_\star}^2} \leq L k^2\Bar{\sigma}^2(kC) \\
        &= \frac{L\Tilde{\sigma}_\star^2 k(M-kC)}{C(M-1)} = \frac{L\Tilde{\sigma}_\star^2}{C^2} \frac{kC(M-kC)}{kC(M-1)} \leq \frac{L\Tilde{\sigma}_\star^2M}{2C^2},
    \end{align*}
    which concludes the proof.
\end{proof}
Equipped with the bounds for the variance terms, we are ready to proceed with the exact convergence bounds.

\section{Proof in case of $f_{m}^j$'s strong convexity}

In this section, we need to use the following equalities to prove convergence guarantees:
 \begin{align*}
x^{r+1}_t &= \frac{1}{C}\sum_{m\in S_t^\clper} x^{r,n}_{m,t}, \\
x^{r,j+1}_{m,t} &= x^{r,j}_{m,t} - \gamma\nabla f_m^{\pi_j}\left( x^{r,j}_{m,t} \right), \quad\quad x^{r,j}_{m,t} = x^{r}_{t} - \gamma\sum_{k=1}^{j-1}\nabla f_m^{\pi_k}\left( x^{r,k}_{m,t} \right),\\
x^{r+1}_t &= x^r_t - \gamma \frac{1}{C}\sum_{m\in S_t^\clper}\sum_{j=1}^{N-1}\nabla f_m^{\pi_j}\left(x^{r,j}_{m,t}\right).
 \end{align*}
These equations are necessary for one-step, local, and meta-epoch analysis. Let us start from the case when all individual functions are $\mu$-strongly convex. 
\begin{theorem}
	Suppose that the functions $f_1, \ldots , f_n$ are $\mu$-strongly convex and $L$-smooth. Then for Algorithm 1 with a constant stepsize $\gamma\leq \frac{1}{L}$, the iterates generated by either of the algorithm satisfy
	\begin{align}
			\Exp{\mynorm{x_{T} - x_\star }  }	&\leq (1-\gamma\mu)^{NRT} \mynorm{ x_0 - x_{\star} }  +\frac{2\gamma^2}{\mu}\max_{r,m}\sigma^2_{m,\text{DS}}. 
	\end{align}

\end{theorem}
\begin{proof}
	We start our proof by analyzing the distance between intermediate point $x^{r,j+1}_{m,t}$ and a point of auxiliary sequence:
	\begin{align}
		\label{eq:sdsdsds}
		\Exp{ \mynorm{ x^{r,j+1}_{m,t} - x^{r,j+1}_{m,\star} } }&= \Exp{\mynorm{x^{r,j}_{m,t} - \gamma\nabla f^{\pi_j}_m\left(x^{r,j}_{m,t}\right) - \left(x^{r,j}_{m,\star} -\gamma\nabla f_m^{\pi_j}\left(x_\star\right)\right)  }}\notag\\
		& = \Exp{\mynorm{x^{r,j}_{m,t} - x^{r,j}_{m,\star} - \gamma\left(\nabla f^{\pi_j}_m\left(x^{r,j}_{m,t}\right) -\gamma\nabla f_m^{\pi_j}\left(x_\star\right)\right)  }}\notag\\
		&= \Exp{ \mynorm{ x^{r,j}_{m,t} - x^{r,j}_{m,\star} }  - 2\gamma\left\langle \nabla f_m^{\pi_j} \left( x^{r,j}_{m,t} \right)- \nabla f^{\pi_j}_{m}\left(x_\star\right), x^{r,j}_{m,t}  - x^{r,j}_{m,\star}   \right\rangle}\notag\\
		&\quad + \gamma^2 \Exp{\mynorm{\nabla f_m^{\pi_j}\left(x^{r,j}_{m,t}  \right)-\nabla f_m^{\pi_j}\left(x_\star  \right)}}.
	\end{align}
Using a three-point identity, we have 
\begin{align}
	\label{eq:3point-1}
	\left\langle \nabla f_m^{\pi_j} \left( x^{r,j}_{m,t} \right)- \nabla f^{\pi_j}_{m}\left(x_\star\right), x^{r,j}_{m,t}  - x^{r,j}_{m,\star}   \right\rangle = 
	D_{f_m^{\pi_j}}\left(  x^{r,j}_{m,\star}, x^{r,j}_{m,t}   \right) +	D_{f_m^{\pi_j}}\left(  x^{r,j}_{m,t}, x_\star  \right) - 	D_{f_m^{\pi_j}}\left(  x^{r,j}_{m,\star}, x_\star \right).
\end{align}
Plugging~\eqref{eq:3point-1} into \eqref{eq:sdsdsds} we obtain
\begin{align*}
		\Exp{ \mynorm{ x^{r,j+1}_{m,t} - x^{r,j+1}_{m,\star} } }&= \Exp{ \mynorm{ x^{r,j}_{m,t} - x^{r,j}_{m,\star} } } - 2\gamma\Exp{	D_{f_m^{\pi_j}}\left(  x^{r,j}_{m,\star}, x^{r,j}_{m,t}   \right) } - 2\gamma\Exp{ D_{f_m^{\pi_j}}\left(  x^{r,j}_{m,t}, x_\star  \right) }\\
		&\quad+	2\gamma\Exp{D_{f_m^{\pi_j}}\left(  x^{r,j}_{m,\star}, x_\star \right)}+ \gamma^2 \Exp{\mynorm{\nabla f_m^{\pi_j}\left(x^{r,j}_{m,t}  \right)-\nabla f_m^{\pi_j}\left(x_\star  \right)}}.
\end{align*}
Using $\mu$-convexity and $L$-smoothness, we have 
\begin{align*}
	\frac{\mu}{2}\mynorm{x^{r,j}_{m,\star} - x^{r,j}_{m,t} }	& \leq D_{f_m^{\pi_j}}\left(  x^{r,j}_{m,\star}, x^{r,j}_{m,t}   \right),\\
	\Exp{\mynorm{\nabla f_m^{\pi_j}\left(x^{r,j}_{m,t}  \right)-\nabla f_m^{\pi_j}\left(x_\star  \right)}} &\leq 2 L D_{f_m^{\pi_j}}\left(  x^{r,j}_{m,t}, x_\star  \right).
\end{align*}
Using this inequality, we have 
\begin{align*}
	\Exp{ \mynorm{ x^{r,j+1}_{m,t} - x^{r,j+1}_{m,\star} } }&\leq (1-\gamma\mu)\Exp{ \mynorm{ x^{r,j}_{m,t} - x^{r,j}_{m,\star} } } - 2\gamma\Exp{ D_{f_m^{\pi_j}}\left(  x^{r,j}_{m,t}, x_\star  \right) }\\
	&\quad+	2\gamma\Exp{D_{f_m^{\pi_j}}\left(  x^{r,j}_{m,\star}, x_\star \right)}+ 2L\gamma^2 D_{f_m^{\pi_j}}\left(  x^{r,j}_{m,t}, x_\star  \right)\\
	&\leq(1-\gamma\mu)\Exp{ \mynorm{ x^{r,j}_{m,t} - x^{r,j}_{m,\star} } } +	2\gamma\Exp{D_{f_m^{\pi_j}}\left(  x^{r,j}_{m,\star}, x_\star \right)}\\
	&\quad - 2\gamma\left(1-\gamma L\right)\Exp{ D_{f_m^{\pi_j}}\left(  x^{r,j}_{m,t}, x_\star  \right) }.
\end{align*}
Using $\gamma\leq \frac{1}{L}$ and definition of $\sigma^2_{m,\text{CS}}$, we get the following bound:
\begin{align*}
		\Exp{ \mynorm{ x^{r,j+1}_{m,t} - x^{r,j+1}_{m,\star} } }	\leq(1-\gamma\mu)\Exp{ \mynorm{ x^{r,j}_{m,t} - x^{r,j}_{m,\star} } } +	2\gamma^3\sigma^2_{m,\text{CS}}.
\end{align*}
Unrolling this recursion, we obtain
\begin{align*}
	\Exp{ \mynorm{ x^{r,n}_{m,t} - x^{r,n}_{m,\star} } }&	\leq(1-\gamma\mu)^N\Exp{ \mynorm{ x^{r,0}_{m,t} - x^{r,0}_{m,\star} } } + 	2\gamma^3\sigma^2_{m,\text{CS}}\sum_{j=0}^{N-1} (1-\gamma\mu)^j\\
&	=(1-\gamma\mu)^N\Exp{ \mynorm{ x^{r}_{t} - x^{r}_{\star} } } +2\gamma^3\sigma^2_{m,\text{CS}}\sum_{j=0}^{N-1} (1-\gamma\mu)^j.	
\end{align*}
Now we need to establish recursion for rounds of communication:
\begin{align*}
\Exp{\mynorm{x^{r+1}_t - x^{r+1}_\star }  } &= \Exp{\mynorm{\frac{1}{C}\sum_{m\in S^\clper_t} x^{r,n}_{m,t}  - \frac{1}{C}\sum_{m\in S^\clper_t} x^{r,n}_{m,\star} }  } \\
&\leq \frac{1}{C}\sum_{m\in S^\clper_t}\Exp{\mynorm{ x^{r,n}_{m,t}  - x^{r,n}_{m,\star} }  }\\
&\leq  \frac{1}{C}\sum_{m\in S^\clper_t}\left( (1-\gamma\mu)^N\Exp{ \mynorm{ x^{r}_{t} - x^{r}_{\star} } } +2\gamma^3\sigma^2_{m,\text{CS}}\sum_{j=0}^{N-1} (1-\gamma\mu)^j	\right)\\
& =   (1-\gamma\mu)^N\Exp{ \mynorm{ x^{r}_{t} - x^{r}_{\star} } } +2\gamma^3\max_m\sigma^2_{m,\text{CS}}\sum_{j=0}^{N-1} (1-\gamma\mu)^j.
\end{align*}
Using the fact that $x_{t+1} = x^R_t$ and $x_\star=x^R_\star$, we can unroll this recursion again for index $r$:
\begin{align*}
	\Exp{\mynorm{x_{t+1} - x_\star }  }  &= 	\Exp{\mynorm{x^R_{t} - x^R_\star }  }\\
	&\leq  (1-\gamma\mu)^{NR}\Exp{ \mynorm{ x_t - x_{\star} } } +2\gamma^3\max_{r,m}\sigma^2_{m,\text{DS}}\sum_{j=0}^{N-1} (1-\gamma\mu)^j\sum_{r=0}^{R-1} (1-\gamma\mu)^{rN}.
\end{align*}
Unrolling this recursion again for index $t$ and using tower property, we have 
\begin{align*}
		\Exp{\mynorm{x_{T} - x_\star }  } &\leq  (1-\gamma\mu)^{NRT} \mynorm{ x_0 - x_{\star} } \\
		&\quad +2\gamma^3\max_{r,m}\sigma^2_{m,\text{DS}}\sum_{j=0}^{N-1} (1-\gamma\mu)^j\sum_{r=0}^{R-1} (1-\gamma\mu)^{rN}\sum_{t=0}^{T-1} (1-\gamma\mu)^{tRN}\\
		&\leq (1-\gamma\mu)^{NRT} \mynorm{ x_0 - x_{\star} }  +2\gamma^3\sum_{j=0}^{NRT-1}(1-\gamma\mu)^j\max_{r,m}\sigma^2_{m,\text{DS}}\\
		&\leq (1-\gamma\mu)^{NRT} \mynorm{ x_0 - x_{\star} }  +\frac{2\gamma^2}{\mu}\max_{r,m}\sigma^2_{m,\text{DS}}.
\end{align*}

\end{proof}

\section{Proof in case of $f_i$'s strong convexity}
This section provides convergence bounds for the case when only $f_m$ is $\mu$-strongly convex. In this regime, we cannot use the trick with the additional sequence as we do in the previous section. Due to the biased nature of gradient updates, we cannot take expectations directly, and we need to get an error bound of gradient approximation. Formally, we have a lemma for this below. 

\begin{lemma}
		\label{lemma:diff_g_t}
	Assume that each $f^j_{m}$ is $L$-smooth, then we have 
	\begin{align*}
			\Exp{	\mynorm{g_t^{r} - \frac{1}{C}\sum_{m\in S^\clper_t}\nabla f_m(x^r_t)}} \leq L^2 V_t^r, 
	\end{align*}
where $V_t^r $ is defined as
$$ V_t^r =  \frac{1}{M}\sum_{m=1 }^{M}\frac{1}{N}\sum_{j=0}^{N-1}\mynorm{x^{r,j}_{m,t} -x^r_t}.$$
\end{lemma}
\begin{proof}
	We start from definition of $g_t^{r} $ and then we apply Jensen's inequality and $L$-smooth assumption:
	\begin{align*}
	\Exp{	\mynorm{g_t^{r} - \frac{1}{C}\sum_{m\in S^\clper_t}\nabla f_m(x^r_t)}} &= \Exp{\mynorm{\frac{1}{C}\sum_{m\in S^\clper_t}\frac{1}{N}\sum_{j=0}^{N-1}\nabla f_{m}^\dtper(x^{r,j}_{m,t}) - \frac{1}{C}\sum_{m\in S^\clper_t}\nabla f_m(x^r_t)}}\\
		&=\Exp{\mynorm{\frac{1}{C}\sum_{m\in S^\clper_t}\frac{1}{N}\sum_{j=0}^{N-1}\nabla f_{m}^\dtper(x^{r,j}_{m,t}) - \frac{1}{C}\sum_{m\in S^\clper_t}\frac{1}{N}\sum_{j=0}^{N-1}\nabla f_{m}^\dtper(x^r_t)}}\\
		&=\Exp{\mynorm{\frac{1}{C}\sum_{m\in S^\clper_t}\frac{1}{N}\sum_{j=0}^{N-1}\left(\nabla f_{m}^\dtper(x^{r,j}_{m,t}) -\nabla f_{m}^\dtper(x^r_t)\right)}}\\
		&\overset{\text{Y}}{\leq}  \Exp{\frac{1}{C}\sum_{m\in S^\clper_t}\frac{1}{N}\sum_{j=0}^{N-1}\mynorm{\nabla f_{m}^\dtper(x^{r,j}_{m,t}) -\nabla f_{m}^\dtper(x^r_t)}}\\
		&\overset{\text{L}}{\leq} L^2 \frac{1}{M}\sum_{m=1 }^M\frac{1}{N}\sum_{j=0}^{N-1}\mynorm{x^{r,j}_{m,t} -x^r_t}\\
		&= L^2 V_t^r.
	\end{align*}
\end{proof}
We manage to get an error bound using the sum of distances between intermediate point $x^{r,j}_{m,t}$ and starting point  $x^{r}_t$. Now we need to provide bounds for such sums $V_t$. The following lemma does it formally. 
\begin{lemma}
		Assume that each $f_{m,j}$ is $L$-smooth, then we have 
	\begin{equation}
	\Exp{V^r_t} \leq 8\gamma^2n^2L \frac{1}{M}\sum_{m=1}^{M}D_{f_m}\left(x^{r}_t,x_\star\right) + 2\gamma^2n^2\frac{1}{M}\sum_{m=1}^{M}\left\| \nabla f_m(x_\star) \right\|^2+2\gamma^2n\frac{1}{M}\sum_{m=1}^{M}\frac{1}{N}\sum_{j=0}^{N-1}\left\| \nabla f^\dtper_{m} (x_\star)\right\|^2.
	\end{equation}
\begin{proof}
	We start with the update rule:
	\begin{align*}
		x^{r,j}_{m,t} = x^{r}_{t} - \gamma\sum_{k=0}^{j-1}\nabla f^{\pi_j}_{m} \left(x^{r,k}_{m,t}\right).
	\end{align*}
Using this form, we get
\begin{align*}
	\E{\left\| x^{r,j}_{m,t} - x^{r}_t \right\|^2} &= \gamma^2\E{\mynorm{\sum^{j-1}_{k=0}\nabla f^{\pi_k}_{m} \left(x_{m,t}^{r,k}\right)}}\\
	&\leq 2\gamma^2\E{\mynorm{\sum_{k=0}^{j-1}\left( \nabla f^{\pi_k}_{m} \left(x^{r,k}_{m,t} \right) - \nabla f^{\pi_k}_{m} \left(x^{r}_t  \right)\right)}}+2\gamma^2\E{\mynorm{\sum_{k=0}^{j-1} \nabla f^{\pi_k}_{m} \left(x^{r}_{t}\right)}}\\
	&\leq 2\gamma^2j\sum_{k=0}^{j-1}\E{\mynorm{ \nabla f^{\pi_k}_{m} \left(x^{r,k}_{m,t} \right) -\nabla f^{\pi_k}_{m} \left(x^{r}_t  \right)}}+2\gamma^2\E{\mynorm{\sum_{k=0}^{j-1} \nabla \nabla f^{\pi_k}_{m} \left(x^{r}_t\right)}}\\
	& \leq 2\gamma^2j L^2 \sum_{k=0}^{j-1}\E{\mynorm{  x^{r,k}_{m,t}  - x^{r}_t  }}+2\gamma^2\E{\mynorm{\sum_{k=0}^{j-1} \nabla f^{\pi_k}_{m} \left(x^{r}_t\right)}}\\
	& \leq 2\gamma^2j L^2 \sum_{k=0}^{j-1}\E{\mynorm{  x^{r,k}_{m,t}  - x^{r}_t  }}+2\gamma^2\left(j^2\mynorm{\nabla f^\clper_m(x^r_t)} + \frac{j(N-j)}{N-1}\left( \sigma_{m,t} \right)^2\right).
\end{align*}
Now we can sum these inequalities:
\begin{align*}
	\frac{1}{M}\sum_{m=1}^{M}\frac{1}{N}\sum_{j=0}^{N-1} 	\E{\left\| x^{r,j}_{m,t} - x^{r}_t \right\|^2} &\leq 2\gamma^2 L^2	\frac{1}{M}\sum_{m=1}^{M}\frac{1}{N}\sum_{j=0}^{N-1}  j  \sum_{k=0}^{j-1}\E{\mynorm{  x^{r,k}_{m,t}  - x^{r}_t  }}\\
	&\quad+2\gamma^2 \frac{1}{M}\sum_{m=1}^{M}\frac{1}{N}\sum_{j=0}^{N-1} \left(j^2\mynorm{\nabla f_m(x^r_t)} + \frac{j(N-j)}{N-1}\left( \sigma_{m,t} \right)^2\right).
\end{align*}
\end{proof}

\end{lemma}

Now we are ready to prove the theorem.

\begin{theorem}
	Assume that each $f_m$ is $\mu$-strongly convex. Also, assume that each $f_{m}^j$ is convex and $L$-smooth. Let $\eta \leq \frac{1}{4L}$ and $\gamma\leq \frac{1}{8NL\sqrt{\kappa}}$, then for iterates generated by Algorithm~\ref{alg:pp-jumping}, we have
	\begin{align*}
		\Exp{\|x_{T} - x_\star\|^2} &\leq (1-\eta\mu)^{RT}\| x_0 - x_\star \|^2+\frac{4}{\mu}\eta^2\max_{r,m}\sigma^2_{m,\text{CS}}+12\kappa^2\gamma^2 N^2\sigma^2_\star+12\gamma^2\kappa^2N\sigma^2_\star,
	\end{align*}
	where $\sigma^2_{m,\text{CS}} = \frac{1}{\eta^2}\Exp{D_{f_{m}}\left(x_\star^r, x_\star\right)}.$
\end{theorem}

\begin{proof}
	We start from the following equations:
	\begin{align*}
		x^{r+1}_t &= x^r_t - \eta \frac{1}{C}\sum_{m\in S^\clper_t}\frac{1}{N}\sum_{j=0}^{N-1}\nabla f^{\dtper}_m \left( x^{r,j}_{m,t} \right) = x^r_t - \eta g^r_t,\\
		g^r_t &= \frac{1}{C}\sum_{m\in S^\clper_t}\frac{1}{N}\sum_{j=0}^{N-1}\nabla f^\dtper_m(x^{r,j}_{m,t}),\\
		x^r_\star &= x_\star - \eta \sum_{r=0}^{R-1} \frac{1}{C} \sum_{m\in S^\clper_t}\nabla f_m \left( x^{r,j}_{m,t} \right),\quad x^{r+1}_\star = x^r_\star - \eta \frac{1}{C}\sum_{m\in S^\clper_t}\nabla f_m(x_\star).
	\end{align*}
We start from the distance to the solution:
\begin{align*}
	\Exp{\mynorm{x^{r+1}_t - x^{r+1}_\star}} &= \Exp{\mynorm{ x^r_t - \eta g^r_t - \left(x^r_\star -\eta \frac{1}{C}\sum_{m\in S_t^\clper} \nabla f_m(x_*) \right) }}	\\
	&=\Exp{  \mynorm{x^r_t - x^r_\star} - 2\eta \left\langle g^r_t - \frac{1}{C}\sum_{m\in S_t^\clper}\nabla f_m(x_\star),x^r_t - x_\star^r \right\rangle +\eta^2\mynorm{g^r_t - \frac{1}{C} \sum_{m\in S_t^\clper} \nabla f_m(x_\star)  }}\\
	&=\Exp{  \mynorm{x^r_t - x^r_\star}+ \eta^2\mynorm{g^r_t + \frac{1}{C} \sum_{m\in S_t^\clper} \nabla f_m(x^r_t) - \frac{1}{C} \sum_{m\in S_t^\clper} \nabla f_m(x^r_t) - \frac{1}{C} \sum_{m\in S_t^\clper} \nabla f_m(x_\star)  }}\\
	& \quad- 2\eta \Exp{ \left\langle g^r_t+ \frac{1}{C} \sum_{m\in S_t^\clper} \nabla f_m(x^r_t)-\frac{1}{C} \sum_{m\in S_t^\clper} \nabla f_m(x^r_t) - \frac{1}{C}\sum_{m\in S_t^\clper}\nabla f_m(x_\star),x^r_t - x_\star^r \right\rangle }.
	\end{align*}
Using Young's inequality we have 
\begin{align*}
	\Exp{\mynorm{x^{r+1}_t - x^{r+1}_\star}} 	&=\Exp{  \mynorm{x^r_t - x^r_\star}+ \eta^2\mynorm{g^r_t + \frac{1}{C} \sum_{m\in S_t^\clper} \nabla f_m(x^r_t) - \frac{1}{C} \sum_{m\in S_t^\clper} \nabla f_m(x^r_t) - \frac{1}{C} \sum_{m\in S_t^\clper} \nabla f_m(x_\star)  }}\\
	& \quad- 2\eta \Exp{ \left\langle g^r_t+ \frac{1}{C} \sum_{m\in S_t^\clper} \nabla f_m(x^r_t)-\frac{1}{C} \sum_{m\in S_t^\clper} \nabla f_m(x^r_t) - \frac{1}{C}\sum_{m\in S_t^\clper}\nabla f_m(x_\star),x^r_t - x_\star^r \right\rangle }.
\end{align*}
Applying Young's inequality again, we have
\begin{align*}
	\Exp{\mynorm{x^{r+1}_t - x^{r+1}_\star}} 	&\leq \Exp{  \mynorm{x^r_t - x^r_\star}+2\eta^2\mynorm{g^r_t - \frac{1}{C} \sum_{m\in S_t^\clper} \nabla f_m(x^r_t) }}\\
	&\quad+2 \eta^2\Exp{\mynorm{ \frac{1}{C} \sum_{m\in S_t^\clper} \nabla f_m(x^r_t)  - \frac{1}{C} \sum_{m\in S_t^\clper} \nabla f_m(x_\star)  }}\\
	& \quad- 2\eta \Exp{ \left\langle g^r_t-\frac{1}{C} \sum_{m\in S_t^\clper} \nabla f_m(x^r_t),x^r_t - x_\star^r \right\rangle }\\
	& \quad- 2\eta \Exp{ \left\langle  \frac{1}{C} \sum_{m\in S_t^\clper} \nabla f_m(x^r_t) - \frac{1}{C}\sum_{m\in S_t^\clper}\nabla f_m(x_\star),x^r_t - x_\star^r \right\rangle }.
\end{align*}
Using Young's inequality for inner product $-2\left\langle a,b \right\rangle \leq \frac{1}{t}\|a\|^2 + t\|b\|^2$ we get
\begin{align*}
		\Exp{\mynorm{x^{r+1}_t - x^{r+1}_\star}} &\leq \Exp{  \mynorm{x^r_t - x^r_\star}+2\eta^2\mynorm{g^r_t - \frac{1}{C} \sum_{m\in S_t^\clper} \nabla f_m(x^r_t) }}\\
		&\quad+2 \eta^2\Exp{\mynorm{ \frac{1}{C} \sum_{m\in S_t^\clper} \nabla f_m(x^r_t)  - \frac{1}{C} \sum_{m\in S_t^\clper} \nabla f_m(x_\star)  }}\\
		& \quad +\eta \Exp{ \frac{\mu}{2}\left\| x^r_t - x^r_\star \right\|^2 + \frac{2}{\mu}\left\| g^r_t - \frac{1}{C} \sum_{m \in S_t^\clper} \nabla f_m(x_\star) \right\|^2 }\\
		& \quad- 2\eta \Exp{ \left\langle  \frac{1}{C} \sum_{m\in S_t^\clper} \nabla f_m(x^r_t) - \frac{1}{C}\sum_{m\in S_t^\clper}\nabla f_m(x_\star),x^r_t - x_\star^r \right\rangle }.
\end{align*}
Rearraging terms leads to
\begin{align}
	\label{eq:qqqq}
	\Exp{\mynorm{x^{r+1}_t - x^{r+1}_\star}} &\leq \Exp{  \left(1+\frac{\mu\eta}{2}\right)\mynorm{x^r_t - x^r_\star}+\left(2\eta^2+2\frac{\eta}{\mu}\right)\mynorm{g^r_t - \frac{1}{C} \sum_{m\in S_t^\clper} \nabla f_m(x^r_t) }}\\
	&\quad+2 \eta^2\Exp{\mynorm{ \frac{1}{C} \sum_{m\in S_t^\clper} \nabla f_m(x^r_t)  - \frac{1}{C} \sum_{m\in S_t^\clper} \nabla f_m(x_\star)  }}\\
	& \quad- 2\eta \Exp{ \left\langle  \frac{1}{C} \sum_{m\in S_t^\clper} \nabla f_m(x^r_t) - \frac{1}{C}\sum_{m\in S_t^\clper}\nabla f_m(x_\star),x^r_t - x_\star^r \right\rangle }.
\end{align}
Let us consider the last term:
\begin{align*}
	 \left\langle  \frac{1}{C} \sum_{m\in S_t^\clper} \nabla f_m(x^r_t) - \frac{1}{C}\sum_{m\in S_t^\clper}\nabla f_m(x_\star),x^r_t - x_\star^r \right\rangle &=   \left\langle  \frac{1}{C} \sum_{m\in S_t^\clper} \left(\nabla f_m(x^r_t) - \nabla f_m(x_\star)\right),x^r_t - x_\star^r \right\rangle \notag\\
	 & =  \frac{1}{C} \sum_{m\in S_t^\clper} \left\langle  \nabla f_m(x^r_t) - \nabla f_m(x_\star),x^r_t - x_\star^r \right\rangle.
	\end{align*}
Let us look at inner product and use three-point identity:
\begin{align}
	\label{eq:3-point-1}
	\left\langle  \nabla f_m(x^r_t) - \nabla f_m(x_\star),x^r_t - x_\star^r \right\rangle =  D_{f_{m}}\left(x_\star^r, x_t^r\right)+D_{f_{m}}\left(x_t^r, x_\star\right)-D_{f_{m}}\left(x_\star^r, x_\star\right).
\end{align} 
By $\mu$-strong convexity of $f_m$, the first term in \eqref{eq:3-point-1} satisfies
\begin{align}
	\label{eq:17}
	\frac{\mu}{2}\left\|x_t^r-x_\star^r\right\|^2 \stackrel{(2)}{\leq} D_{f_{m}}\left(x_\star^r, x_t^r\right).
\end{align}
Combining \eqref{eq:17},\eqref{eq:18} and using in \eqref{eq:3-point-1} we have
\begin{align}
	\label{eq:fin-n}
-2\eta \Exp{ \frac{1}{C} \sum_{m\in S_t^\clper}	\left\langle  \nabla f_m(x^r_t) - \nabla f_m(x_\star),x^r_t - x_\star^r \right\rangle}&= -2\eta\Exp{ \frac{1}{C}  \sum_{m\in S_t^\clper}  D_{f_{m}}\left(x_\star^r, x_t^r\right)+D_{f_{m}}\left(x_t^r, x_\star\right)} \notag\\
&\quad+2\eta\Exp{\frac{1}{C} \sum_{m\in S_t^\clper}  D_{f_{m}}\left(x_\star^r, x_\star\right)}\notag\\
&\leq - \eta\mu \Exp{\|x^r_t - x^r_\star\|^2}\notag\\
&\quad - 2\eta \Exp{	 \frac{1}{C}\sum_{m\in S_t^\clper} D_{f_{m}}\left(x_t^r, x_\star\right)}\notag\\
&\quad+2\eta\Exp{ \frac{1}{C} \sum_{m\in S_t^\clper}   D_{f_{m}}\left(x_\star^r, x_\star\right)}.
\end{align}
Applying \eqref{eq:fin-n} in \eqref{eq:qqqq} we obtain
\begin{align*}
	\Exp{\mynorm{x^{r+1}_t - x^{r+1}_\star}} &\leq \Exp{  \left(1+\frac{\mu\eta}{2}\right)\mynorm{x^r_t - x^r_\star}+\left(2\eta^2+2\frac{\eta}{\mu}\right)\mynorm{g^r_t - \frac{1}{C} \sum_{m\in S_t^\clper} \nabla f_m(x^r_t) }}\\
	&\quad+2 \eta^2\Exp{\mynorm{ \frac{1}{C} \sum_{m\in S_t^\clper} \nabla f_m(x^r_t)  - \frac{1}{C} \sum_{m\in S_t^\clper} \nabla f_m(x_\star)  }}- \eta\mu \Exp{\|x^r_t - x^r_\star\|^2}\notag\\
	&\quad - 2\eta \Exp{	 \frac{1}{C}\sum_{m\in S_t^\clper} D_{f_{m}}\left(x_t^r, x_\star\right)}+2\eta\Exp{ \frac{1}{C} \sum_{m\in S_t^\clper}   D_{f_{m}}\left(x_\star^r, x_\star\right)}.
	\end{align*}
We can bound first term in second line using $L$-smoothness:
\begin{align}
	\frac{1}{2 L}\left\|\nabla f_{m}\left(x_t^r\right)-\nabla f_{m}\left(x_\star\right)\right\|^2 \stackrel{(3)}{\leq} D_{f_{m}}\left(x_t^r, x_\star\right).
	\label{eq:18}
\end{align}
Applying \eqref{eq:18} we get 
\begin{align*}
		\Exp{\mynorm{x^{r+1}_t - x^{r+1}_\star}}	&\leq \Exp{  \left(1+\frac{\mu\eta}{2}\right)\mynorm{x^r_t - x^r_\star}+\left(2\eta^2+2\frac{\eta}{\mu}\right)\mynorm{g^r_t - \frac{1}{C} \sum_{m\in S_t^\clper} \nabla f_m(x^r_t) }}\\
	&\quad+4 L \eta^2 \Exp{	 \frac{1}{C}\sum_{m\in S_t^\clper} D_{f_{m}}\left(x_t^r, x_\star\right)}- \eta\mu \Exp{\|x^r_t - x^r_\star\|^2}\notag\\
	&\quad - 2\eta \Exp{	 \frac{1}{C}\sum_{m\in S_t^\clper} D_{f_{m}}\left(x_t^r, x_\star\right)}+2\eta\Exp{ \frac{1}{C} \sum_{m\in S_t^\clper}   D_{f_{m}}\left(x_\star^r, x_\star\right)}.
\end{align*}
Rearranging the terms we have 
\begin{align*}
	\Exp{\mynorm{x^{r+1}_t - x^{r+1}_\star}} &\leq \Exp{  \left(1-\frac{\mu\eta}{2}\right)\mynorm{x^r_t - x^r_\star}+\left(2\eta^2+2\frac{\eta}{\mu}\right)\mynorm{g^r_t - \frac{1}{C} \sum_{m\in S_t^\clper} \nabla f_m(x^r_t) }}\\
&\quad+2\eta\left(2 \eta L -1 \right) \Exp{	 \frac{1}{C}\sum_{m\in S_t^\clper} D_{f_{m}}\left(x_t^r, x_\star\right)}+2\eta\Exp{ \frac{1}{C} \sum_{m\in S_t^\clper}   D_{f_{m}}\left(x_\star^r, x_\star\right)}\\
	&=\Exp{  \left(1-\frac{\mu\eta}{2}\right)\mynorm{x^r_t - x^r_\star}+2\eta\left(\eta+\frac{1}{\mu}\right)\mynorm{g^r_t - \frac{1}{C} \sum_{m\in S_t^\clper} \nabla f_m(x^r_t) }}\\
	&\quad+2\eta\left(2 \eta L -1 \right) \Exp{	 \frac{1}{C}\sum_{m\in S_t^\clper} D_{f_{m}}\left(x_t^r, x_\star\right)}+2\eta\Exp{ \frac{1}{C} \sum_{m\in S_t^\clper}   D_{f_{m}}\left(x_\star^r, x_\star\right)}.
\end{align*}
Since $L\geq \mu$ and $\eta \leq \frac{1}{2L}$ we have $\eta \leq \frac{1}{2\mu} $. Using this we get
\begin{align*}
	\Exp{\mynorm{x^{r+1}_t - x^{r+1}_\star}} &\leq \Exp{  \left(1-\frac{\mu\eta}{2}\right)\mynorm{x^r_t - x^r_\star}+\frac{3\eta}{\mu}\mynorm{g^r_t - \frac{1}{C} \sum_{m\in S_t^\clper} \nabla f_m(x^r_t) }}\\
	&\quad+2\eta\left(2 \eta L -1 \right) \Exp{	 \frac{1}{C}\sum_{m\in S_t^\clper} D_{f_{m}}\left(x_t^r, x_\star\right)}+2\eta\Exp{ \frac{1}{C} \sum_{m\in S_t^\clper}   D_{f_{m}}\left(x_\star^r, x_\star\right)}.
\end{align*}
Applying lemma we have 

\begin{align*}
	\Exp{\mynorm{x^{r+1}_t - x^{r+1}_\star}} &\leq \Exp{  \left(1-\frac{\mu\eta}{2}\right)\mynorm{x^r_t - x^r_\star}}\\
	&\quad+\frac{3\eta}{\mu}L^2\left(8\gamma^2n^2L \frac{1}{M}\sum_{m=1}^{M}D_{f_m}\left(x^{r}_t,x_\star\right) + 2\gamma^2n^2\frac{1}{M}\sum_{m=1}^{M}\left\| \nabla f_m(x_\star) \right\|^2\right)\\
	&\quad+\frac{3\eta}{\mu}L^2\left(2\gamma^2n\frac{1}{M}\sum_{m=1}^{M}\frac{1}{N}\sum_{j=0}^{N-1}\left\| \nabla f^\dtper_{m} (x_\star)\right\|^2\right)\\
	&\quad+2\eta\left(2 \eta L -1 \right)  \frac{1}{M}\sum_{m=1}^M D_{f_{m}}\left(x_t^r, x_\star\right)+2\eta\Exp{ \frac{1}{C} \sum_{m\in S_t^\clper}   D_{f_{m}}\left(x_\star^r, x_\star\right)}.
\end{align*}
Rearraging terms leads to
\begin{align*}
	\Exp{\mynorm{x^{r+1}_t - x^{r+1}_\star}} &\leq 	\Exp{  \left(1-\frac{\mu\eta}{2}\right)\mynorm{x^r_t - x^r_\star}}\\
	&\quad+\frac{3\eta}{\mu}L^2\left( 2\gamma^2n^2\frac{1}{M}\sum_{m=1}^{M}\left\| \nabla f_m(x_\star) \right\|^2\right)+\frac{3\eta}{\mu}L^2\left(2\gamma^2n\frac{1}{M}\sum_{m=1}^{M}\frac{1}{N}\sum_{j=0}^{N-1}\left\| \nabla f^\dtper_{m} (x_\star)\right\|^2\right)\\
	&\quad+2\eta\left(2 \eta L -1-12\gamma^2N^2L^2\kappa \right)  \frac{1}{M}\sum_{m=1}^M D_{f_{m}}\left(x_t^r, x_\star\right)+2\eta\Exp{ \frac{1}{C} \sum_{m\in S_t^\clper}   D_{f_{m}}\left(x_\star^r, x_\star\right)}.
\end{align*}
Using $\eta\leq \frac{1}{4L}$ and $\gamma\leq \frac{1}{5NL\sqrt{\kappa}}$ we have 
$2 \eta L -1-12\gamma^2N^2L^2\kappa\geq0.$
Finally, we get 
\begin{align*}
	\Exp{\mynorm{x^{r+1}_t - x^{r+1}_\star}} &\leq 	\Exp{  \left(1-\frac{\mu\eta}{2}\right)\mynorm{x^r_t - x^r_\star}}+2\eta\Exp{ \frac{1}{C} \sum_{m\in S_t^\clper}   D_{f_{m}}\left(x_\star^r, x_\star\right)}\\
	&\quad+\frac{3\eta}{\mu}L^2\left( 2\gamma^2n^2\frac{1}{M}\sum_{m=1}^{M}\left\| \nabla f_m(x_\star) \right\|^2\right)+\frac{3\eta}{\mu}L^2\left(2\gamma^2n\frac{1}{M}\sum_{m=1}^{M}\frac{1}{N}\sum_{j=0}^{N-1}\left\| \nabla f^\dtper_{m} (x_\star)\right\|^2\right).
\end{align*}
Using $x_{t+1} = x^R_t $, $x^R_\star = x_\star$, $x^0_t = x_t$ and $x^0_\star = x_\star$ we can unroll the recursion:
\begin{align*}
	\Exp{\|x_{t+1} - x_\star\|^2} &\leq (1-\eta\mu)^R\| x_t - x_\star \|^2\\
	&\quad+\sum_{r=0}^{R-1}(1-\eta\mu)^r\left(\frac{3\eta}{\mu}L^2\left( 2\gamma^2n^2\frac{1}{M}\sum_{m=1}^{M}\left\| \nabla f_m(x_\star) \right\|^2\right)\right)\\
	&\quad + \sum_{r=0}^{R-1}(1-\eta\mu)^r\left(\frac{3\eta}{\mu}L^2\left(2\gamma^2n\frac{1}{M}\sum_{m=1}^{M}\frac{1}{N}\sum_{j=0}^{N-1}\left\| \nabla f^\dtper_{m} (x_\star)\right\|^2\right)\right)\\
	&\quad+\sum_{r=0}^{R-1}(1-\eta\mu)^r\left(2\eta\max_r\Exp{ \frac{1}{C} \sum_{m\in S_t^\clper}   D_{f_{m}}\left(x_\star^r, x_\star\right)}\right).
\end{align*}
Unrolling this recursion across $T$ epochs, we obtain
\begin{align*}
	\Exp{\|x_{T} - x_\star\|^2} &\leq \left(1-\frac{1}{2}\eta\mu\right)^{RT}\| x_0 - x_\star \|^2\\
	&\quad+\sum_{t=0}^{T-1}\left(1-\frac{1}{2}\eta\mu\right)^{Rt}\sum_{r=0}^{R-1}\left(1-\frac{1}{2}\eta\mu\right)^{r}\left(\frac{3\eta}{\mu}L^2\left( 2\gamma^2n^2\frac{1}{M}\sum_{m=1}^{M}\left\| \nabla f_m(x_\star) \right\|^2\right)\right)\\
	&\quad +\sum_{t=0}^{T-1}\left(1-\frac{1}{2}\eta\mu\right)^{Rt}\sum_{r=0}^{R-1}\left(1-\frac{1}{2}\eta\mu\right)^{r}\left(\frac{3\eta}{\mu}L^2\left(2\gamma^2n\frac{1}{M}\sum_{m=1}^{M}\frac{1}{N}\sum_{j=0}^{N-1}\left\| \nabla f^\dtper_{m} (x_\star)\right\|^2\right)\right)\\
	&\quad+\sum_{t=0}^{T-1}\left(1-\frac{1}{2}\eta\mu\right)^{Rt}\sum_{r=0}^{R-1}\left(1-\frac{1}{2}\eta\mu\right)^{r}\left(2\eta\max_r\Exp{ \frac{1}{C} \sum_{m\in S_t^\clper}   D_{f_{m}}\left(x_\star^r, x_\star\right)}\right).
\end{align*}
Note that 
\begin{align*}
	\sum_{t=0}^{T-1}\left(1-\frac{1}{2}\eta\mu\right)^{Rt}\sum_{r=0}^{R-1}\left(1-\frac{1}{2}\eta\mu\right)^{r} = \sum_{t=0}^{T-1}\sum_{r=0}^{R-1}\left(1-\frac{1}{2}\eta\mu\right)^{tR+r}\leq \frac{2}{\eta\mu}.
\end{align*}
Applying this inequality leads to
\begin{align*}
	\Exp{\|x_{T} - x_\star\|^2} &\leq (1-\eta\mu)^{RT}\| x_0 - x_\star \|^2\\
	&\quad+12\kappa^2\gamma^2N^2\frac{1}{M}\sum_{m=1}^{M}\left\| \nabla f_m(x_\star) \right\|^2\\
	&\quad+12\kappa^2\gamma^2N\frac{1}{M}\sum_{m=1}^{M}\frac{1}{N}\sum_{j=0}^{N-1}\left\| \nabla f^\dtper_{m} (x_\star)\right\|^2\\
	&\quad+\frac{4}{\mu}\max_r\Exp{ \frac{1}{C} \sum_{m\in S_t^\clper}   D_{f_{m}}\left(x_\star^r, x_\star\right)}.
\end{align*}
Using definition we obtain
\begin{align*}
	\Exp{\|x_{T} - x_\star\|^2} &\leq (1-\eta\mu)^{RT}\| x_0 - x_\star \|^2+12\kappa^2\gamma^2N^2\sigma^2_\star+12\kappa^2\gamma^2N\frac{1}{M}\sum_{m=1}^{M}\sigma^2_{m,\star}+\frac{4}{\mu}\eta^2\max_{r,m}\sigma^2_{\text{m,CS}}.
\end{align*}
\end{proof}

\section{Proof in case of $f$'s strong convexity}
In this section we prove the bound for the most general case. We need to bound the second moment of the gradient approximation.

\begin{lemma}
	\label{lemma-sq-gen-f}
	Assume that each $f_m^j$ is $L$-smooth function, then we have the following bound:
	$$\left\| \frac{1}{R}  \sum_{r=0}^{R-1}\frac{1}{C}\sum_{m\in S^\clper_t} \frac{1}{N}\sum_{j=0}^{N-1} \nabla f_m^{\pi_j}\left( x^{r,j}_{m,t} \right) \right\|^2 \leq 2L^2 V_t + 4L(f(x_t) - f(x_\star)).$$
\end{lemma}
\begin{proof}
	We use Young's inequality and $L$-smoothness to obtain the following bound:
\begin{align*}
\left\| \frac{1}{R}  \sum_{r=0}^{R-1}\frac{1}{C}\sum_{m\in S^\clper_t} \frac{1}{N}\sum_{j=0}^{N-1} \nabla f_m^{\pi_j}\left( x^{r,j}_{m,t} \right) \right\|^2 &\leq 2 \left\| \frac{1}{R}  \sum_{r=0}^{R-1}\frac{1}{C}\sum_{m\in S^\clper_t} \frac{1}{N}\sum_{j=0}^{N-1} \nabla f_m^{\pi_j}\left( x^{r,j}_{m,t} \right) - \nabla f(x_\star)\right\|^2 + 2\| \nabla f(x_t) \|^2\\
&\leq \frac{2L^2}{RCN}  \sum_{r=0}^{R-1}\sum_{m\in S^\clper_t} \sum_{j=0}^{N-1} \left\| x_t - x^{r,j}_{m,t} \right\|^2 + 4L\left( f(x_t) - f(x_\star) \right)\\
&\leq 2L^2 V_t + 4L(f(x_t) - f(x_\star)).
\end{align*}	
\end{proof}
	As previously, we need to bound the sum of distances $V_t$:
\begin{lemma}
	Assume that each $f_m^j$ is $L$-smooth function, then we have the following bound:
\begin{align*}
	\mathbb{E}\left[V_t\right] \leq 8\eta^2 L (2+R^2)D_f(x_t,x_\star) + 4\eta^2 \frac{R}{N^2} \frac{M-C}{(M-1)C} \sigma^2_\star + 4\gamma^2 N^2 \frac{1}{M}\sum_{m=1}^{M}\|\nabla f_m(x_\star)\|^2 + 4\gamma^2 N \frac{1}{M}\sum_{m=1}^{M}\sigma^2_{m,\star}.
\end{align*}

\end{lemma}
\begin{proof}
	We start from definition of $\left\|x_t - x^{r,j}_{m,t}\right\|^2$:
	\begin{align*}
		\left\|x_t - x^{r,j}_{m,t}\right\|^2 &= \left\| x^{r,j}_{m,t} - x_t\right\|^2\\
		&= \left\|-\eta \sum_{k=0}^{r-1}\frac{1}{C}\sum_{m\in S^{\lambda_k}_t} \frac{1}{N}\sum_{j=0}^{N-1} \nabla f_m^{\pi_j} \left(x^{k,j}_{m,t}\right) - \gamma\sum_{l=0}^{j-1}\nabla f_m^{\pi_l}\left( x_{m,t}^{r,l} \right)\right\|^2\\
		& =  \left\|\eta \sum_{k=0}^{r-1}\frac{1}{C}\sum_{m\in S^{\lambda_k}_t} \frac{1}{N}\sum_{j=0}^{N-1} \nabla f_m^{\pi_j} \left(x^{k,j}_{m,t}\right) + \gamma\sum_{l=0}^{j-1}\nabla f_m^{\pi_l}\left( x_{m,t}^{r,l} \right)\right\|^2\\
		&\leq 2 \left\|\eta \sum_{k=0}^{r-1}\frac{1}{C}\sum_{m\in S^{\lambda_k}_t} \frac{1}{N}\sum_{j=0}^{N-1} \nabla f_m^{\pi_j} \left(x^{k,j}_{m,t}\right) \right\|^2+2\left\| \gamma\sum_{l=0}^{j-1}\nabla f_m^{\pi_l}\left( x_{m,t}^{r,l} \right) \right\|^2.
	\end{align*}
Using Young's inequality, we obtain
\begin{align*}
			\left\|x_t - x^{r,j}_{m,t}\right\|^2		&\leq 4 \left\|\eta \sum_{k=0}^{r-1}\frac{1}{C}\sum_{m\in S^{\lambda_k}_t} \frac{1}{N}\sum_{j=0}^{N-1} \left(\nabla f_m^{\pi_j} \left(x^{k,j}_{m,t}\right) -  \nabla f_m^{\pi_j} \left(x_t\right)\right) \right\|^2\\
			&+4\left\|\eta \sum_{k=0}^{r-1}\frac{1}{C}\sum_{m\in S^{\lambda_k}_t} \frac{1}{N}\sum_{j=0}^{N-1}   \nabla f_m^{\pi_j} \left(x_t\right) \right\|^2\\
			&+4\left\| \gamma\sum_{l=0}^{j-1}\left(\nabla f_m^{\pi_l}\left( x_{m,t}^{r,l} \right) - f_m^{\pi_l}\left( x_t \right) \right) \right\|^2\\
			&+4\left\|  \gamma\sum_{l=0}^{j-1}\nabla f_m^{\pi_l}\left( x_t \right) \right\|^2.
	\end{align*}
Using $L$-smoothness and lemma we have 
\begin{align*}
			\left\|x_t - x^{r,j}_{m,t}\right\|^2 &\leq 4 \eta^2 r^2 L^2 \frac{1}{rCN}  \sum_{k=0}^{r-1}\sum_{m\in S^\clper_t} \sum_{j=0}^{N-1} \left\| x_t - x^{k,j}_{m,t} \right\|^2 + 4\gamma^2 L^2 j \sum_{l=0}^{j-1}\left\| x_{m,t}^{r,l} - x_t \right\|^2\\
			&+4\gamma^2\left(j^2 \| \nabla f_m(x_t) \|^2 + \frac{j(N-j)}{N-1}\sigma^2_{m,t}\right) + 4\eta^2\frac{1}{N^2 C^2} \left( N^2 C^2 r^2 \|\nabla f(x_t)\|^2 + \frac{C r(M - Cr)}{M-1}\sigma^2_t\right). 
\end{align*}
Using this bound we obtain
\begin{align*}
	\mathbb{E}\left[V_t\right] & = \frac{1}{CRN} \sum_{r=0}^{R-1}\sum_{m\in S^\clper_t} \sum^{N-1}_{j=0} \mathbb{E} \left\| x^{k,j}_{m,t} - x_t \right\|^2\\
	& \leq \frac{1}{CRN} \sum_{r=0}^{R-1}\sum_{m\in S^\clper_t} \sum^{N-1}_{j=0}\left( 4 \eta^2 r^2 L^2 \frac{1}{rCN}  \sum_{k=0}^{r-1}\sum_{m\in S^\clper_t} \sum_{j=0}^{N-1} \left\| x_t - x^{k,j}_{m,t} \right\|^2 + 4\gamma^2 L^2 j \sum_{l=0}^{j-1}\left\| x_{m,t}^{r,l} - x_t \right\|^2\right)\\
	&+\frac{1}{CRN} \sum_{r=0}^{R-1}\sum_{m\in S^\clper_t} \sum^{N-1}_{j=0}\left(4\gamma^2\left(j^2 \| \nabla f_m(x_t) \|^2 + \frac{j(N-j)}{N-1}\sigma^2_{m,t}\right)\right)\\
	& +\frac{1}{CRN} \sum_{r=0}^{R-1}\sum_{m\in S^\clper_t} \sum^{N-1}_{j=0}\left( 4\eta^2\frac{1}{N^2 C^2} \left( N^2 C^2 r^2 \|\nabla f(x_t)\|^2 + \frac{C r(M - Cr)}{M-1}\sigma^2_t\right)\right).
\end{align*}
Using sums over indices we get
\begin{align*}
	\mathbb{E}\left[V_t\right]	&\leq \frac{R(R-1)}{2}4\eta^2 L^2\mathbb{E} \left[V_t\right] + \frac{N(N-1)}{2}4\gamma^2 L^2 \mathbb{E}\left[V_t\right]\\
	& + \frac{2}{3}\gamma^2 \frac{1}{M}\sum_{m=1}^{M}\|\nabla f_m(x_t)\|^2(N-1)(2N-1)+\frac{2}{3}\gamma^2(N+1)\frac{1}{M}\sum_{m=1}^{M}\sigma^2_{m,t}\\
	& + \frac{2}{3}\eta^2 \| \nabla f(x_t) \|^2 (R-1)(2R-1)+ \frac{2}{3} \frac{M-C}{(M-1)C} \eta^2\frac{R+1}{N^2}\sigma^2_t\\
	&\leq 2\eta^2L^2(1+R^2)V_t + \frac{2}{3}\gamma^2\frac{1}{M}\sum_{m=1}^{M}\| \nabla f_m(x_t) \|^2(N-1)(2N-1)\\
	& + \frac{2}{3}\eta^2 \| \nabla f(x_t) \|^2 (R-1)(2R-1)+\frac{2}{3}\gamma^2(N+1)\frac{1}{M}\sum_{m=1}^{M}\sigma^2_{m,t}+\frac{2}{3}\eta^2\frac{R+1}{N^2}\frac{M-C}{(M-1)C}\sigma^2_t.
\end{align*}
To extract the $\mathbb{E}\left[V_t\right]$ we need to assume that  $\gamma N R \leq \eta R \leq \theta \leq \frac{1}{16L}$ to have $1-2\eta^2L^2(1+R^2)>0$. This leads to
\begin{align*}
	\mathbb{E}\left[V_t\right]& \leq 2\gamma^2 N^2 \frac{1}{M}\sum_{m=1}^{M}\|\nabla f_m(x_t)\|^2 + 2\eta^2R^2\| \nabla f(x_t) \|^2 +2\gamma^2 N \frac{1}{M} \sum_{m=1}^{M}\sigma^2_{m,t}+2\eta^2 \frac{R}{N^2} \frac{M-C}{(M-1)C}\sigma^2_t\\
	&\leq 4\gamma^2 N^2 \frac{1}{M}\sum_{m=1}^{M}\|\nabla f_m(x_t) - \nabla f_m(x_\star)\|^2 + 4\gamma^2 N^2 \frac{1}{M}\sum_{m=1}^{M}\| \nabla f_m(x_\star)\|^2 \\
	&+4\gamma^2 N \frac{1}{M}\sum_{m=1}^{M}\frac{1}{N}\sum_{j=0}^{N-1}\left\| \nabla f^{\pi_j}_m(x_t) - \nabla f^{\pi_j}_m(x_\star)    \right\|^2 + 4\gamma^2 N \frac{1}{M} \sum_{m=1}^{M}\sigma^2_{m,\star}\\
	&+2\eta^2 R^2 \| \nabla f(x_t) - \nabla f(x_\star) \|^2+ 4\eta^2R\frac{M-C}{(M-1)C}\frac{1}{M}\sum_{m=1}^{M}\| \nabla f_m(x_t) - \nabla f_m(x_\star) \|^2+4\eta^2 \frac{R}{N^2} \frac{M-C}{(M-1)C} \sigma^2_\star\\
	&\leq 8\eta^2 L (2+R^2)D_f(x_t,x_\star) + 4\eta^2 \frac{R}{N^2} \frac{M-C}{(M-1)C} \sigma^2_\star + 4\gamma^2 N^2 \frac{1}{M}\sum_{m=1}^{M}\|\nabla f_m(x_\star)\|^2 + 4\gamma^2 N \frac{1}{M}\sum_{m=1}^{M}\sigma^2_{m,\star}.
\end{align*}
\end{proof}

We also need to bound the inner product.

\begin{lemma}	Assume that each $f_m^j$ is $L$-smooth function and $f$ is $\mu$-strongly convex, then we have the following bound:
	\label{lemma:inner-gen-f}
	\begin{align*}
		- 2\theta \left\langle   \frac{1}{R}\sum_{r=0}^{R-1}\frac{1}{C}\sum_{m\in S^\clper_t} \frac{1}{N}\sum_{j=0}^{N-1} \nabla f_m^{\pi_j}\left( x^{r,j}_{m,t} \right), x_t - x_\star\right\rangle \leq - \frac{\theta\mu}{2} \| x_t - x_\star \|^2 - \theta \left(f(x_t) - f(x_\star)\right) + \theta L V_t.
	\end{align*}
\end{lemma}
\begin{proof}
	We start from initial term:
	\begin{align*}
		- 2\theta \left\langle   \frac{1}{R}\sum_{r=0}^{R-1}\frac{1}{C}\sum_{m\in S^\clper_t} \frac{1}{N}\sum_{j=0}^{N-1} \nabla f_m^{\pi_j}\left( x^{r,j}_{m,t} \right), x_t - x_\star\right\rangle  = -2\theta \frac{1}{RCN}\sum_{r=0}^{R-1}\sum_{m\in S^\clper_t}\sum_{j=0}^{N-1}\left\langle  \nabla f_m^{\pi_j}\left( x^{r,j}_{m,t} \right), x_t - x_\star  \right\rangle.
	\end{align*}
	Let us consider $\left\langle  \nabla f_m^{\pi_j}\left( x^{r,j}_{m,t} \right), x_t - x_\star  \right\rangle$:
	\begin{align*}
		\left\langle  \nabla f_m^{\pi_j}\left( x^{r,j}_{m,t} \right), x_t - x_\star  \right\rangle &= f_m^{\pi_j}\left( x_t \right) - f_m^{\pi_j}\left( x_\star \right) + f_m^{\pi_j}\left( x_\star \right) - f_m^{\pi_j}\left( x^{r,j}_{m,t} \right) + \left\langle   \nabla f_m^{\pi_j}\left( x^{r,j}_{m,t} \right), x^{r,j}_{m,t} - x_\star    \right\rangle - f_m^{\pi_j}\left( x_t \right)\\
		&+f_m^{\pi_j}\left( x^{r,j}_{m,t} \right) + \left\langle   \nabla f_m^{\pi_j}\left( x^{r,j}_{m,t} \right), x_t - x^{r,j}_{m,t}    \right\rangle\\
		& = \left( f_m^{\pi_j}(x_t) - f_m^{\pi_j}(x_\star) \right) + D_{f_m^{\pi_j}}\left(x_\star, x^{r,j}_{m,t}  \right) - D_{f_m^{\pi_j}}\left(x_t, x^{r,j}_{m,t}  \right).
	\end{align*}
	Plugging this identity we get
	\begin{align*}
		&-2\theta \frac{1}{RCN}\sum_{r=0}^{R-1}\sum_{m\in S^\clper_t}\sum_{j=0}^{N-1}\left\langle  \nabla f_m^{\pi_j}\left( x^{r,j}_{m,t} \right), x_t - x_\star  \right\rangle\\
		= 	&-2\theta \frac{1}{RCN}\sum_{r=0}^{R-1}\sum_{m\in S^\clper_t}\sum_{j=0}^{N-1}\left\langle \left( f_m^{\pi_j}(x_t) - f_m^{\pi_j}(x_\star) \right) + D_{f_m^{\pi_j}}\left(x_\star, x^{r,j}_{m,t}  \right) - D_{f_m^{\pi_j}}\left(x_t, x^{r,j}_{m,t}  \right)  \right\rangle\\
		\leq& - 2\theta\left( f(x_t) - f(x_\star) \right) -  \frac{2\theta}{RCN} \sum_{r=0}^{R-1}\sum_{m\in S^\clper_t}\sum_{j=0}^{N-1} D_{f_m^{\pi_j}}\left(x_\star, x^{r,j}_{m,t} \right) +  \frac{\theta L}{RCN} \sum_{r=0}^{R-1}\sum_{m\in S^\clper_t}\sum_{j=0}^{N-1} \left\| x_t - x^{r,j}_{m,t} \right\|^2\\
		= & - \theta \left( f(x_t) - f(x_\star) \right) - \theta \left( f(x_t) - f(x_\star) \right) - \frac{2\theta}{RCN} \sum_{r=0}^{R-1}\sum_{m\in S^\clper_t}\sum_{j=0}^{N-1} D_{f_m^{\pi_j}}\left(x_\star, x^{r,j}_{m,t} \right)\\
		&+  \frac{\theta L}{RCN} \sum_{r=0}^{R-1}\sum_{m\in S^\clper_t}\sum_{j=0}^{N-1} \left\| x_t - x^{r,j}_{m,t} \right\|^2\\
		\leq & - \frac{\theta\mu}{2}\| x_t - x_\star \|^2 - \theta \left( f(x_t) - f(x_\star) \right) + \frac{\theta L}{RCN} \sum_{r=0}^{R-1}\sum_{m\in S^\clper_t}\sum_{j=0}^{N-1} \left\| x_t - x^{r,j}_{m,t} \right\|^2\\
		\leq & - \frac{\theta\mu}{2}\| x_t - x_\star \|^2 - \theta \left( f(x_t) - f(x_\star) \right) + \theta L V_t,
	\end{align*} 
	where 
	\begin{align*}
		V_t =  \frac{1}{RCN} \sum_{r=0}^{R-1}\sum_{m\in S^\clper_t}\sum_{j=0}^{N-1} \left\| x_t - x^{r,j}_{m,t} \right\|^2.
	\end{align*}
\end{proof}

Now we are ready to formulate the final theorem.
\begin{theorem}
	Suppose that each $f^j_m$ is convex and $L$-smooth, $f$ is $\mu$-strongly convex. Then provided the step size satisfies $\gamma N R \leq \eta R \leq \theta \leq \frac{1}{16L}$ the final iterate generated by Algorithm 1 satisfies
	\begin{align*}
		\Exp{\mynorm{x_T - x_\star}} &\leq \left(1-\frac{\theta\mu}{2}\right)^T\mynorm{x_0 - x_\star}+16\gamma^2 \kappa n \frac{1}{M}\sum_{m=1}^{M}\left( n\| \nabla f_m(x_\star) \|^2 +\sigma^2_{m,\star}\right)\\
		&\quad+16 \eta^2\frac{\kappa}{n^2R}\frac{M-C}{(M-1)C}\sigma^2_\star.
	\end{align*}
\end{theorem}
\begin{proof}
	We start from definition of the points $x^R_t$ and $x^{r,j}_{m,t}$:
	\begin{align*}
		x^R_t &= x_t - \eta \sum_{r=0}^{R-1}\frac{1}{C}\sum_{m\in S^\clper_t} \frac{1}{N}\sum_{j=0}^{N-1} \nabla f_m^{\pi_j}\left( x^{r,j}_{m,t} \right),\\
		x^{r,j}_{m,t} &= x_t - \eta \sum_{k=0}^{r-1}\frac{1}{C}\sum_{m\in S^{\lambda_k}_t} \frac{1}{N}\sum_{j=0}^{N-1} \nabla f_m^{\pi_j} \left(x^{k,j}_{m,t}\right) - \gamma\sum_{l=0}^{j-1}\nabla f_m^{\pi_l}\left( x_{m,t}^{r,l} \right),\\
		x_{t+1} &= x_t - \frac{\theta}{\eta R}(x_t - x^R_{t}).
	\end{align*}
	Let us start from distance to the solution:
	\begin{align*}
		\|x_{t+1} - x_\star\|^2 &= \left\| x_t - \frac{\theta}{\eta R}(x_t - x_t^R) - x_\star \right\|^2\\
		& = \|x_t - x_\star\|^2 - 2\frac{\theta}{\eta R} \left\langle x_t - x_t^R, x_t - x_\star \right\rangle + \frac{\theta^2}{\eta^2R^2}\|x_t - x_t^R\|^2\\
		& = \|x_t - x_\star\|^2 - 2\frac{\theta}{\eta R} \left\langle  \eta \sum_{r=0}^{R-1}\frac{1}{C}\sum_{m\in S^\clper_t} \frac{1}{N}\sum_{j=0}^{N-1} \nabla f_m^{\pi_j}\left( x^{r,j}_{m,t} \right), x_t - x_\star\right\rangle\\
		&+\frac{\theta^2}{\eta^2R^2}\left\| \eta \sum_{r=0}^{R-1}\frac{1}{C}\sum_{m\in S^\clper_t} \frac{1}{N}\sum_{j=0}^{N-1} \nabla f_m^{\pi_j}\left( x^{r,j}_{m,t} \right) \right\|^2.
	\end{align*}
Finally, we have 
\begin{align*}
	\|x_{t+1} - x_\star\|^2	&= \|x_t - x_\star\|^2 - 2\theta \left\langle   \frac{1}{R}\sum_{r=0}^{R-1}\frac{1}{C}\sum_{m\in S^\clper_t} \frac{1}{N}\sum_{j=0}^{N-1} \nabla f_m^{\pi_j}\left( x^{r,j}_{m,t} \right), x_t - x_\star\right\rangle\\
		&+\theta^2\left\| \frac{1}{R}  \sum_{r=0}^{R-1}\frac{1}{C}\sum_{m\in S^\clper_t} \frac{1}{N}\sum_{j=0}^{N-1} \nabla f_m^{\pi_j}\left( x^{r,j}_{m,t} \right) \right\|^2.
	\end{align*}
	Let us apply Lemma~\ref{lemma:inner-gen-f} and Lemma~\ref{lemma-sq-gen-f} and we obtain
	\begin{align*}
		\|x_{t+1} - x_\star\|^2 &\leq \|x_t - x_\star\|^2 - \frac{\theta\mu}{2} \| x_t - x_\star \|^2 - \theta \left(f(x_t) - f(x_\star)\right) + \theta L V_t +\theta^2 \left(   2L^2 V_t + 4L(f(x_t) - f(x_\star))\right)\\
		&\leq  \left(1 - \frac{\theta\mu}{2}\right)\|x_t - x_\star\|^2 - \theta \left(f(x_t) - f(x_\star)\right)\left(1 - 4L\theta\right)+\theta L V_t(1+2L\theta).
	\end{align*}
	Using lemma we have 
	\begin{align*}
		\|x_{t+1} - x_\star\|^2 	&\leq  \left(1 - \frac{\theta\mu}{2}\right)\|x_t - x_\star\|^2 - \theta \left(f(x_t) - f(x_\star)\right)\left(1 - 4L\theta\right)\\
		&+\theta L (1+2L\theta)\left(  8\eta^2 L (2+R^2)D_f(x_t,x_\star) + 4\eta^2 \frac{R}{N^2} \frac{M-C}{(M-1)C} \sigma^2_\star\right)\\
		& +\theta L (1+2L\theta)\left( 4\gamma^2 N^2 \frac{1}{M}\sum_{m=1}^{M}\|\nabla f_m(x_\star)\|^2 + 4\gamma^2 N \frac{1}{M}\sum_{m=1}^{M}\sigma^2_{m,\star}\right)\\
		&\leq  \left(1 - \frac{\theta\mu}{2}\right)\|x_t - x_\star\|^2 - \theta \left(f(x_t) - f(x_\star)\right)\left(1 - 4L\theta -  8L (1+2L\theta)\eta^2 L (2+R^2) \right)\\
		&+4\theta L (1+2L\theta)\eta^2 \frac{R}{N^2} \frac{M-C}{(M-1)C} \sigma^2_\star\\
		& +4\gamma^2 N\theta L (1+2L\theta)\left(  \frac{1}{M}\sum_{m=1}^{M}N\|\nabla f_m(x_\star)\|^2 +  \frac{1}{M}\sum_{m=1}^{M}\sigma^2_{m,\star}\right).
	\end{align*}
	Using  $\gamma N R \leq \eta R \leq \theta \leq \frac{1}{16L}$ we have that $\left(1 - 4L\theta -  8L (1+2L\theta)\eta^2 L (2+R^2) \right) \geq 0$, it leads to 
	\begin{align*}
		\|x_{t+1} - x_\star\|^2 	&\leq   \left(1 - \frac{\theta\mu}{2}\right)\|x_t - x_\star\|^2+8\theta L \eta^2 \frac{R}{N^2} \frac{M-C}{(M-1)C} \sigma^2_\star\\
		& +8\gamma^2 N\theta L \left(  \frac{1}{M}\sum_{m=1}^{M}N\|\nabla f_m(x_\star)\|^2 +  \frac{1}{M}\sum_{m=1}^{M}\sigma^2_{m,\star}\right).
	\end{align*}
	Unrolling this recursion leads to 
	\begin{align*}
		\|x_{T} - x_\star\|^2 	&\leq   \left(1 - \frac{\theta\mu}{2}\right)^T\|x_0 - x_\star\|^2+16 \frac{L}{\mu} \eta^2 \frac{R}{N^2} \frac{M-C}{(M-1)C} \sigma^2_\star\\
		& +16\gamma^2 N \frac{L}{\mu} \left(  \frac{1}{M}\sum_{m=1}^{M}N\|\nabla f_m(x_\star)\|^2 +  \frac{1}{M}\sum_{m=1}^{M}\sigma^2_{m,\star}\right).
	\end{align*}
	This finishes the proof. 
\end{proof}

%\end{document}

\section{Deterministic client shuffling}
\label{sec:deterministic_rr}
In this section, we discuss how to extend our method's applicability beyond random reshuffling to any (including deterministic) reshuffling. We discuss each setup individually.

\subsection{$f_m^j$ is strongly convex}
In the provided analysis, we do not specify the type of reshuffling, which means that this result can be applied to any, including deterministic, type of client shuffling. However, in this case, we have to slightly adjust the analysis as we cannot take expectations because client sampling is not necessarily random. The absence of randomization means that we need to consider the worst-case scenario instead of the average and the bound of $\frac{1}{C}\sum_{m\in S_t^\clper}\sigma^2_{m,\text{Shuffle}}$ will be more significant.

\subsection{$f_m$ is strongly convex}
Similarly, in the analysis of the case when only $f_m$ is $\mu$-convex, we do not specify the type of reshuffling. Therefore, we can apply any, including deterministic, shuffling of clients. As in the previous case, we cannot use expectations, and the bound of $\sigma^2_{r,\text{client}} $ will be up to $R$ times larger, similarly to \citep[Theorem 5 (option 1)]{mishchenko2020random}, but applied to the shuffling of clients. 

\subsection{$f$ is strongly convex}
Finally, our analysis uses Lemma 1 from \citet{mishchenko2020random} for the most restrictive case. In the case of any, including deterministic, permutations, the term connected to client shuffling $16 \eta^2\frac{\kappa}{n^2R}\frac{M-C}{(M-1)C}\sigma^2_\star$ will also be up to $R$ times larger, and we would have $16 \eta^2\frac{\kappa}{n^2}\frac{M-C}{(M-1)C}\sigma^2_\star$ appearing in our bound of Theorem~\ref{thm:main_last} since we cannot take expectations. 

\clearpage

\section{Extra Numerical Experiments}

\label{sec:extra_experiments}

In this section, we provide additional numerical experiments missing in the main part. The setup is exactly the same as described before. We note that the observations that we can make for the extra experiments are consistent with the conclusions provided in the main paper.

\begin{figure}[h!]
	\centering
	\captionsetup[sub]{font=scriptsize,labelfont={}}	
	
	\includegraphics[width=0.43\textwidth]{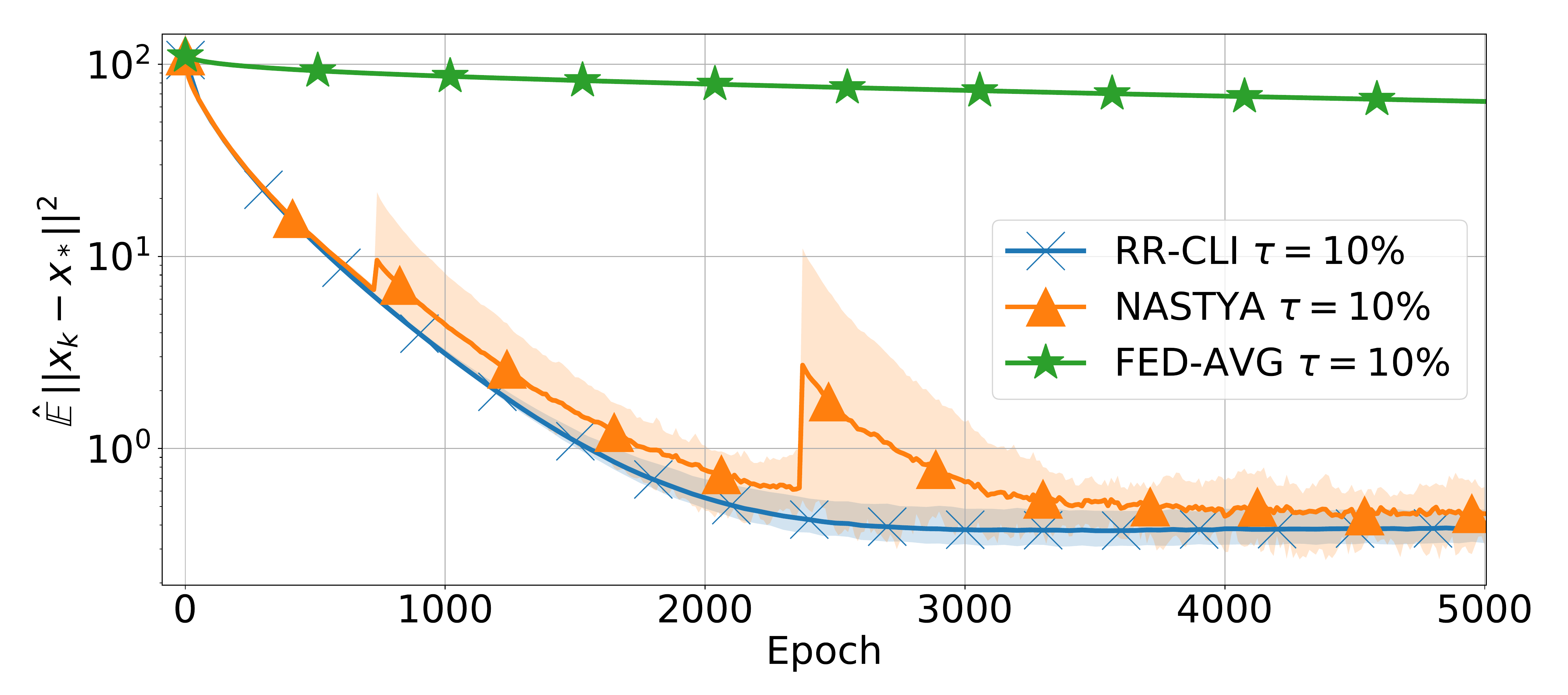}
	\includegraphics[width=0.43\textwidth]{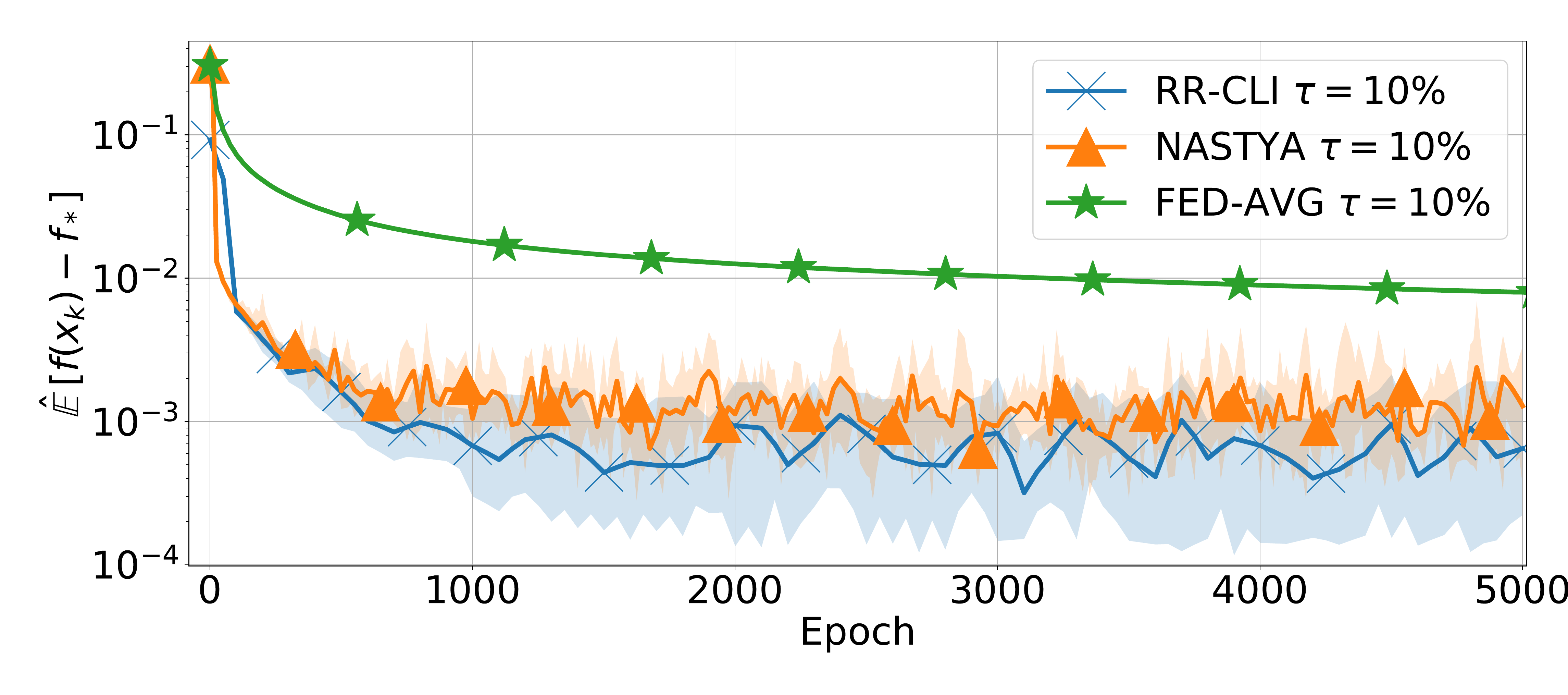}
	\caption{\small{Training \texttt{Logistic Regression} on \texttt{a3a}, with $n=12$ clients. Theoretical global step size and tuned local step sizes. Partial participation with $3$ clients per round with $10$ local steps.}}
	\label{fig:exp1_th_step sizes_extra_a3a}
\end{figure}

\begin{figure}[h!]
	\centering
	\captionsetup[sub]{font=scriptsize,labelfont={}}	
	
	\includegraphics[width=0.43\textwidth]{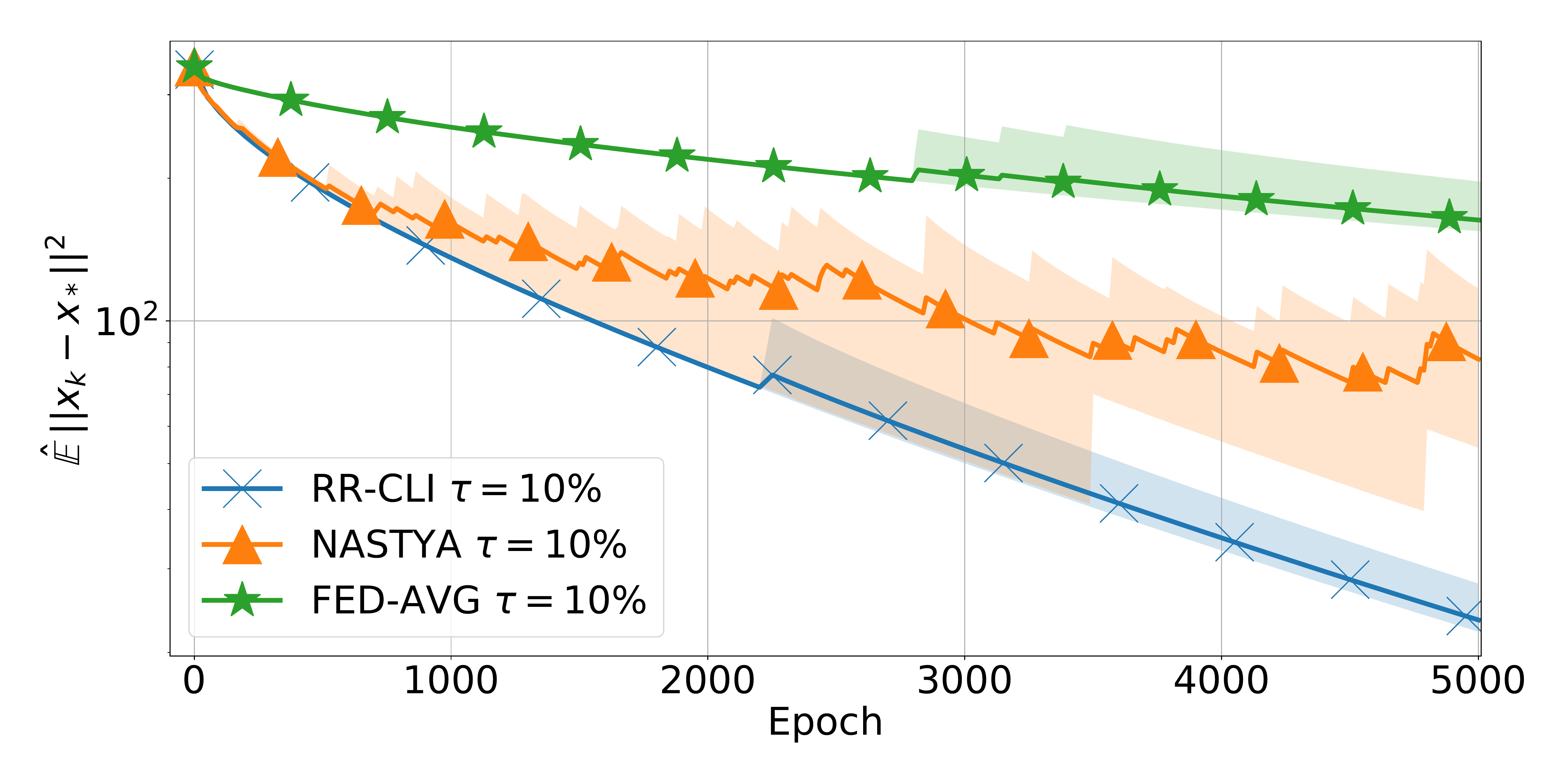}
	\includegraphics[width=0.43\textwidth]{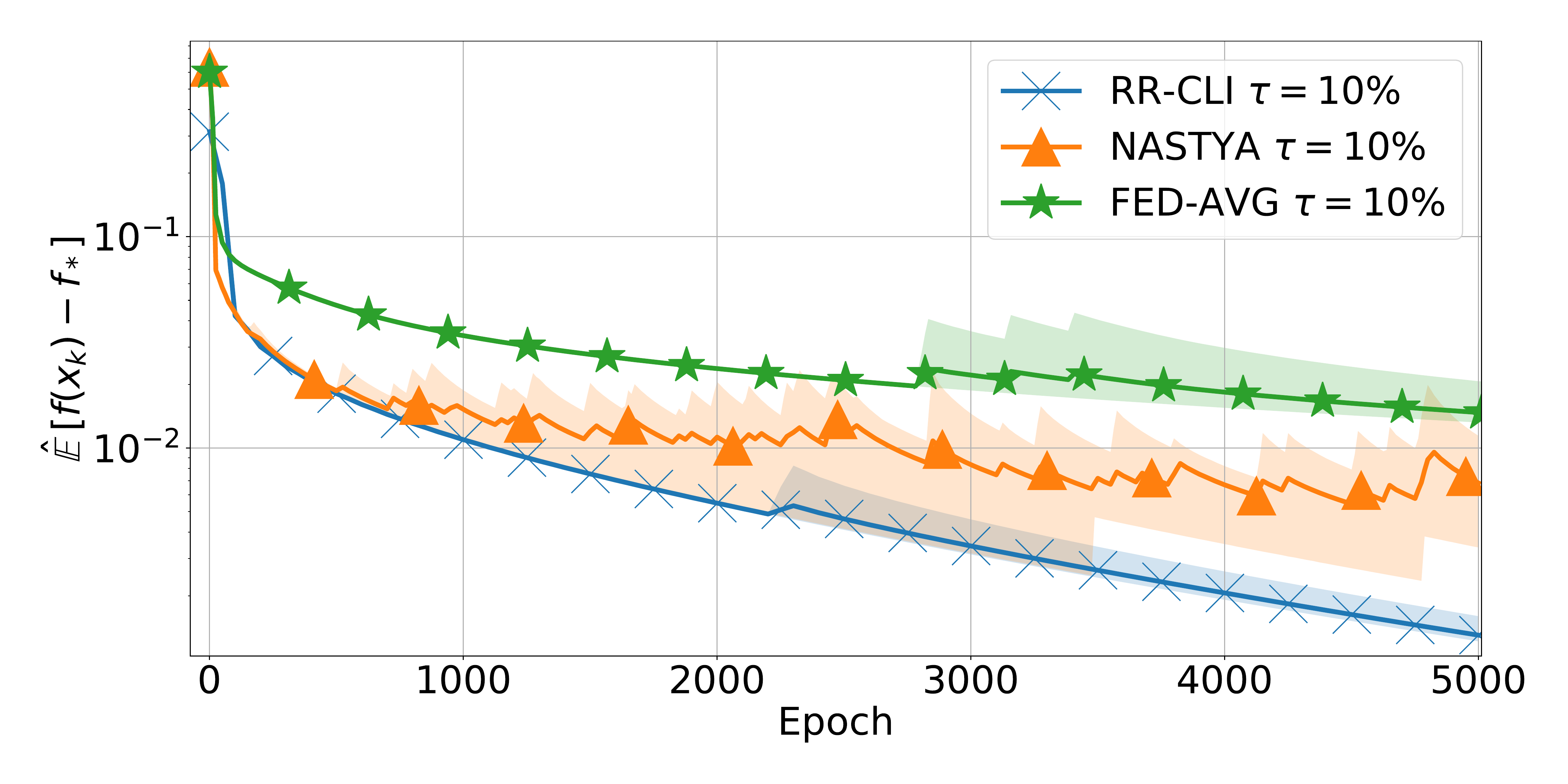}
	\caption{\small{Training \texttt{Logistic Regression} on \texttt{w3a} with $n=12$ clients. Theoretical global step size and tuned local step sizes. Partial participation with $3$ clients per round with $10$ local steps.}}
	\label{fig:exp1_th_step sizes_extra_w3a}
\end{figure}

\begin{figure}[h!]
	\centering
	\captionsetup[sub]{font=scriptsize,labelfont={}}	
	
	\includegraphics[width=0.43\textwidth]{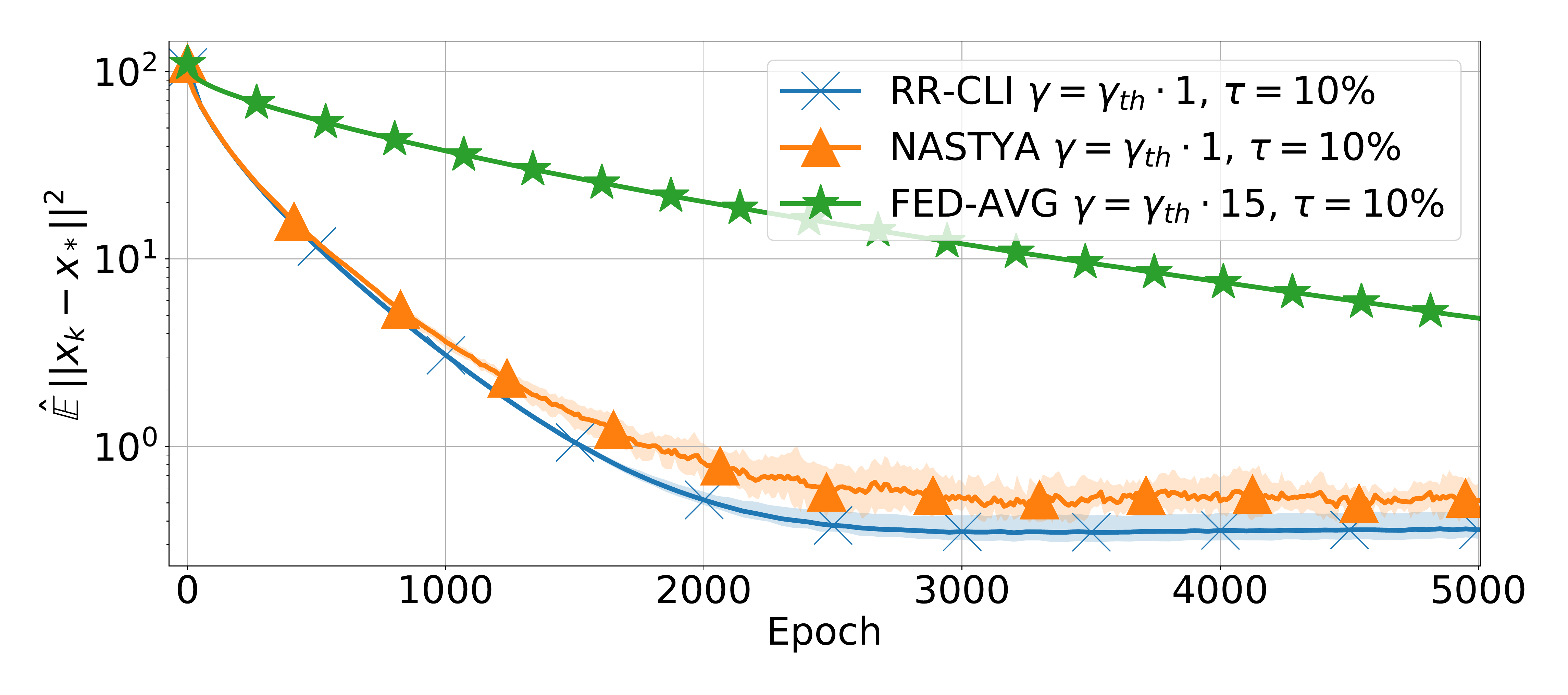}
	\includegraphics[width=0.43\textwidth]{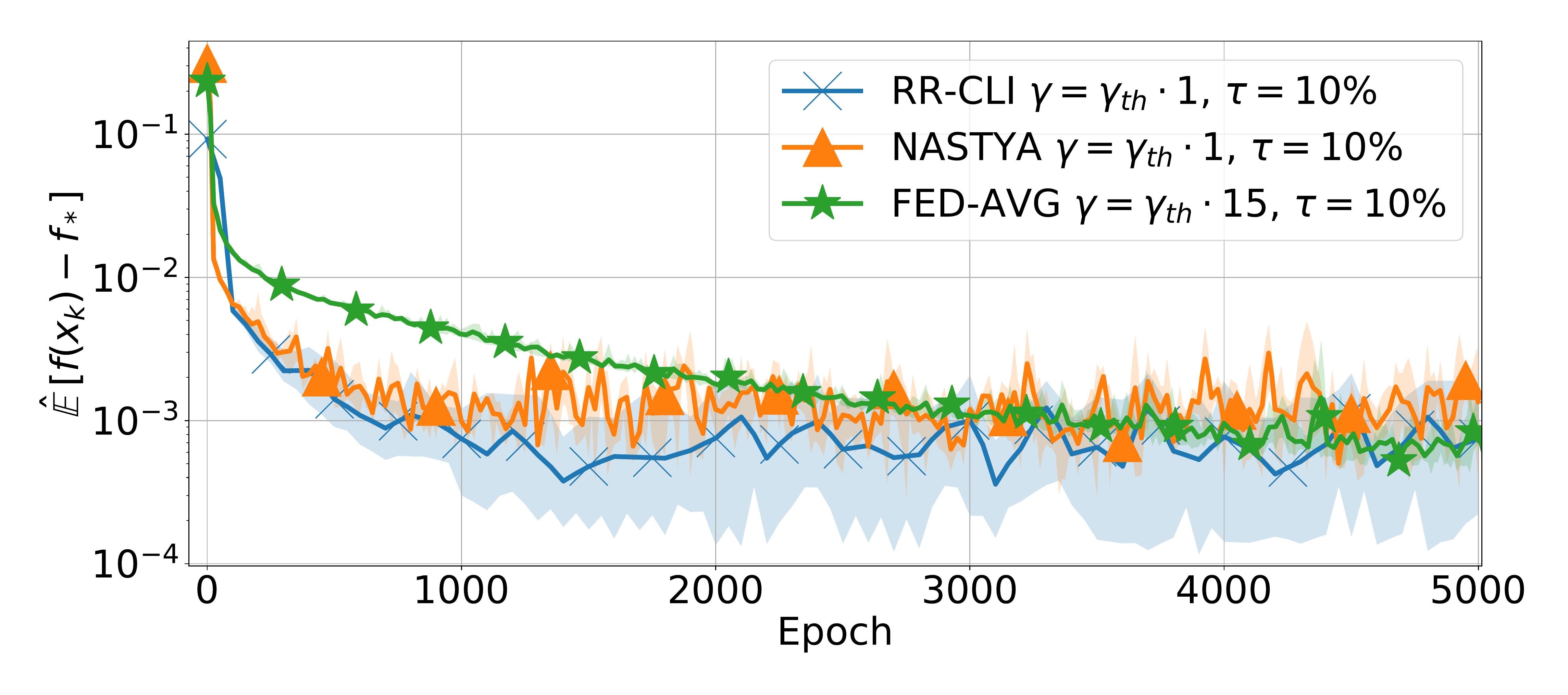}	
	\caption{\small{Training \texttt{Logistic Regression} on \texttt{a3a} with $n=12$ clients. Theoretical global step size. Local step sizes are multipliers of theoretical. PP with $3$ clients per round with, $10$ local steps.}}
	\label{fig:exp2_multth_step sizes_best_to_best_extra_a3a}
\end{figure}

\begin{figure}[h!]
	\centering
	\captionsetup[sub]{font=scriptsize,labelfont={}}	
	
	\includegraphics[width=0.43\textwidth]{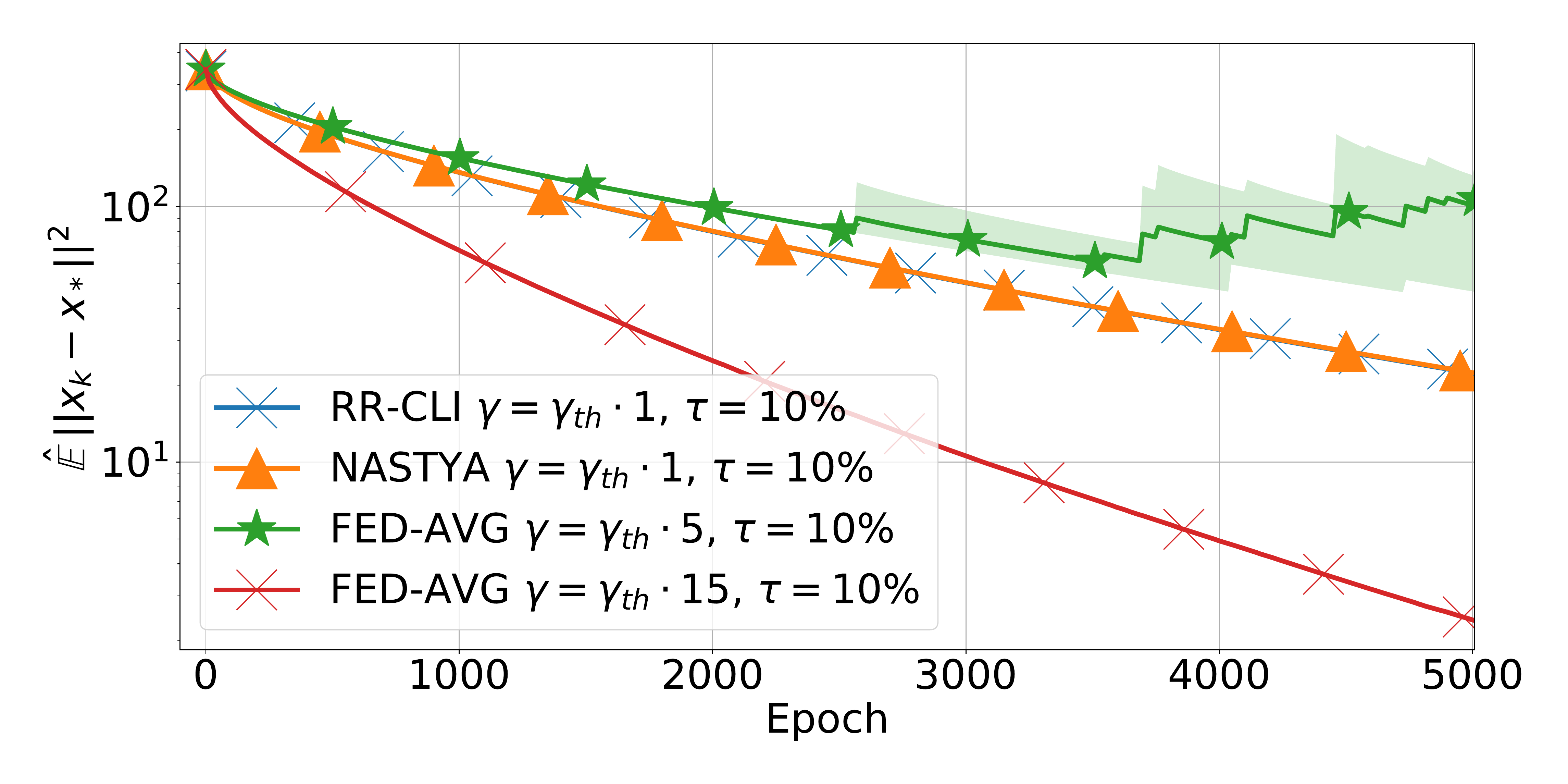}
	\includegraphics[width=0.43\textwidth]{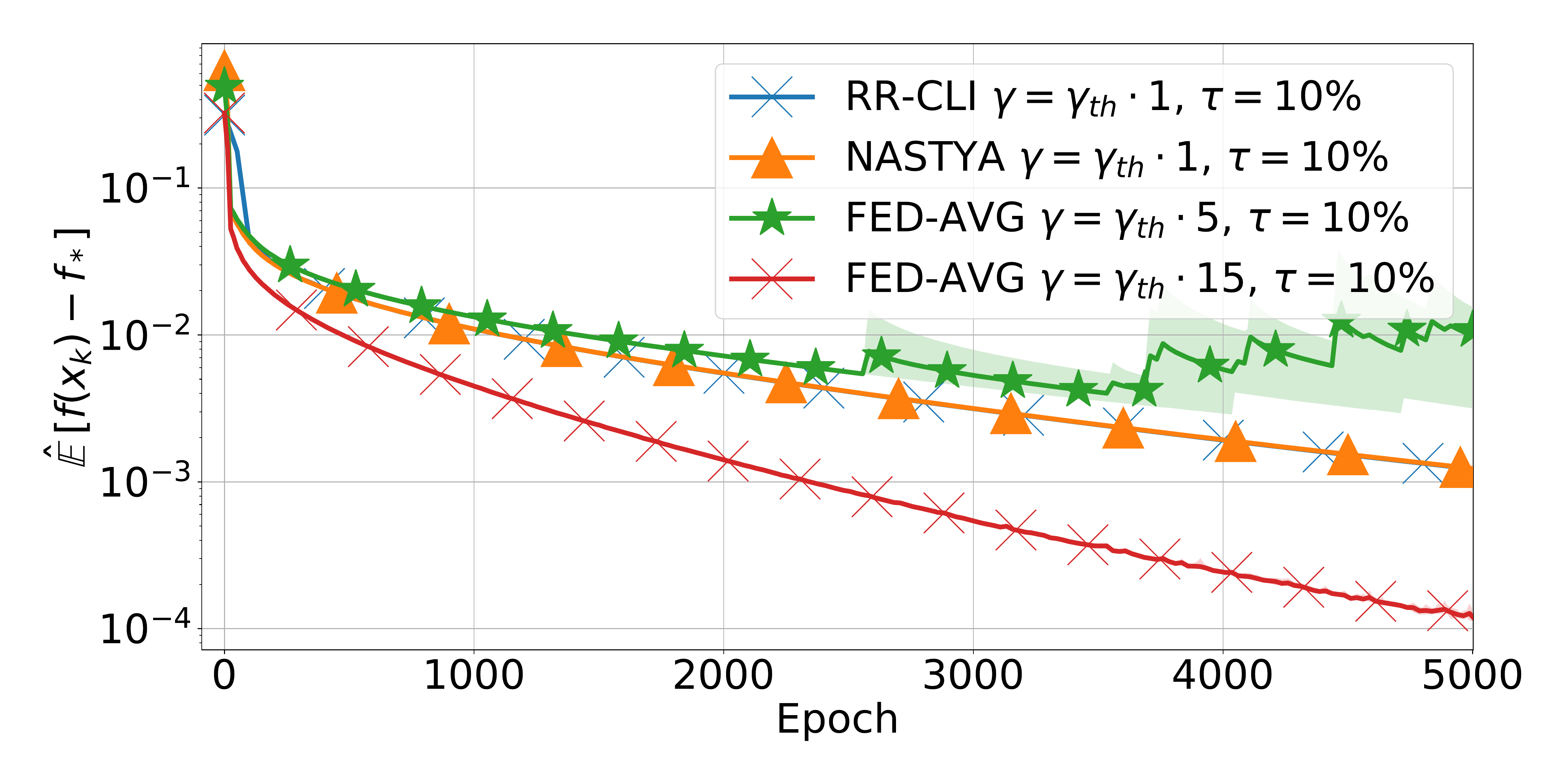}	
	\caption{\small{Training \texttt{Logistic Regression} on \texttt{w3a} with $n=12$ clients. Theoretical global step size and tuned local step sizes. Partial participation with $3$ clients per round with $10$ local steps.}}
	\label{fig:exp2_multth_step sizes_best_to_best_extra_w3a}
\end{figure}

\begin{figure*}[t!]
	\centering
	\captionsetup[sub]{font=scriptsize,labelfont={}}	
	\includegraphics[width=0.43\textwidth]{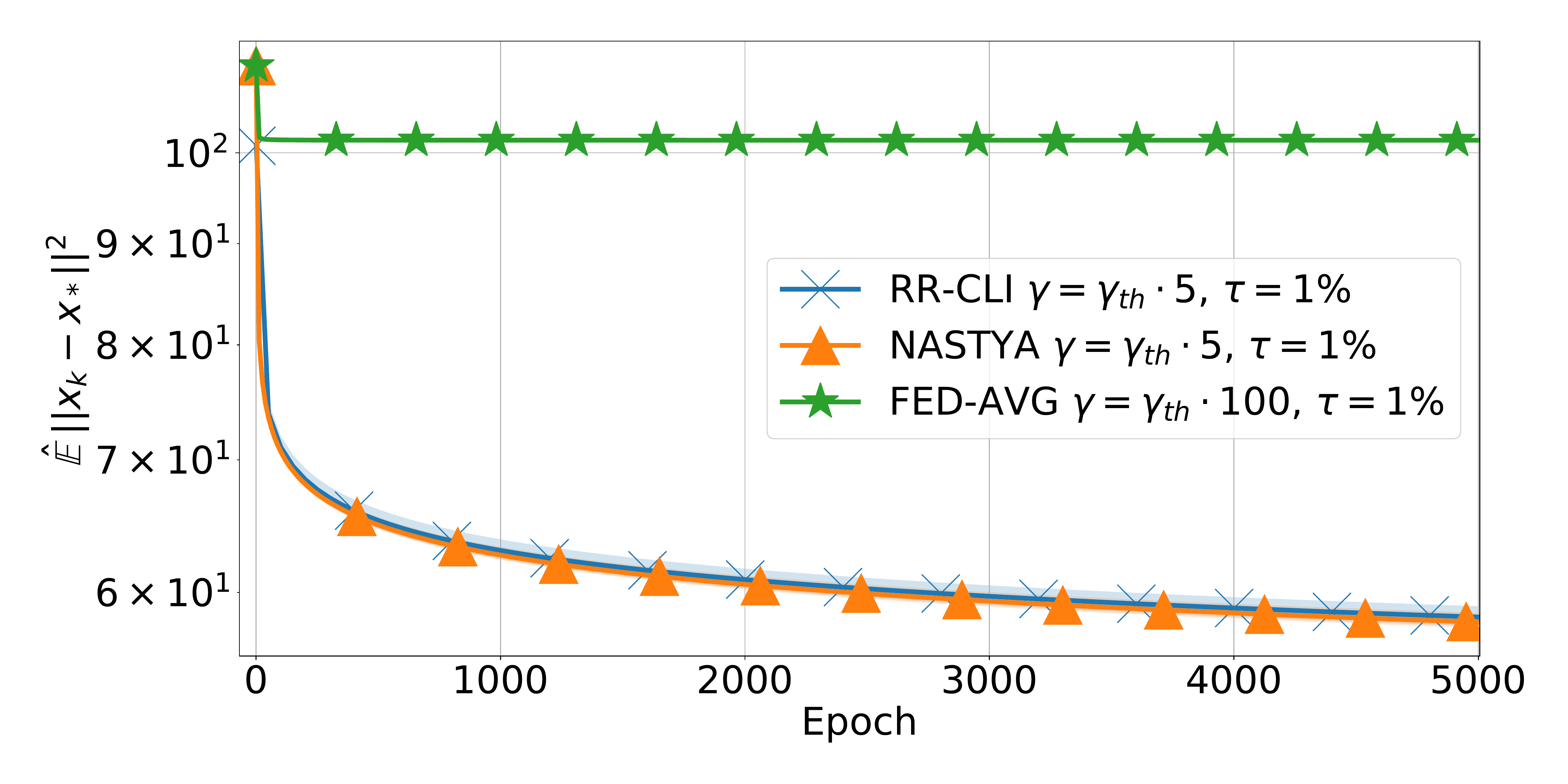}
	\includegraphics[width=0.43\textwidth]{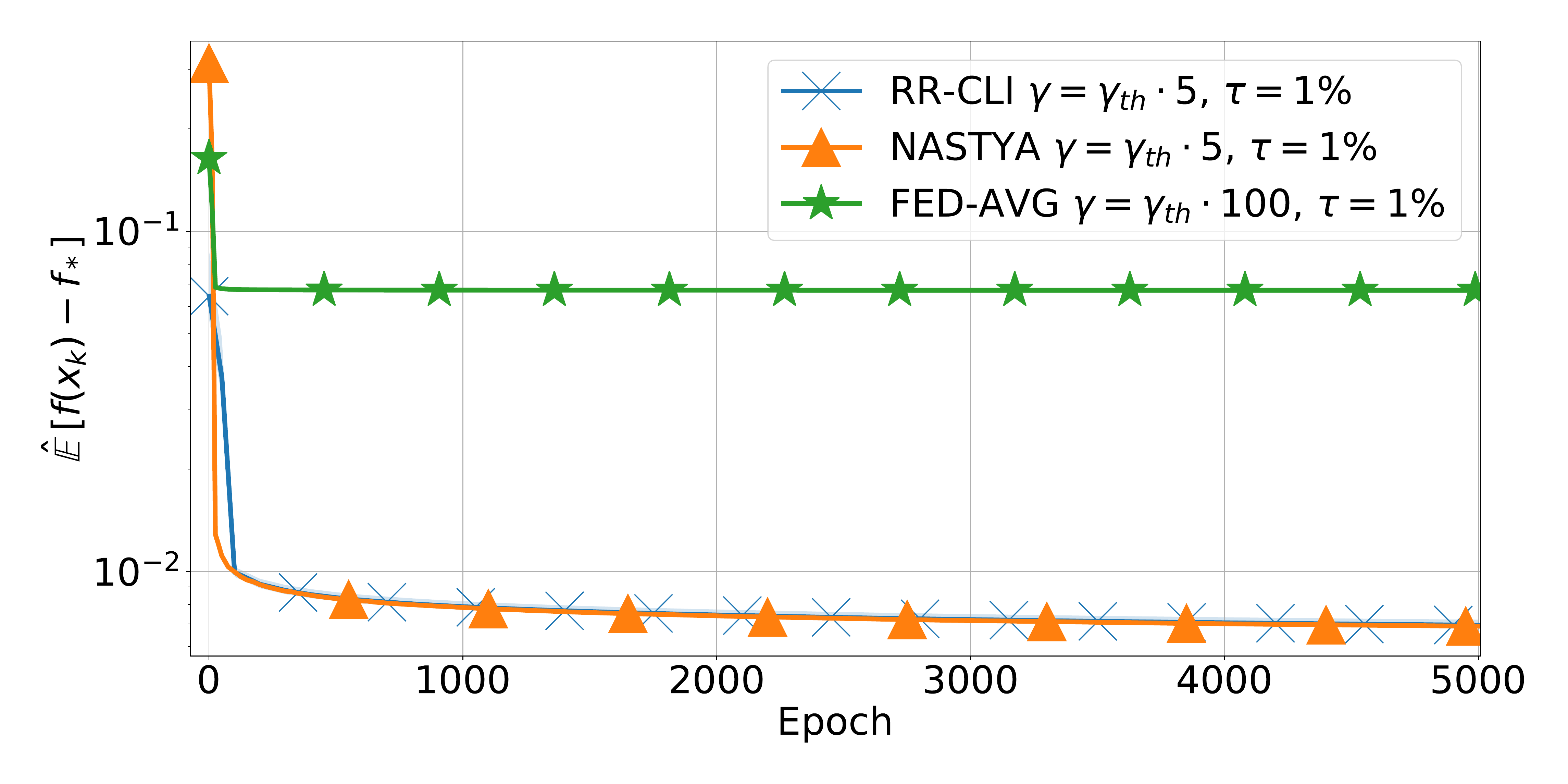}	
	\caption{\small{Training \texttt{Logistic Regression} on \texttt{w3a}, with $n=12$ clients. Theoretical global step size. Local step sizes are multipliers of theoretical. Partial participation with $3$ clients per round, $100$ local step. Local and global step size are decaying $\propto \frac{1}{1+\mathrm{\#passed epochs}}$. Local gradient estimators are computed with $1\%$ of local samples.}}
	\label{fig:exp3_multth_step sizes_best_to_best_decay_extra_a3a}
\end{figure*}

\begin{figure*}[t!]
	\centering
	\captionsetup[sub]{font=scriptsize,labelfont={}}	
	\includegraphics[width=0.43\textwidth]{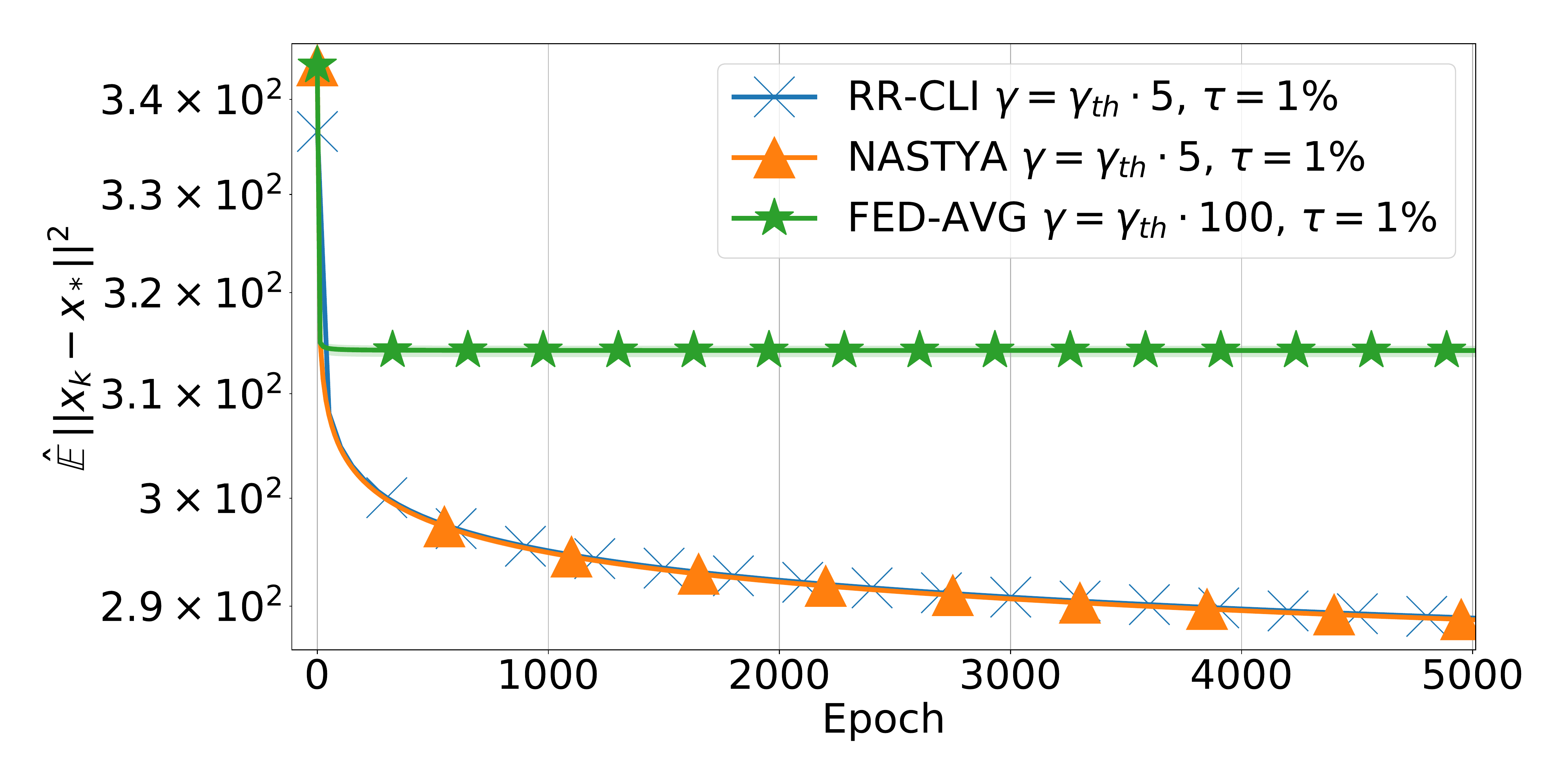}
	\includegraphics[width=0.43\textwidth]{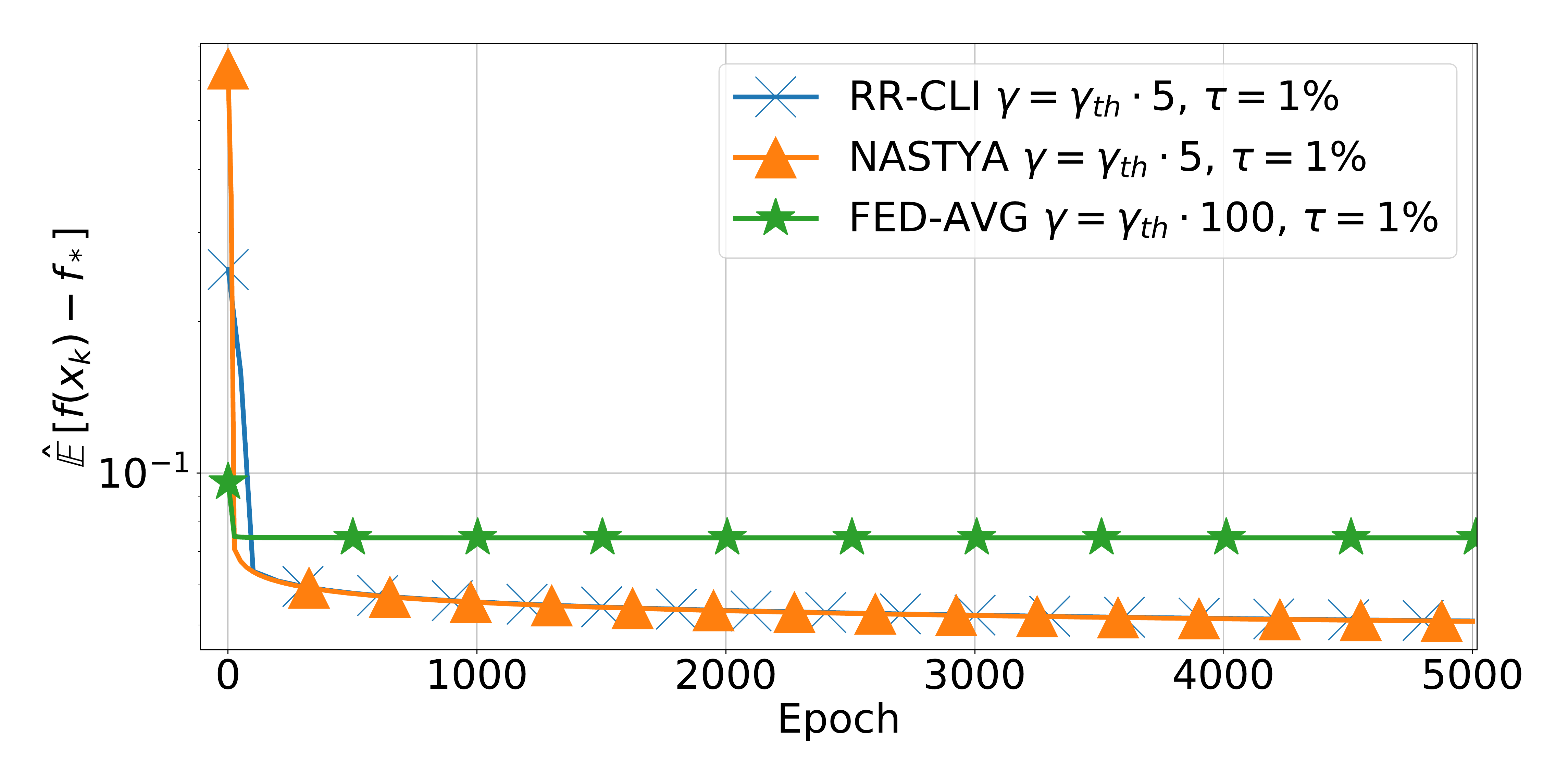}	
	\caption{\small{Training \texttt{Logistic Regression} on \texttt{w3a}, with $n=12$ clients. Theoretical global step size. Local step sizes are multipliers of theoretical. Partial participation with $3$ clients per round, $100$ local step. Local and global step size are decaying $\propto \frac{1}{1+\mathrm{\#passed epochs}}$. Local gradient estimators are computed with $1\%$ of local samples.}}
	\label{fig:exp3_multth_step sizes_best_to_best_decay_extra_w3a}
\end{figure*}

\end{document}

%% file: macros.tex
\usepackage{amsmath,amssymb,amsthm}
\usepackage{color,mathtools}
\usepackage{multirow}
\usepackage{enumitem}

\usepackage{nicefrac}

\usepackage{tablefootnote}

\usepackage{algorithm}
\usepackage{bm}

\usepackage{amssymb,amsmath,amsthm,dsfont}

  \providecommand{\R}{\mathbb{R}} % Reals
  \providecommand{\N}{\mathbb{N}} % Naturals
  
  \renewcommand{\epsilon}{\varepsilon}

  \providecommand{\tt}{\mathbf{t}}

\theoremstyle{plain}
\theoremstyle{definition}
\theoremstyle{remark}
\DeclarePairedDelimiterX{\inp}[2]{\langle}{\rangle}{#1, #2}
\DeclarePairedDelimiterX{\abs}[1]{\lvert}{\rvert}{#1}
\DeclarePairedDelimiterX{\norm}[1]{\lVert}{\rVert}{#1}
\DeclarePairedDelimiterX{\cbr}[1]{\{}{\}}{#1} % curly bracket
\DeclarePairedDelimiterX{\rbr}[1]{(}{)}{#1} % round bracket
\DeclarePairedDelimiterX{\sbr}[1]{[}{]}{#1} % square bracket

\definecolor{mydarkblue}{rgb}{0,0.08,0.45}
\usepackage{xcolor}
\usepackage{color}
\usepackage{graphicx}
\usepackage{verbatim}
\usepackage{multirow}

\newcommand{\cC}{{\cal C}}

\newcommand{\cO}{{\cal O}}

\newcommand{\eqdef}{\overset{\text{def}}{=}} 
\newcommand{\dotprod}[2]{\left\langle #1,#2\right\rangle} % product with 2 arguments
\newcommand{\Exp}[1]{{\rm E}\left[#1\right] }    % expectation
\newcommand{\E}[1]{{\bf E}\left[#1\right] }

\usepackage{url}